\documentclass[letterpaper]{article}

\usepackage{PRIMEarxiv}
\usepackage[utf8]{inputenc}
\usepackage[T1]{fontenc}
\usepackage{fullpage}
\usepackage{amsmath}
\usepackage{amsthm}
\usepackage{amsfonts}
\usepackage{fancyvrb}
\usepackage{xspace}
\usepackage{booktabs}
\usepackage[font=footnotesize]{caption}
\usepackage{tabularx}
\usepackage{rotating}
\usepackage[round,authoryear]{natbib}
\usepackage{array}
\usepackage{comment}
\usepackage{listings}
\usepackage{hyperref}
\usepackage{xcolor}
\usepackage{etoolbox}

\let\cite\citep

\usepackage[section]{placeins} %

\makeatletter
\newcommand{\mathleft}{\@fleqntrue\@mathmargin\parindent}
\newcommand{\mathcenter}{\@fleqnfalse}
\patchcmd{\NAT@citex}
  {\@citea\NAT@hyper@{%
     \NAT@nmfmt{\NAT@nm}%
     \hyper@natlinkbreak{\NAT@aysep\NAT@spacechar}{\@citeb\@extra@b@citeb}%
     \NAT@date}}
  {\@citea\NAT@nmfmt{\NAT@nm}%
   \NAT@aysep\NAT@spacechar\NAT@hyper@{\NAT@date}}{}{}
\patchcmd{\NAT@citex}
  {\@citea\NAT@hyper@{%
     \NAT@nmfmt{\NAT@nm}%
     \hyper@natlinkbreak{\NAT@spacechar\NAT@@open\if*#1*\else#1\NAT@spacechar\fi}%
       {\@citeb\@extra@b@citeb}%
     \NAT@date}}
  {\@citea\NAT@nmfmt{\NAT@nm}%
   \NAT@spacechar\NAT@@open\if*#1*\else#1\NAT@spacechar\fi\NAT@hyper@{\NAT@date}}
  {}{}

\makeatother

\newenvironment{rcases}{\left.\begin{aligned}}{\end{aligned}\right\rbrace}

\newcolumntype{Y}{>{\raggedright\let\newline\\\arraybackslash\hspace{0pt}}X}

\newcommand{\figref}{Fig.~\ref}
\newcommand{\eqeqref}{Eq.~\eqref}
\newcommand{\Eqeqref}{Equation~\eqref}

\newcommand{\1}{\mathbf{1}}
\newcommand{\E}{\mathbb{E}}
\newcommand{\std}{\sigma}

\newcommand{\0}{\mathbf{0}}
\newcommand{\A}{\mathcal{A}}
\newcommand{\G}{\mathcal{G}}
\newcommand{\M}{\mathcal{M}}
\renewcommand{\P}{\mathcal{P}}

\newcommand{\Loss}{\mathcal{L}}
\newcommand{\Otilde}{\tilde O}
\newcommand{\Omegatilde}{\tilde \Omega}
\newcommand{\R}{\mathbb{R}}
\newcommand{\Nat}{\mathbb{N}}
\newcommand{\eps}{\varepsilon}
\newcommand{\poly}{\mathrm{poly}}
\newcommand{\gauss}{\mathcal{N}}

\newcommand{\dragonegg}[1]{}

\newcommand{\itemone}{[No Easy Simulation]\xspace}
\newcommand{\itemtwo}{[Particles Talk]\xspace}
\newcommand{\itemthree}{[Attention Deforms Pairwise Connections]\xspace}
\newcommand{\itemfour}{[Time Dictates Structure]\xspace}

\newcommand{\abc}{{\ensuremath{X(i),Y(j)\to \sigma(i,j)}}\xspace}

\newcommand{\definitionword}{Definition}
\newcommand{\claimword}{Claim}
\newcommand{\findingword}{Empirical Finding}
\newcommand{\hypothesisword}{Hypothesis}
\newcommand{\observationword}{Observation}

\newtheorem{definition}{\definitionword}
\newtheorem{claim}{\claimword}
\newtheorem{finding}{\findingword}
\newtheorem{hypothesis}{\hypothesisword}
\newtheorem{observation}{\observationword}

\newcommand{\JL}{Johnson-Lindenstrauss\xspace}

\newcommand{\BDHGPU}{\textrm{BDH-GPU}\xspace}

\newcommand{\xysparse}{y}
\newcommand{\xsparse}{x}
\newcommand{\relu}[1]{\left(#1\right)^{+}}
\newcommand{\relusm}[1]{(#1)^{+}}
\newcommand{\yKV}{a^*}
\newcommand{\vv}{v^*}
\newcommand{\encoder}{E}
\newcommand{\decoder}{D}
\newcommand{\decodery}{{D_y}}
\newcommand{\decoderx}{{D_x}}
\newcommand{\graph}{G}
\newcommand{\graphy}{{G_y}}
\newcommand{\graphx}{{G_x}}
\newcommand{\graphs}{{G_s}}
\newcommand{\ee}{^{\mathfrak{e}}}
\newcommand{\ii}{^{\mathfrak{i}}}
\newcommand{\rope}{U}
\newcommand{\state}{{\boldsymbol\rho}}
\newcommand{\corr}{{\boldsymbol\sigma}}
\renewcommand{\*}{\odot}
\newcommand{\layernorm}[1]{\mathsf{LN}\left(#1\right)}

\newcommand{\tensor}{\otimes}
\newcommand{\lr}[1]{\left(#1\right)}
\newcommand{\ket}[1]{{#1}}
\newcommand{\bra}[1]{{#1}^T}
\newcommand{\braket}[2]{{#1}^T {#2}}

\newcommand{\ix}[2]{#1(#2)}
\newcommand{\ruleone}[5]{&#1(#2)\xrightarrow{#3}#4(#5)}
\newcommand{\ruletwo}[7]{&#1(#2),\ #3(#4)\xrightarrow{#5}#6(#7)}
\newcommand{\ruledown}[3]{&#1(#2)\downarrow_{#3}}
\newcommand{\ruletwodown}[6]{&\relusm{#1(#2)-#3(#4)}\xrightarrow{}#5(#6)}
\newcommand{\rulethreedown}[8]{&\relusm{#1(#2)-#3(#4)},\ #5(#6)\xrightarrow{}#7(#8)}

\title{
The Dragon Hatchling: The Missing Link\\between the Transformer and Models of the Brain}

\author{
Adrian Kosowski\thanks{Author contributions are listed at the end of the paper.}\,
\thanks{Corresponding author.}\and
\textbf{Przemysław Uznański}\footnotemark[2]\and
\textbf{Jan Chorowski}\footnotemark[2]\and
\textbf{Zuzanna Stamirowska}\footnotemark[2]\and
\textbf{Michał Bartoszkiewicz}\footnotemark[2]\\[2mm]
Pathway, Palo Alto, USA\\
\texttt{research@pathway.com}
}

\begin{document}

\vspace*{-6mm}
\maketitle

\vspace*{-7mm}
\begin{abstract}
\enlargethispage{3mm}%
The relationship between computing systems and the brain has served as motivation for pioneering theoreticians since John von Neumann and Alan Turing. 
Uniform, scale-free biological networks, such as the brain, have powerful properties, including generalizing over time, which is the main barrier for Machine Learning on the path to Universal Reasoning Models.

We introduce `Dragon Hatchling' (BDH), a new Large Language Model architecture based on a scale-free biologically inspired network of $n$ locally-interacting neuron particles. BDH couples strong theoretical foundations and inherent interpretability without sacrificing Transformer-like performance.

BDH is a practical, performant state-of-the-art 
attention-based state space sequence learning architecture. 
In addition to being a graph model, BDH admits a GPU-friendly formulation. %
It exhibits Transformer-like scaling laws: we find empirically that BDH rivals GPT2-architecture Transformer performance on language and translation tasks, at the same number of parameters (10M to 1B), for the same training data.

BDH provides theoretical foundations for understanding model behavior in the limit of large size and reasoning time. 
Our results, formalized as a chain of reductions of expressiveness in the framework of computational Complexity Theory and Distributed Computing, and combined with findings on the BDH model, show a macro-to-micro correspondence of function between the general attention mechanisms in state-of-the-art Language Models, and attention mechanisms observed in the brain. These attention mechanisms formally converge as closed-form local graph dynamics at neurons and synapses: ``the equations of reasoning''.

BDH can be represented as a brain model. It contains $n$ neurons, organized as an excitatory circuit and an inhibitory circuit with integrate-and-fire thresholding of input signals at neurons. The working memory of BDH during inference entirely relies on synaptic plasticity with Hebbian learning using spiking neurons, at potentiation scales of minutes for the brain (up to hundreds of tokens). We confirm empirically that specific, individual synapses strengthen connection whenever BDH hears or reasons about a specific concept while processing language inputs. The neuron interaction network of BDH is a graph of high modularity with heavy-tailed degree distribution. The BDH model is biologically plausible, explaining one possible mechanism which human neurons could use to achieve speech.

BDH is designed for interpretability. Activation vectors of BDH are sparse and positive. We demonstrate monosemanticity in BDH on language tasks, including representation of concept abstractions, which happens even for small models, below 100M-parameter scale. Interpretability of state, which goes beyond interpretability of neurons and model parameters, is an inherent feature of the BDH architecture. 

We believe BDH opens the door to a new theory of ``Thermodynamic Limit'' behavior for language and reasoning models, with the ultimate goal of Probably Approximately Correct (PAC)-like bounds for generalization of reasoning over time.

\vspace*{1.5mm}
$\rhd$ Technical blog entry: \url{https://pathway.com/research/bdh}.

$\rhd$ Code listings: \url{https://github.com/pathwaycom/bdh}.
\end{abstract}

\tableofcontents

\section{Introduction}

Long reasoning and long context inference pose a severe challenge of generalization across scales of time. From vibe coding to market research, users of Language Models and agentic systems are increasingly relying on defining tasks through informal prompts, which the language model is expected to follow over long sequences of actions or decisions, like a reasonable human actor would. Implicitly, most users expect machines to follow the generalization patterns of human reasoning, i.e., to generalize reasoning in the same way as humans do. The complexity of tasks attempted in this way has gone from the equivalent of hours of human work for a single prompt, to weeks~\cite{epoch2025outputlength}. However, experimental evidence suggests that the Transformer and other state-of-the-art architectures do not systematically generalize chain-of-thought (CoT) reasoning to scenarios longer than the ones seen during training~\cite{shojaee2025illusionthinkingunderstandingstrengths}.

Chain-of-Thought reasoning models can be considered through the lens of computational complexity theory. For a Language Model to generalize human reasoning on a given class of tasks, we expect this model to be able to emulate the corresponding reasoning function of the human brain efficiently.\footnote{We provide a more formal explanation of this point in Appendix~\ref{appx:one}.}
While the Transformer with Chain-of-Thought is Turing-complete and can efficiently emulate certain restricted classes of formal languages~\cite{merrill2024expressivepowertransformerschain}, this does not in itself provide a satisfactory answer as to how it emulates human reasoning. The human brain is an extremely complex graph-based distributed computing system with $n \approx 8\cdot 10^{10}$ neurons, and $m > 10^{14}$ neuron connections (synapses), of which a certain percentage is actively used. The direct simulation of such a distributed system by a Language Model through generic Turing-machine reductions would require billions of CoT tokens of the Language Model to represent a single step of reasoning in the brain. So, do Transformer-like models actually relate to brain function?

Such a relationship should follow more closely from a tighter, more direct simulation. Finding such a connection between Language Models and human brain function has, so far, proved elusive. Indeed, when comparing a tensor-based Language Model based on feed-forward network blocks and attention, to a uniform, scale-free graph-based distributed system, such as the brain, the two may, at first glance, appear very dissimilar.

This apparent dissimilarity of structure between Language Models and brain structure has been one of the main causes of concern in attempts to reconcile Computation and the Brain~\cite{simons2018}, as well as a cause of concern regarding the difficulty to foresee the behavior of autonomous AI systems.

In this paper, we show the link between the Transformer and Brain models.

\subsection{Motivation}

The development of Artificial Intelligence and the understanding of Neural Science have gone hand in hand since the 1940's, both being efforts to understand the ``mystery of intelligence''. The relationship between computing systems and the brain served as motivation for the pioneering theoreticians such as John von Neumann~(\citeyear{10.5555/578873}), Alan Turing~(\citeyear{turing1950computing}), Goeff Hinton~(\citeyear{10.5555/1642293.1642643}), Warren McCulloch and  Walter Pitts~(\citeyear{mcculloch1943logical}), and Horace Barlow~(\citeyear{doi:10.1068/p010371}). 

Since then, milestones in Machine Learning around Artificial Neural Networks --- using backpropagation with SGD~\cite{rumelhart1986learning}, followed by Deep Learning~\cite{lecun2015deep}, and the Attention mechanism~\cite{bahdanau2014neural,vaswani2017attention} --- have split the ``mystery of how intelligence works'' into two.  First, we still have no clear explanation for the micro-to-macro correspondence of the reasoning function of the brain. Second, we do not understand the correspondence between the artificial and natural systems --- notably, how effects observed in the brain (emergent network; sparse activations; oscillatory phenomena; unknown relationship to backpropagation mechanisms) map into those which appear in systems based on dense tensors, trained using gradient back-propagation over time.

\paragraph{Reconciling Reasoning Function of the Brain with Language Models.}

There is a seemingly deep divide between state-of-the-art language models, like the Transformer, and natural distributed systems with local graph dynamics, like those of the brain. Specifically, for the brain, we do not understand how the reasoning function emerges from neuronal dynamics at the microscale. For the Transformer, the interpretation of function is given at the level of vectors, but not at the level of particle dynamics or a uniform distributed computing system. 

Language and reasoning are the key areas of higher-order brain function for which we do not yet have a complete understanding. Many other areas of brain function have been explained through analogies to Machine Learning architectures.
For example, the visual cortex is becoming well-understood, especially in its peripheral layers, and the observed inference dynamics are shown to have a correspondence to known Deep Learning architectures~\cite{mohsenzadeh}. The use of sparse coding by the brain was considered in the context of processing visual cues~\cite{olshausen1997sparse}, as well as for the olfactory systems~\cite{lin2014sparse}. By contrast, higher-order cognitive functions of the association cortex of the human brain, such as language and reasoning, are among the least understood. A number of models provide partial explanations and have been verified at small scales. Some of the first attempts include explaining context-dependent computation in the prefrontal cortex using population dynamics of an RNN~\cite{mante2013context}. Later approaches include the Tolman-Eichenbaum Machine~\cite{WHITTINGTON20201249,whittington2022relating}, as well as a number of more recent works~\cite{papadimitriou1,papadimitriou2,papadimitriou3}. One of the main stumbling blocks concerns going from spiking activation patterns at neurons, and localized attention effects at synapses, to a higher-order function, serving a reasoning purpose, efficiently organized at a scale of millions to billions of neurons.

Conversely, for Language Models architectures such as the Transformer, we miss a compact micro-interpretation as a distributed system. The expressiveness of the Transformer has been approximated using approaches from centralized computing and Complexity Theory, rather than from distributed systems. In the centralized perspective, a language model can be seen as a transformation function from inputs into outputs. The computational expressiveness of the Transformer architecture may then be approximated through frameworks based on RASP, such as RASP-L~\cite{zhou2023algorithms} or C-RASP~\cite{yang2024countingliketransformerscompiling,huang2025a}. RASP-L provides a very convenient heuristic for estimating Transformer expressiveness at the rather coarse level of vector operations, while C-RASP provides a more specialized lower-bound on expressiveness, capturing a class of formulas of temporal counting logic. Both frameworks have been used to suggest theoretical models of task length generalization.
This type of expressiveness techniques, however, do not lead to a uniform asymptotic model for the behavior of the Transformer, whether in GPT2 architecture or simplified. The scaling of the Transformer in its different dimensions, and the need to manipulate context length, complicate this goal.

The lack of such a uniform model also makes it hard to compare the capabilities of the Transformer to the capabilities of the brain at the level of correspondence of structure. Generally, the temporal behavior of a state-space system is reflected in its structure\footnote{For a linear system, temporal behavior would be a direct consequence of the spectral properties of the system. The considered systems dynamics are not linear.}.

Understanding whether it is possible to show alignment of the temporal behavior of two systems, which do not display any structural correspondence, and without a clear idea of how the weight tensors and state representation of one system `embed' into the graph structure and state representation of the other system, is an awkward task.

This brings us naturally to our motivational objective: Can we create Machine Learning models which are closer to the desirable properties of natural (human) reasoning systems, and which exhibit the same types of limit and scaling behavior as such natural systems?

\paragraph{Towards scale-free foreseeable AI.}

Ensuring correct scaling behavior of inference over time is of paramount importance for the deployment of AI whose reasoning or actions are not subject to strict human supervision. Most reasoning models and AI agentic systems admit limit objects (i.e., extensions to infinite time and infinite size) which are Turing-complete (cf.~e.g.~\cite{merrill2024expressivepowertransformerschain,attentionisturingcomplete,jojic2023gpt}). This means that they should be treated like computer programs --- and should be approached by the users with the same standards of care, as a computer program of unknown origin and unknown purpose. 

An AI model can malfunction when allowed to run for a long time autonomously, i.e., without human validation of actions and reasoning outcomes. The most painful of all consequences, perhaps, is the concept of a failed generalization of reasoning (a malfunction with respect to the original task objective) over time, leading to a grotesque effect known as the ``Paperclip Factory''~\cite{bostrom2014superintelligence}.

Can the risk of such unsuccessful generalization be bounded?

There are at least two scenarios in which a black-box model $M$ cannot be considered to have undergone previous empirical validation, and consequently cannot be used in higher-risk autonomous AI use cases.
\begin{enumerate}
\item Length-generalization scenario: Model $M$ is expected to act autonomously on a task which is longer than tasks forming part of its validation set.
\item Model scaling scenario: Model $M$ is not exactly the same closed system as the one which was tested during validation. For example, suppose that models $M_1$ and $M_2$ were tested individually on smaller tasks, and let $M$ be an agentic system composed of instances of $M_1$ and $M_2$ which communicate with exchange messages with each other during inference.
\end{enumerate}
A natural way of avoiding both difficulties consists in studying systems which are scale-free with respect to size and time, and admit a form of uniform ``thermodynamic limit'' behavior. The limit behavior of computational systems at criticality naturally connects the size of the system with the probable duration of its operation, with the connection usually taking polynomial form (cf.~e.g.~\cite{BJORNER1991283,haltingsandpiles,Rolla_2020} for examples of graph-based interacting particle systems for which rigorous results have been obtained in this direction). Consider a model $M_n$ with architecture $\A$, parameterized by its size $n$ (with the interpretation of the number of uniform particles), and sampled from some space of $n$-neuron models in architecture $\A$ in some space equipped with a probability measure, $M_n \sim \P_\A (n)$. Informally, if the limit object $\P_\A := \lim_{n \to \infty}\P_\A(n)$ exists (under an appropriate, well-defined sense of uniformity of limit) then models $M_n$, for $n$ sufficiently large, will admit in the limit asymptotic properties, which can be used to characterize their behavior over time.

The existence of such a limit theory means that we can characterize, with bounded probability of error, the behavior of a family of large models, having $O(n)$ parameters, while relying on a theory which is independent of the specific structure and size of the specific model. %
In this way, the limit behavior of a system of a very large number of interacting uniform particles over time becomes (stochastically) foreseeable in the sense of its adherence to expected behavior, which can be extrapolated from observations at shorter time scales. Thus, small tests may be conceived in order to provide validation for a scale-free system at long time scales.

\paragraph{Introducing Axiomatic AI.}

Axiomatic systems are those in which micro-foundations and the macro-description which arises from them are consistent and well-understood. The need for axiomatic understanding was highlighted by David Hilbert~(\citeyear{Hilbert:1902:MP}), and has become the foundation in Statistical Physics (e.g.~thermodynamics, fluid dynamics, spin glass  theory), cellular mechanisms, Social Networks Science, and reconciliation of Microeconomics and Macroeconomics through a Network Economics perspective.

This paper brings a micro-foundational understanding to Language Model inference, to the mechanisms of in-context learning, and Chain-of-Thought reasoning dynamics.

The considerations in this work naturally support a shift of perspective from \emph{Interpretable AI}, which gives an approximate understanding of what the model is doing now (without necessarily telling us what its current actions are going to lead to over longer time scales), to \emph{Axiomatic AI}, where we also understand the micro-foundations of how the model can be expected to behave subsequently over time.

\subsection{Intuition of results: combining \textsl{modus ponens} reasoning with Hebbian learning}\label{sec:intuition}

In this section we provide the reader with some of the main intuitions behind this work which, we hope, will help to navigate the remaining, more formal parts of this paper with ease.

While there are many formal deductive systems in logic, they predominantly rely on the \textsl{modus ponens} inference rule. Applied to a rule-based reasoning system, it takes the following form:

If we know that the $i$-th fact is true, and our ruleset $\sigma$ indicates that the $i$-th fact implies the $j$-th fact, then we know that the $j$-th fact is true as well. In an approximate reasoning system, the strength of the rule $\sigma(i,j)$ indicates %
how the belief $X(i)$ of the system affects its belief $A(j)$. We could write: 
\begin{equation}\label{eq:intuitive_modusponens}
X(i), \sigma(i,j) \xrightarrow{} A(j),
\end{equation}
to indicate that if $X(i)$ is a weighted belief, it contributes $X(i)\sigma(i,j)$ to the system's belief $A(j)$.

Practical logical inference systems differ in strategies employed for rule selection, with the most advanced ones allowing direct manipulation of the ruleset, effectively resulting in a form of program evolution during inference\footnote{The authors' personal experience with writing efficient Prolog programs confirms that such direct ruleset management is often a necessary pragmatic evil, guiding the inference system in the right direction.}. For an approximate reasoning system, such a heuristic could manipulate the strength of rules, modulating the impact of belief $X(i)$ on the system's belief $A(j)$.

Hebbian learning \cite{hebb1949organization}, often presented as the mnemonic ``\emph{Neurons that fire together wire together}'', can be seen as a heuristic for ruleset manipulation. It postulates that synaptic connections are strengthened when the activity of one neuron, $Y(i)$, led to the firing of another neuron, $X(j)$. In the context of an adaptive, approximate inference system, the Hebbian heuristic means that if during the course of operation a fact $i$ contributed some evidence for $j$, the system increases the significance of the implication $\sigma (i,j)$. We could write this rule as:
\begin{equation}\label{eq:intuitive_hebb}
Y(i), X(j) \xrightarrow{} \sigma(i,j),
\end{equation}
with the interpretation that co-presence (or a spike) of $Y(i)$ followed by $X(j)$ increases $\sigma(i,j)$ by $Y(i)X(j)$.

The relations \eqref{eq:intuitive_modusponens} and \eqref{eq:intuitive_hebb}, over a set of $n$ facts, may form the basis of a simple approximate reasoning system that adapts its operation to the problem at hand. Starting with some initial connections between facts, the system applies the rules to discover new facts, at the same time reweighting the ruleset in a way that strengthens the connections between the initial and derived facts. Effectively, should the system be rerun with the new ruleset, it would arrive at similar conclusions faster.

Suppose now that the reasoning system is equipped with two sets of rules: a fixed set $G$ and an evolving set $\sigma$. From a machine learning perspective, the fixed ruleset $G$ can be seen as model weights in Deep Learning terminology, learned using e.g. error backpropagation on a training set. On the other hand, the evolving ruleset can be seen as the temporal state of the reasoning system, sometimes called ``fast weights'' \cite{hinton1987using, schmidhuber1993reducing, ba2016using}. Fast-weights systems have a favorable ratio of state size to parameter count. A system with $n$ facts has $m= O(n^2)$ trainable parameters (expressed using one or more $n\times n$ matrices). A classical recurrent neural net, such as the LSTM \cite{10.1162/neco.1997.9.8.1735}, treats individual fact (neuron) activations as its state, thus maintaining only $O(n)$ state variables. On the other hand, the evolving set of fast-weights $\sigma$ has $m=O(n^2)$ state entries. We believe this 1-1 ratio of trainable parameter to state size is important in designing practical reasoning systems and may justify the success of the Transformer \cite{vaswani2017attention} and state-space \cite{gu2024mambalineartimesequencemodeling} sequence processing models. 

Now, bearing in mind that the trainable parameters and state have comparable size $m$, we can adjust the ratio between this value $m$ and the size $n$ of the fact base. This will happen through a choice of sparsity for the $n\times n$ matrices carrying parameters and state, resulting in a specific relationship of the two values, $n \ll m \ll n^2$. In this way, our system gets a natural interpretation in terms of graphs on $n$ nodes and $m$ edges, with the graph edges tasked with their first two roles: carrying state, and, carrying trainable parameters. Finally, we will give our system an interpretation of a dynamical system with distributed (localized) dynamics, and we will task our edges with their third crucial role: mediating in communication between nodes of the system. In this way, through assimilation of edges to natural function in the brain, we will refer to the $m$ edges as \emph{synapses} connecting a set of $n$ \emph{neurons} into a distributed graph-based system.

In the following Section~\ref{sec:bdhgraph}, we will introduce BDH, a reasoning system that formalizes and combines relations \eqref{eq:intuitive_modusponens} and \eqref{eq:intuitive_hebb} with dynamics involving fixed rules. The BDH system:
\begin{enumerate}
    \item is a reasoning system, efficiently using the \textsl{modus ponens} reasoning rule with heuristic rule reweighting, based on \eqref{eq:intuitive_modusponens} and \eqref{eq:intuitive_hebb},
    \item can be implemented with local graph dynamics, making it suitable for brain-like execution model, and amenable to a principled, axiomatic description,
    \item contains a set of fixed connections (parameters), and a set of dynamically adjusted connections ($\sigma$), which can be seen as its dynamic state updated with a Hebbian learning rule,
    \item admits as its special case \BDHGPU, a GPU-efficient reasoning model architecture, introduced in Section \ref{sec:bdh_arch} and experimentally validated at scale in Section \ref{sec:experiments} in direct comparison to state-of-the-art GPT2-like Transformers.
\end{enumerate}

\subsection{Contribution of this work}%

The focus of this paper is in explaining the dynamics of the primary function of language and reasoning models: inference. We provide a description of a language model architecture which is directly comparable to the Transformer, and admits a clear and interpretable local interpretation of its inference dynamics as a programmable interacting particle system.

\paragraph{Language Models as Local Graph Dynamics.} 
In Section~\ref{sec:bdhgraph}, we introduce a graph-based model architecture called \emph{BDH}, where all model parameters are represented as topology and weights of the communication graph, and model state during inference is represented as edge-reweighting applied to this graph topology.

\begin{claim}[informal overview of theoretical results for BDH]
We introduce a state-space Machine Learning architecture called BDH, formed by a system of $n$ particles called neurons which communicate in a way governed by the weights and topology of the system graph, representing a ``communication by wire'' network.
\begin{itemize}
\item The inference dynamics of BDH, treated as a distributed system, can be represented as execution of local rulesets for $n$ particles with programmable interactions, with particles acting as nodes of the interaction graph and scalar state variables located on its edges (cf.~Section~\ref{sec:equations_of_reasoning}).
\item The local kernel of BDH can be naturally expressed (emulated) by a graph-based Spiking Neural Network system capable of Hebbian learning dynamics, an Excitatory circuit, and an Inhibitory circuit on an $n$-neuron system described by a neuron interaction graph (cf.~Section~\ref{sec:brain_models}).
\end{itemize}
\end{claim}
In order to train BDH efficiently and analyze its performance, we restrict it, making this restriction the core of a GPU-friendly architecture called \emph{\BDHGPU}. This restriction is obtained by treating the communication of the $n$ particles as proceeding through a mean-field (``radio network''), rather than a graph (``communication by wire''), cf.~\figref{fig:ss} for an explanation of how the state-space equations of \BDHGPU are obtained from BDH. 

This allows us to train a mathematically equivalent model, while localizing its state in short vectors at neurons, not at connections (synapses) of the system. 

\paragraph{A tensor-friendly case of BDH: the \BDHGPU architecture.}
The \BDHGPU architecture, like the Transformer, crucially relies on an attention mechanism, and is amenable to token-parallel training on GPU for next token prediction tasks. Unlike the Transformer, activation vectors of \BDHGPU appear in a very high dimension $n$, are positive by design, and turn out to be sparse.
\begin{claim}[informal overview of theoretical results for \BDHGPU]
We introduce a Machine Learning architecture called \BDHGPU, parameterized by a single (very large) scaling parameter $n$ and a second parameter $d$, $\log n < d \ll n$ ($d=256$ in practice), such that:
\begin{itemize}
\item A model in \BDHGPU$(n,d)$ has $(3+o(1)) nd$ parameters, and admits a precise interpretation as a state-space system following the local dynamics of a $n$-particle system in an interaction field subject to equations of state~\eqref{eq:bdh}. This system is described by $O(d)$ parameters per particle, whose interaction field has mean field interpretation, which in a computational view corresponds to a particle communication network realized by means of ``noisy radio broadcast''.
\item \BDHGPU is a special case of BDH in the sense that, for any \BDHGPU model with $n$ particles, there exists a BDH model with $n$ particles with the same inference behavior and the same size $O(nd)$ of trainable parameters, with the two models being formally equivalent  up to placement of Layer Norms (cf.~Claims~\ref{claim:graphs}~and~\ref{claim:att_equiv}).
\item The \BDHGPU architecture relies on a combination of two blocks: a specific kind of \emph{ReLU-lowrank} feed-forward network, and a \emph{linear attention} mechanism, which both operate in the same neuron dimension $n$, using positive activation vectors.
\item The mechanisms of \BDHGPU, considered at the macro-level of activation vectors in $R^n$, can be compared to those of the Transformer (cf.~Section~\ref{sec:bdhgpumacroexpattention}, Section~\ref{sec:macrocomparison-of-feedforward}). This justifies the applicability of the frameworks of approximate macro-expressiveness, based on RASP~\cite{weiss2012thinking,zhou2023algorithms,yang2024countingliketransformerscompiling} and designed for the Transformer, to \BDHGPU.
\item The micro-interpretation of \BDHGPU mechanisms as neuron-neuron interaction dynamics: (1) explains mechanisms of in-cluster communication of neurons and the spontaneous emergence of graph structure with high Newman modularity in the neuron-neuron communication network (cf.~Section~\ref{sec:modularity}), and (2) provides a strict correspondence between the macro-mechanism of in-context inference based on attention and the local representation of state on individual neuron-neuron pairs (synapses) with state update dynamics based on sporadic updates to synaptic edge weight (cf.~Section~\ref{sec:attention}).
\end{itemize}
\end{claim}

The above results are complemented by empirical findings.

\begin{finding}[informal overview of empirical results of \BDHGPU]
\BDHGPU is represented as a tensor-based architecture and can be trained with standard back-propagation methods (cf.~Section~\ref{sec:bdh_arch}).
\begin{itemize}
\item The \BDHGPU architecture is shown to follow scaling laws (parameters vs.\ loss) of optimized Transformers in the GPT architecture, at parameter scales between 10M to 1B, on all next token prediction tasks we tested, including tasks of language and translation reminiscent of those in the original benchmark set for the Transformer architecture (cf.~Section~\ref{sec:comparison_bdh_gpt_transformers}).
\item An emergent network reflecting the associated BDH graph dynamics can be read out directly from the parameter matrices of a trained \BDHGPU model, showing emergence of graph structure~(cf.~Section~\ref{sec:bdh_empiricalgraphs}).
\item The positive activations of \BDHGPU exhibit sparsity (at about 5\% level) in the $\xysparse$ vectors of its state space dynamics, with sparsity levels reflecting the amount of activity being performed by \BDHGPU for a given token~(cf.~Section~\ref{sec:sparsity}).
\item In-context state of \BDHGPU attention is shown to localize on the same synapses (neuron-neuron links) consistently across multiple prompts, allowing for some basic features, the interpretation of the current in-context state based on the reading of state of an individual synapse associated with such a feature (cf.~Section~\ref{sec:monosynapse}).
\end{itemize}
\end{finding}

A more detailed discussion of the training approach is provided in Appendix~\ref{sec:bdh_scaling_details}, while the code listing for \BDHGPU is provided in Appendix~\ref{sec:bdh_code_listing}. For the purposes of our experiments, we did not apply any specific training method which would be known to guide the system towards any of the observed emergent properties. (In particular, L1-regularization was disabled.) The observed emergent effects follow naturally from the design choices of the BDH and \BDHGPU architectures, and are largely attributable to the combination of: the choice of model dimensions with comparable model-to-state ratio, reliance on linear attention in high dimension, reliance on ReLU thresholds for ensuring that activation vectors are positive (trivially) and sparse (an effect empirically noted in~\cite{haziza2025acceleratingtransformerinferencetraining}).

We also remark that the \BDHGPU architecture allows for the uniform asymptotic scaling of the model in one dimension, $n$. For example, a composition of models, obtained by concatenation, is also model in the same architecture, with a larger value of $n$ (cf.~Section~\ref{sec:experiments_merging} for an empirical study of this effect for practical translation tasks).
Historically, a link has been established between infinitely wide feedforward networks and Gaussian Processes~\cite{neal2012bayesian,lee2017deep,yang2019wide}. BDH allows the study of limit behavior of reasoning models.

\begin{quotation}
\emph{With BDH and \BDHGPU, we show that Language Models can be amenable to a particle-based interpretation. In fact, two micro-foundations --- particle-based behavior and logic-programming behavior of a reasoning system --- fuse together in these architectures.}
\end{quotation}

\paragraph{The bridge between the Transformer and Brain models.}

The inference dynamics of BDH and \BDHGPU act as a natural bridge between Transformer, and neuromorphic models of the brain and its subsystems. We illustrate this in \figref{fig:bridge}.

\begin{figure}
  \centering
  \includegraphics[width=\textwidth]{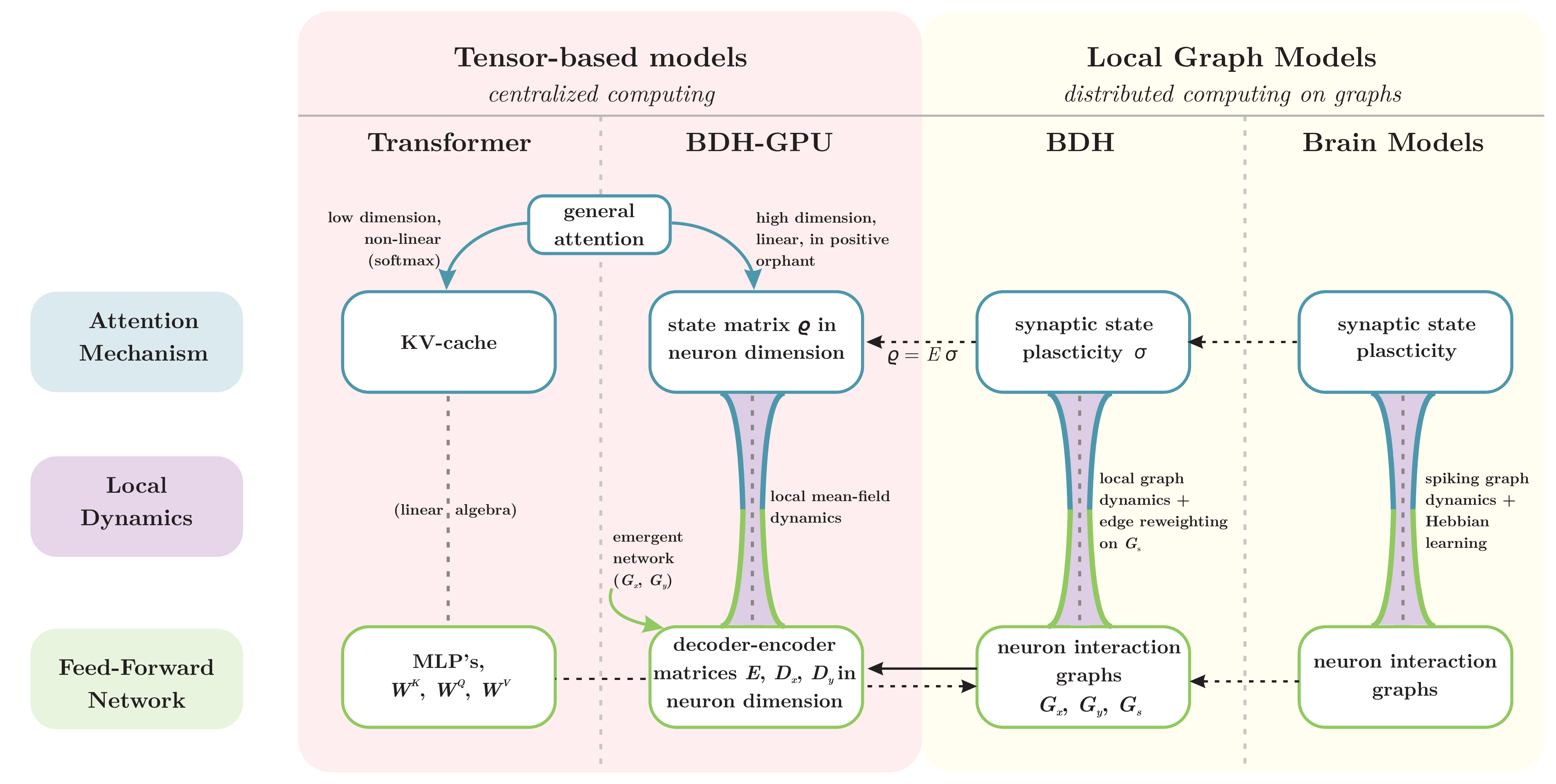}
  \caption{General overview of architectures and their relationships: the inference dynamics of BDH and \BDHGPU act as a natural bridge between Transformer and models of the brain. The two main inference mechanisms of a reasoning architecture, attention and the feed-forward network, are defined at a macro-level through tensor operations for the Transformer, and at the micro-level of neuron interactions through local graph dynamics for Brain models. The new \BDHGPU architecture is naturally defined both at the level of vectors and of particle dynamics of neurons and synapses, acting as a bridge between these two approaches. See also Table~\ref{tab:comparison} at the end of the paper for a more detailed comparison of architecture properties.}\label{fig:bridge}
\end{figure}

\paragraph{Implications for learning dynamics of natural lifelong inference systems.}

A lifelong learning system progresses in time, performing extremely rapid inference, combined with several training mechanisms at different time scales.

In this work, we provide and validate at scale a plausible explanation of what the \emph{predominant} dynamics of such a system could look like, taking the system from `split-second' scale, to the scale of inference during `minutes', considering the flow of time at the natural rate of thought and language for humans.

A complementary discussion of learning dynamics would aim to provide an explanation of how to take such a lifelong inference system from the scale of `minutes' into even longer timescales. This would concern the slower transfer of ``fast-weights''-like inference state to long-term memory, starting at the order of $10^3$---$10^4$ tokens, and taking into account feedback signals. In this work, we do not provide a direct answer as to how the brain actually handles this effect at longer timescales. However, a constructive way to resolve this problem seems to be less challenging, once the local inference dynamics of the brain are better understood (we come back to this in the Conclusions). The modeling approach provided in Section~\ref{sec:brain_models} is proposed as a suitable framework for such a study.

\subsection{Notation}\label{sec:notation}

\paragraph{State-space models.} For describing inference dynamics of any system, we will use state-space notation, and consider a state-space system composed of two parts: a set of \emph{model parameters} $M$ which does not change during inference, and a \emph{state} $\sigma(t)$ which changes during inference. The model performs inference following state-space equation $\sigma(t+1) := \A(M, \sigma(t), a_t)$, where $a_t$ is a possible external input to the system at time $t$ (such as a language token), $t=0,1,2,\ldots$, and $\A$ is referred to as the \emph{architecture} $\A$ which drives its progress. During inference without external input, usually autoregressive inference, we will shorten this to $\sigma(t) := \A^t(M, \sigma_0)$. %

\paragraph{Models as programs.} In settings that are of interest to us (inference with combining multiple facts, reasoning), we opt for terminology from computing. $M$ has the interpretation of a computer program code, $\A$ has the interpretation of a computational machine architecture which runs it, and $\sigma$ has the interpretation of the variable state of the program. We will use the terms `model $M$' and `program $M$' interchangeably.

\paragraph{Graphs and their dynamical systems interpretation.} For a square matrix with non-negative coefficients, $H \in (\R^+)^{n,n}$, $n \in \Nat$, we will consider two more equivalent representations. In one, we will treat $H$ as a graph defined on some nodeset $V$, with $V=|n|$. Formally, we can take $V = \{e_1,\ldots,e_n\}$, where $e_i = (0,\ldots,0,1,0\ldots,0) \in \R^{n\times 1}$ with $1$ on the $i$-th position, forming an orthonormal basis. Non-zero entries of $H$ are referred to as \emph{edges}. By an overloading of notation, we will write $H(i,j):= \bra{e_j} H \ket{e_i} \geq 0$, to represent the node affinity function, or \emph{edge weight}, from $i$ to $j$. We define the \emph{edge set} $E(H) := \{(i,j) \in V\times V: H(i,j)>0\}$.

In discussions of graph-based model architectures, we will take the standard interpretation of graphs from a linear dynamical systems perspective, applied to positive vectors. When $v \in (\R^+)^{n\times 1}$ is a non-negative vector, $Hv \in (\R^+)^{n}$ has the interpretation of a linear transformation of $v$. If $H$ satisfies the condition of stochasticity (column-normalization to $1$), then $v\mapsto Hv$ is a Markov chain transition, with $\|Hv\|_1 = \|v\|_1$. From a distributed systems perspective, transitions of stochastic matrices can be represented either through the direct simulation of (probabilities) of such a Markov chain, or described by the token dynamics of an extremely simple stochastic token distribution scheme in which a token located at node $e_i$ goes to node $e_j$ with probability $H(i, j)$. If $H$ is not stochastic, the operation $v\mapsto Hv$ additionally necessitates the suppression of a fraction of tokens, or the multiplication of tokens, at each step at each node, depending on the column-normalization of a given node.\footnote{We provide a graph distributed systems interpretation only for dynamics on graphs with non-negative matrix entries (positive-weight edges). Negative-weight edges are hard to represent using natural local dynamics based on token distribution or spiking models.}

For two graphs $H_1, H_2 \in \R^{n\times n}$, the graph $H = H_2 H_1$ is obtained through (linear algebraic) matrix multiplication, and in a distributed system, the corresponding transition $v \mapsto Hv$ is obtained with two steps of token dynamics, one following graph $H_1$, the next following graph $H_2$. 

Representing $m$ edge-weights of a sparse $n$-node graph with $b$ bits of numerical precision per parameter is possible with $O(m (b + \log n))$ bits of information, which corresponds to $O(m (1 + \frac{\log n}{b}))$ parameters. For the sake of simplicity, we will assume in asymptotics that the second term of the sum does not dominate (i.e., $\log n = O(b)$), and so we simply say that we represent the graph with $O(m)$ parameters.

\section{BDH: a language model architecture given by local distributed graph dynamics}\label{sec:bdhgraph}

\subsection{Formalism for local graph-based language models}\label{sec:formalism}

We consider model architectures $\A$ which correspond to models of graph-based distributed computing (cf.~\cite{Peleg,hirvonen2025distributed}). A specific model $M$ in architecture $\A$ corresponds to the weights and topology of the communication graph or graphs used by such a system.

\paragraph{Introduction to distributed graph systems.} The distributed system architecture $\A$, representing the model architecture, is defined through a \emph{scheduler}, and a local dynamics (\emph{kernel} $K(\A)$) describing the local computations to be performed at each node of the system, and, communication between pairs of nodes connected by edges of the graph representing a given model $M$.

We will generally accept that computations are performed only at $n$ neuron nodes (particles), whereas state variables of the system may appear both on nodes and edges. We will, for simplicity of analysis, consider systems governed by a \emph{synchronous scheduler}, which in successive rounds, acts in two sub-rounds:
\begin{enumerate}
  \item \textbf{Computation:} computations of the kernel of $\A$ are run at all neuron nodes independently.
  \item \textbf{Communication ``over wire''}: each neuron node sends specified `output variables' to specified `input variables' of its neighboring neurons.
\end{enumerate}
We expect the scheduler to follow the same communication pattern between neurons over time in a uniform way. In order to avoid artificial constructions of cyclic time-counters at nodes, we will define the architecture kernel through a short sequence of kernels, with the scheduler executing them in successive rounds in round-robin manner. Specifically, when $\A$ is BDH, we will have a sequence of four kernels, $K(\A)=(K_1(\A),K_2(\A),K_3(\A),K_4(\A))$, with $K_i(\A)$ being executed in every round $r$ such that $r \equiv i \textrm{\, mod \, } 4$.

\paragraph{Programmable rulesets and the interaction kernel.}

We recall from Section~\ref{sec:notation} that a model architecture $\A$ has the interpretation of a computational machine architecture, and models $M$ have the interpretation of programs in architecture $\A$.  We also recall that a graph-based model $M$ is defined through a set of parameters which represent the topology and weights of the communication graph of the system.

The above considerations lead directly to the following observation: \emph{The graph of the communication network, which is used for communication between sites by the distributed system architecture $\A$ during reasoning and language inference, has the interpretation of a (trainable, rule-based) program.} Consequently, we embed the subsequent definition of BDH in a kernel formalism, given through a form of  \emph{programmable rulesets}, using two-particle interaction rules on a graph.\footnote{We refer the reader to Appendix~\ref{sec:b} for a more principled background discussion, guiding the appropriate choice of formalism for rule-based local interaction.}

The rulesets which we will use to define BDH will closely resemble rulesets (protocols) known from evolutionary and population dynamics~\cite{HofbauerSigmund1998,DBLP:journals/dc/AngluinADFP06,Aspnes2009} and chemical reaction networks~\cite{chen2012deterministic,Feinberg2019}, however, they will be restricted to a special class of interactions.  

We start by presenting the more general form of this \emph{interaction kernel}. We then explain how such a kernel can be restricted, allowing it to be naturally implemented using a local graph-based distributed system (in particular, one relying spiking dynamics), while remaining sufficiently expressive to describe an attention-based language model. The resulting restriction will be called the \emph{edge-reweighting kernel}.

\begin{definition}[Interaction kernel, general form]\label{def:chemistry}
A system with $z$ species, $z\in\Nat$, and state $(q_1,\ldots,q_z) \in Q$, $q_i \in R^+$, performs the \emph{interaction kernel with a ruleset (protocol) $P$} given by a set of transition rates called \emph{rule weights}, $P=((r_{ijk} \in R^+)_{i,j,k\in \{1\ldots,z\}}, (d_k \in R^+)_{k\in \{1\ldots,z\}})$, producing the following transition from a state $(q_1,\ldots,q_z) \in Q$ to a state $(q'_1,\ldots,q'_z) \in Q$:
\begin{equation}\label{eq:chemistry}
q_k' := (1-d_k) q_k + \sum_{i,j} r_{ijk} q_i q_j
\end{equation}
We will describe such a ruleset $P$ using the notational form:
$$
P = (\{``q_i, q_j \xrightarrow{r_{ijk}} q_k"\}_{i,j,k\in \{1\ldots,z\}}, \{``q_k \downarrow_{d_k}\!\!"\}_{k\in \{1\ldots,z\}}).
$$
As a matter of convention, omitted rules correspond to $r_{ijk}=0$ (respectively, $d_k=0$), while rules with no rate value stated next the pointer correspond to $r_{ijk}=1$ (respectively, $d_k=1$). If $q_j$ is omitted from notation on the left-hand side, we assume $q_j=1$.
\end{definition}

\Eqeqref{eq:chemistry} captures the dynamics of the following differential equation: $\frac{d q_k}{dt} = -d_k q_k + \sum_{i,j} r_{ijk} q_i q_j$. Assuming $q_i, q_j, r_{ijk} \in [0,1]$, the expression $r_{ijk} q_i q_j$ has the interpretation of a population dynamics or chemical process of the form ``$i$ and $j$ give $k$'', with this processes happening at rate $r_{ijk}$, assuming $q_i, q_j, q_k$ have the interpretation of concentrations of species $i,j,k$. The formalism we use here assumes non-normalized state variables. %

We will subsequently use a restriction of the interaction kernel to graph-based systems, which we call the \emph{edge-reweighting kernel}, to describe BDH.

\paragraph{Restricting the interaction kernel to spiking signals and graph systems.}

First, we observe that rules of the form used in the interaction kernel from Definition~\ref{def:chemistry} are extremely easy to implement in systems which rely on stochastic 0/1-valued signals. When $\hat q_i$ and $\hat q_j$ are independent random variables in $\{0,1\}$, with $\Pr[\hat q_i = 1] = q_i$ and $\Pr[\hat q_j = 1] = q_j$, then $q_i, q_j \xrightarrow{} q_k$ is expressible as the ``AND gate'' of probability: the random variable $\delta \hat q_k := q_i q_j \in \{0,1\}$ gives the same expected contribution $\E\delta \hat q_k = q_i q_j$ as the considered rule.

We now consider the restriction of interaction kernels to the case of graph systems. In the general formalism, $k$ can be arbitrary with respect to $i$ and $j$. By contrast, consider graph systems, which describe binary relations between nodes, and not (directly) three-point relations. To resolve this, we will require that $i$, $j$, and $k$ have the interpretation of two nodes of a graph and an edge which is incident to them.

For an anchoring in the literature of dynamical systems, we note that already  systems following an interaction kernel with a strongly constrained $k$ of the form $k  \in \{i,j\}$, exhibit powerful nonlinearities: with such a restriction on $k$, \Eqeqref{eq:chemistry} describes the class of evolutionary systems following the  equations of \emph{replicator dynamics}~\cite{HofbauerSigmund1998}, also equivalently known as a non-normalized form of the fundamental Lotka-Volterra predator-prey dynamics. Replicator dynamics can naturally be represented as graph systems whose parameters are defined on \emph{on edges of the graph}, but whose state is updated on \emph{on nodes of the graph}. By contrast, when defining dynamics for reasoning in the current work, we will also need to capture a more powerful class of graph-based systems, where, crucially, state is larger than the number of neuron nodes, appearing on neuron-neuron edges (synapses). %

We are now ready to describe a restriction of the interaction kernel from Definition~\ref{def:chemistry} to the case of node-edge-node interaction rulesets in a graph: the \emph{edge-reweighting kernel}.

\paragraph{Definition of the edge-reweighting kernel.} We consider a graph system with $n$ \emph{nodes}, indexed $V=\{1,\ldots,n\}$. Additionally, a subset $E$ of pairs of indexes $(i,j)$, for $i,j\in\{1,\ldots,n\}$ forms the \emph{edges} of the system.

The system has state variables associated (uniformly) with nodes and edges, which we denote with capital letters, e.g., $X(i)$, for $i\in V$ or $Z(i,j)$, for $(i,j) \in E$.

\begin{definition}[edge-reweighting kernel]\label{def:edgereweighting}
A distributed system follows the \emph{edge-reweighting kernel} if its dynamics are given by the interaction kernel (Definition~\ref{def:chemistry}) with a set of non-negative state variables, defined on the set of nodes $V$ and set of edges $E$ of a graph, such that each local rule with non-zero rate is either a \emph{computational rule} involving only state variables on a single node $i \in V$, or a \emph{communication rule} for an edge $(i,j) \in E$, involving state variables from the nodes $i, j$ and edge $(i,j)$.
\end{definition}
For context, we remark that, in comparison to the strictly simpler dynamics of node-reweighting governed by graph-based replicator dynamics equations, dynamical systems based on the edge-reweighting kernel given by Definition~\ref{def:edgereweighting} are rather elusive to study. We credit the seminal work of Algorithms theory~\cite{madry2011}[Fig.~1, Thm~3.2] as the first rigorous study of local edge-reweighting graph dynamics, combining fast-paced linear kernels on nodes with a slower-paced edge-reweighting process, in order to refine (`focus') electrical flows on graphs towards a sharper form of cost optimality.\footnote{The graph dynamics used in this approach are naturally phrased in distributed computing parlance, see~\cite{natale2018,zou19}.} The BDH dynamics that we will introduce here rely on fundamentally different nonlinearities in the process, and will have the interpretation of guiding the system from premises defined at a subset of nodes, towards search targets at nodes representing a desired outcome, through reasoning inference rules with tunable weights set on edges. %

In the following Subsection, we will use the introduced formalism to define BDH as an edge-reweighting kernel on the union of edges of several graphs ($\graphx\ee, \graphx\ii, \graphy\ee, \graphy\ii, \graphs$) with the same set of $n$ nodes. %

\subsection{Definition of BDH as a local edge-reweighting process (equations of reasoning)}\label{sec:equations_of_reasoning}

Bearing in mind the discussion of graph dynamics suitable for the case of language inference, and specifically the definition of the edge-reweighting kernel (Definition~\ref{def:edgereweighting}), we are now ready to formalize the state-space dynamics of \Eqeqref{eq:bdhgraph} as a local graph dynamics.

\begin{definition}
The BDH model with $n$ neurons, with parameters expressed through graphs $\graphx\ee, \graphx\ii, \graphy\ee, \graphy\ii, \graphs$ is represented as the ruleset of the edge-reweighting kernel, with $O(n + |E(\graphs)|)$ state variables, with rule amplitudes given by ``the equations of reasoning'' in Table~\ref{tab:protocolx}.
\end{definition}

\paragraph{Inference dynamics of BDH.} The BDH dynamics rely on rapid pulse dynamics with state variables $X(i)$, $Y(i)$, $A(i)$, defined on the $n$ neuron sites of the system, and fast-weight-like state variables $\sigma(i,j)$, defined on a subset of edges of the system, $(i,j) \in E(\graphs)$. The full implementation of BDH shown in Table~\ref{tab:protocolx}(b) also includes auxiliary state variables $X\ee(i)$, $X\ii(i)$, $Y\ee(i)$, $Y\ii(i)$ which are used as temporary counters, for integration of excitatory and inhibitory signals received by neurons. The dynamics also rely on a set of damping hyperparameters on state, $u>0$, which may in full generality be defined separately as $u(i,j)$ for each edge $(i,j) \in E(\graphs)$.

Inference with BDH is performed as follows. For some parameter $L$ (e.g.\ $L=8$ in most of this paper), which would correspond to the number of layers in a Transformer-like system, the system scheduler proceeds through rules in round-robin manner, ingesting new tokens every $4L$ rounds and retrieving results $4L$ rounds later. During round $4l + k$, for $0\leq l < L$, the system performs rules from the $k$-th column of Table~\ref{tab:protocolx}, with each such round consisting of a communication step on edges and a local computation step on nodes. 

The state-space dynamics of BDH can be rewritten in vector-tensor form, equivalent to the local dynamics of the interaction kernel given in Table~\ref{tab:protocolx}. This representation is given by Equation~\eqref{eq:bdhgraph} in the following Section.

\begin{observation}\label{obs:protocol_equiv}
The BDH-Graph protocol for the interaction kernel, given for any time round $T=4Lt +(4l+k)$, $0\leq l < L$, $k=\{0,1,2,3\}$ by the ruleset in Table~\ref{tab:protocolx} is equivalent to the state-space dynamics over time $t$ and layers $l$, given by Equation~\eqref{eq:bdhgraph}.
\end{observation}
\begin{proof}
For completeness, a detailed explanation of the equivalence is provided in Appendix~\ref{sec:protocol_equiv}.
\end{proof}

The variables $X(i)$, $Y(i)$, $A(i)$, defined for each of the $n$ nodes of the system, are updated in successive rounds. The state variables $\sigma$ defined on edges are assumed to be distinct over $l$ as $\sigma_l$, for $0\leq l < L$; this distinction serves to facilitate interpretation and to strike a balance between the number of parameters and the size of state of the system (assuming a single state matrix $\sigma$, uniform across $l$, does not fundamentally change the operation and scaling laws of the architecture).

\begin{table}
\small
\noindent
(a) Simple equations of reasoning\\[2mm]
\hspace*{-0.015\textwidth}
\begin{tabularx}{1.03\textwidth}{|X|X|X|X|}
\toprule
\centerline{Round $4l$} & 
\centerline{Round $4l+1$} & 
\centerline{Round $4l+2$} & 
\centerline{Round $4l+3$}\\[1mm]
\centerline{\emph{Inference from state}} &
\centerline{\emph{Reweighting of synapse state}} & 
\emph{\centerline{Neuron replicator dynamics +}\newline\centerline{inference from parameters}}& 
\centerline{\emph{Inference from parameters}}
\\
\midrule
\begin{align*}
\ruletwo{X}{i}{\sigma_l}{i,j}{}{A}{j}\\
\ruledown{\sigma_l}{i,j}{1-u(i,j)}
\end{align*}
&
\begin{align*}
\ruletwo{Y}{i}{X}{j}{G_s(i,j)}{\sigma_l}{i,j}\\
\ruledown{Y}{i}{}
\end{align*}
&
\begin{align*}
\ruletwo{A}{i}{X}{j}{G_y\ee(i,j)}{Y}{j}\\
\ruledown{A}{i}{}
\end{align*}
&
\begin{align*}
\ruleone{Y}{i}{G_x\ee(i,j)}{X}{j}
\end{align*}
\\
\bottomrule
\end{tabularx}\\[5mm]
(b) Complete equations of reasoning of BDH\\[2mm]
\hspace*{-0.015\textwidth}
\begin{tabularx}{1.03\textwidth}{|X|X|X|X|}
\toprule
\centerline{Round $4l$} & 
\centerline{Round $4l+1$} & 
\centerline{Round $4l+2$} & 
\centerline{Round $4l+3$}\\
\midrule
\multicolumn{4}{|c|}{Communication}
\\[1mm]
\begin{align*}
\ruletwo{X}{i}{\sigma_l}{i,j}{}{A}{j}
\end{align*}
&
\begin{align*}
\ruletwo{Y}{i}{X}{j}{G_s(i,j)}{\sigma_l}{i,j}
\end{align*}
&
\begin{align*}
\ruleone{A}{i}{G_y\ee(i,j)}{Y\ee}{j}\\
\ruleone{A}{i}{G_y\ii(i,j)}{Y\ii}{j}
\end{align*}
&
\begin{align*}
\ruleone{Y}{i}{G_x\ee(i,j)}{X\ee}{j}\\
\ruleone{Y}{i}{G_x\ii(i,j)}{X\ii}{j}
\end{align*}
\\
\multicolumn{4}{|c|}{$ $}\\[-3mm]
\multicolumn{4}{|c|}{Computation}
\\[1mm]
\begin{align*}
\ruledown{\sigma_l}{i,j}{1-u(i,j)}\\
\ruledown{X\ee}{i}{}\\
\ruledown{X\ii}{i}{}
\end{align*}
&
\begin{align*}
\ruledown{Y}{i}{}\\
\ruledown{Y\ee}{i}{}\\
\ruledown{Y\ii}{i}{}
\end{align*}
&
\begin{align*}
\rulethreedown{Y\ee}{i}{Y\ii}{i}{X}{i}{Y}{i}
\hspace*{-0.005\textwidth}\\
\ruledown{A}{i}{}
\end{align*}
&
\begin{align*}
\ruletwodown{X\ee}{i}{X\ii}{i}{X}{i}
\end{align*}
\\
\bottomrule
\end{tabularx}
\ \\
\caption{The ``equations of reasoning'': State-space dynamics of the BDH language model expressed through local graph dynamics with the edge reweighting kernel (Definition~\ref{def:edgereweighting}). The rules are executed for a distributed system of $n$ neurons performing steps of parallel computation and communication during inference. Model parameters are expressed through the weights of edges of graphs $\graphx\ee, \graphx\ii, \graphy\ee, \graphx\ii, \graphs$, and BDH model training is equivalent to defining rule probability amplitudes $\graphx\ee(i,j), \graphx\ii(i,j), \graphy\ee(i,j), \graphy\ii(i,j), \graphs(i,j) \geq 0$ for pairs of neurons $i,j \in \{1,\ldots,n\}$ connected by the edges of these graphs. State is encoded in variables $\sigma(i,j)$ at synapses, representing edges of graph $G_s$. The system proceeds in parallel rounds, with new tokens arriving into the system encoded through variables $X(i)$ at neurons and introduced every $4L$ rounds, where $L$ is a parameter of the model (e.g., $L=8$). The set of rules being executed (for each round modulo $4L$) is given in the table. The readout of the system also happens through variables $X(i)$ at the end of each $4L$ rounds.
(a) Set of rules for the simplified version of the BDH model with no neuron inhibitory circuits and no thresholding ($\graphx\ii=\graphy\ii=0$), capturing the general form of the communication structure and synaptic attention of the model. (b) Set of rules for the general case of BDH, including inhibitory circuits $\graphx\ii$, $\graphy\ii$. An execution of the provided rules is equivalent to the state-space dynamics given by \Eqeqref{eq:bdhgraph}.
}
\label{tab:protocolx}
\end{table}

In the representation in Table~\ref{tab:protocolx} we do not impose how the local thresholding operation within some neuron $i$, of the form $\ix{A}{i}\ ,\ix{B}{i}\dashrightarrow\relu{\ix{A}{i}-\ix{B}{i}}$, should be performed. We leave this as a computational primitive, which can be realized based on approximate counting or a comparator. The way natural neurons achieve thresholding to determine whether input signal excitation outweighs inhibition relies on time-integration of impulses. For realizations in other types of distributed systems and population protocols, we refer the reader to the literature on thresholding and Majority Protocols, cf.~e.g.~\cite{DBLP:conf/focs/DotyEGSUS21,
DBLP:journals/jcss/CzyzowiczGKKSU22}.%

The definition of the protocol does not specify how variable $X(i)$ should be reset when the scheduler passes from layer $L$ of one input token to layer $0$ for the next input token. As with the definition of state-space equations in Section~\ref{sec:BDHGPU}, we leave this open to allow the dynamics to work both with externally provided input (for next-token prediction), or in a self-feedback loop (for autoregressive operation).

\paragraph{Notes on training.} Direct training of the BDH model would be performed by selecting the edges of the considered graphs, and then setting rule weights $\graphx\ee(i,j), \graphx\ii(i,j), \graphy\ee(i,j), \graphy\ii(i,j), \graphs(i,j) \geq 0$ for pairs of neurons $i,j \in \{1,\ldots,n\}$ connected by the edges of these graphs. %

In what follows, we will train a tensor-friendly  special case of BDH, called \BDHGPU, relying on an implicit (generally more efficient) representation of the considered graph parameter weights, using a low-rank product representation for the matrices of these graphs. This representation is reminiscent of the hub-labeling graph representation technique, but is directly suitable for describing and evolving high-conductance scale-free networks. The appropriate architecture is introduced in Section~\ref{sec:BDHGPU}.

\subsection{Interpretation of attention as a micro-inductive bias of reasoning}\label{sec:interpretation_attention}

Rule weights in the edge-reweighting kernel have the interpretation of micro-programs, governed by rules of transformation of state variables of the form $A(i),B(j)\to \sigma(i,j)$ and $A(i),\sigma(i,j)\to C(j)$, defined on edges between nodes $i,j$ of some $n$-node graph.

This formalism can be seen as running an enormous circuit with a form of universal gates given by the transition rules, over a structure of computational elements at nodes, and memory elements on edges of a graph.

While the local rulesets have the form of a rule-based micro-assembly, we leave open the extent to which they should be considered to have an interpretation of programming in logic (as would be the case, e.g., for C-RASP~\cite{yang2024countingliketransformerscompiling}). The natural interpretation of $\sigma(i,j) >0$ is a positive bias associated with the neuron pair $(i,j)$, $i,j \in \{1,\ldots,n\}$, which follows from past context. This can be considered by phrasing the local rules of the system in a framework of logic inference; we do so informally, omitting discussion of layers.

\begin{quote}
If past context $(x_\tau : \tau < t)$ implies that implication $i \to j$ has weight $\sigma_{t-1}(i,j)$, and if the current state at time $t$ implies that $i$ follows from this state with weight $x_{t}(i)$, then the current state at time $t$ implies that $j$ follows from this state with weight $x_{t}(i) \sigma_{t-1}(i,j)$.
\end{quote}

The above is intentionally phrased to resemble the logical axiom $(X\to(i \to j))\to((X\to i) \to (X\to j))$, which is perhaps most prevalent across different formalizations of axiomatic logic, with an application of \textsl{modus ponens} as an inference rule. The inference system of the considered model uses state and model weights to devise its own heuristic for the order of evaluation, i.e., to consider which facts appear to be most plausible to be evaluated next, and to evaluate them in an order based on what follows most strongly from context. In a way consistent with what we expect from informal reasoning in language, the considered weights have a more direct interpretation of an increment of utility associated with a given inference.\footnote{Here, the term \emph{utility} is understood in the sense of evolutionary game theory, as applied to the population of neurons, considering the standard interpretation of replicator dynamics, as applied in the ruleset from Table~\ref{tab:protocolx}. Neurons which win in the natural selection process are added to the activation $Y$.} In the setting of argumentation, this utility-based approach could, for example, guide the inference process from a pair of known concepts in context, a source and a target, to an intermediate concept likely to be a common-neighbor shortcut lying on a logical path between this source and target (cf.~Section~\ref{sec:bdh_feedforwardpropagation} for a discussion of how this type of mechanism is enforced in the feed-forward network of \BDHGPU).

The considered micro-foundational interpretation of attention, defined at the level of individual neurons (or logical variables), does not contradict the way in which Transformer attention is often regarded at the coarser level of vectors through key-query lookup intuitions. At the same time, it highlights that an attention state entry $\sigma(i,j)$ (and similarly, a model edge weight leading from $i$ to $j$) does not have the interpretation of a logical value (i.e., something that is true or false), but an inductive bias associated with how likely the system is to consider the implication `$i\to j$' in its next steps of reasoning, when proposing its next conclusions or next ideas for consideration.

Chains of implications in BDH guide activations along paths in the system graphs $\graphx\ee, \graphy\ee, \corr$. For the latter, attention allows specific implications to enter into paths of thought once the corresponding synapses are open in state $\corr$.

\subsection{Interpretation of BDH as an oscillator network toy-model}\label{sec:toy}

Whereas the interpretation from Subsection~\ref{sec:interpretation_attention} focuses on properties which fallow from the computational function (purpose) of the system, here we outline an interpretation of the behavior of BDH considered purely as a dynamical system.

\paragraph{Definition of the toy-model.} We will consider the toy-model of an $n$-particle system shown in \figref{fig:toy} as an illustration of the general form of dynamics of the state-space equation~\eqref{eq:bdhgraph} of BDH. We draw the $n$ particles in a circle.\footnote{This is a direct tribute to the Kuromato coupled oscillators model; the crucial difference being that in BDH, the elements of state with an interpretation similar to oscillators appear on connections between nodes, not nodes.}

\begin{figure}
\centering
\begin{minipage}{0.40\textwidth}
\includegraphics[width=\textwidth]{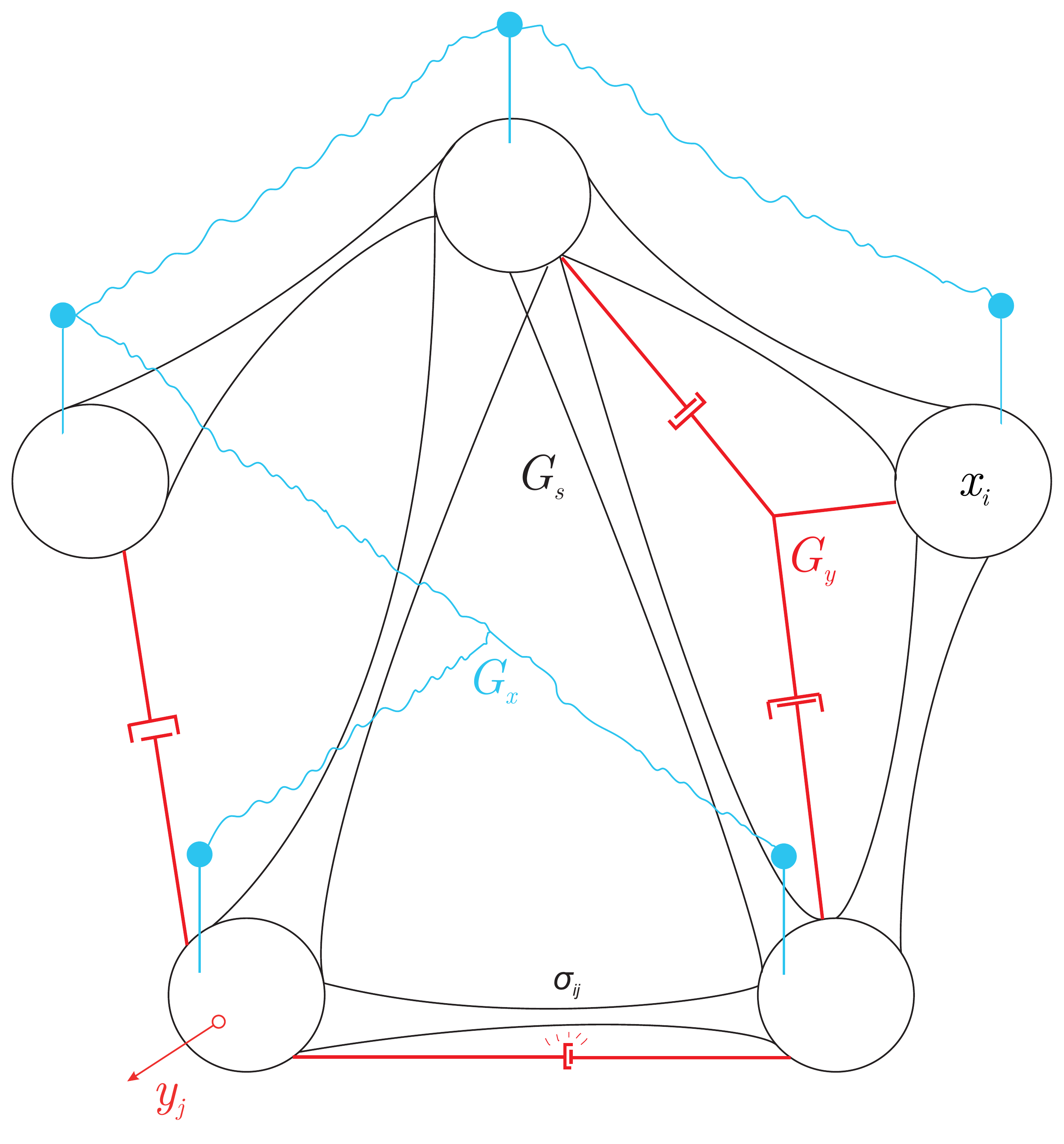}
\end{minipage}
\hspace*{0.01\textwidth}
\begin{minipage}{0.57\textwidth}
\setlength{\extrarowheight}{6pt}
\begin{tabularx}{\textwidth}{lYY}
\toprule
\emph{Symbol} & \multicolumn{2}{c}{\emph{Interpretation in:}}\\
 & Table~\ref{tab:protocolx}, \mbox{State~\Eqeqref{eq:bdhgraph}} & Oscillator Network Toy-Model\\
\midrule
$\graphx, \graphy, \graphs$ & graph parameters of model& wires, prods, and elastic connections\\
$\sigma$ & synaptic state of model & displacement of elastic connections\\
$\xsparse, \xysparse$ & activation vectors & pulses at nodes, state correction\\
\bottomrule
\end{tabularx}
\end{minipage}
\caption{The `physical system' representation of BDH as a physical graph toy-model.}
\label{fig:toy}
\end{figure}

The particles are connected with each other by state elements, represented in \figref{fig:toy} as elastic connectors. The topology of these pairwise connections is given by graph $\graphs$, and may in general be dense.

The signal displays dynamics of state $\state$ through tension on connectors, which evolves at a slower time scale, and a more pulse-like activation dynamics $\xsparse, \xysparse$ (on nodes), appearing and vanishing regularly, at a rapid time scale.

The slower state dynamics represent, in the first order, oscillation or relaxation of the system of elastic connectors. Once an elastic connector between particles $i$ and $j$ has had its endpoints displaced through state $\xsparse$ and $\xysparse$, respectively, a tension appears on this connector, which causes its displacement $\corr(i,j)$ that relaxes over time (damping variant, corresponding to ALiBi), and/or acts as a spring element (oscillator variant, a simplified illustration of RoPE). Initially, $\corr(i,j) = 0$.

The faster dynamics represent the node dynamics of particles. Over time, pulse displacements $x(i)$ happen at nodes, as a result of either previous behavior of the system, or perturbation by an external forcing field (in reality this field would be language input). A node $i$ with displacement $x(i)$ may, due to the aggregated action of tension of \emph{elastic connectors} $\corr(i,\cdot)$ adjacent to it, activate a system of \emph{prods} $\graphy$ adjacent to it, perturbing nodes it hits in this way. If another node $j$ is prodded sufficiently hard, it may cause it to activate a perturbation $y(j)$. The perturbation $y(j)$ of a node $j$ will, in the next step, propagate again to those other nodes $i'$, which are connected to $j$ by a system of \emph{wires} ($\graphx$). If the aggregated pull of wires on a node $i'$ is sufficiently strong, this modifies its pulse displacement $x(i')$. The pulse activation $y(j')$ of some node $j'$, directly followed by pulse activation $x(i')$ of node $i'$, results in an increase in the tension on the connector $(i,j)$, adding to the value of the tension $\corr(i',j')$. All pulse activations $y$ subside, and the pulses propagate, consequently altering the slow state $\corr$.

In general, $\corr(i',j')$ is triggered simply by the temporal connection between the pulse $y(j')$ activating, followed by the pulse $x(i')$ activating immediately afterwards, even if there was no direct causality between the two (although $y(j')$ contributed to pulse $x(i')$ happening if $(j',i') \in \graphx$). An appropriate correspondence of the graphs, $\graphs \subseteq \graphx$, would bring the system close to an observed causal effect on the activated synapse.

The above description of the pulse dynamics was given from the perspective of nodes. From the perspective of connectors, an existing tension on some connector $\corr(i,k)$ propagates through prods $\graphy$ to some nodes $j$, then through wires $\graphx$ to some nodes $i'$, and this finally contributes to tensions on other connectors $\corr(i',j')$. This propagation of state thus happens to 3-hop neighbors, through $i$, $j$, $i'$.

During training, the behavior of the system may, in even longer time scales, result in the propagation of changes of connection weight and structures to graphs $\graphx$ and $\graphy$, as well as (optionally) $\graphs$.

\paragraph{Effects captured by the toy-model.}
We have described a small local graph kernel, with 3-hop locality, capturing the two key effects of the local graph kernel.

The first effect is the graph form of communication pattern between nodes, and thresholding of updates. (We have omitted direct mention of inhibition from discussion of the toy-model, but it is direct to include.)

The second effect is the placement of attention state on node connections, its update patterns, and the dynamics of its relaxation over time.

We intentionally convey the interpretation of node pulses as a differential (gradient) of state on node connections. This interpretation is consistent with our empirical study from Section~\ref{sec:experiments}. It is worth considering once every how many steps of the operation of the toy-model, a single element of state $\corr(i,j)$ is updated. This depends directly on the sparsity of the pulse signals $y(i)$, $x(j)$; at least one of them is, generally, sparse. If the pulses where to happen very seldom for such a pair $(i,j)$, state updates are essentially a ``second-order'' correction effect. By adjusting the frequency of updates, the system can be made to operate exactly at the critical point where this pulse dynamics ceases to be a second-order correction of state $\corr(i,j)$, giving the random variable describing the time between updates of a connection pair $\corr(i,j)$ a heavy power-law-like tail distribution (possibly with different distribution parameters for different pairs $(i,j)$).

In the description of state dynamics, we noted the hop-distance of 3 in the forward propagation of changes to state. Bearing this in mind is helpful when considering how a gradient backpropagation mechanism would follow dependencies between changes of state if such a system were to have its graph weights altered through backpropagation.

Finally, let us clarify the specific choice of kernel we made for BDH. We found it to work well, and we knew how to train BDH models which implement it on GPU (which we will call \BDHGPU). This, with current hardware, made it $10^2-10^5$ times more cost- and time-effective to train models and analyze outcomes than kernels, for which we only knew how to train on CPU. Nonetheless, the question of finding optimal kernels according to different criteria (e.g.: minimality of kernel, best training rate per token, closeness to brain function based on known evidence from brain studies), is an extremely pertinent foundational problem. The problem can be phrased in a ``closed-ended'' way, leaving a finite number of possibilities to be checked, at least when considering small graph kernels. Some kernels may also prove to have superior learning capabilities to the Transformer (and BDH), and if this quality difference is overwhelming, they may eventually prove commercially viable.

In the following, we formalize the choice of kernel for BDH, and also provide a framework to describe other kernels capturing the same effects of graph communication and synaptic attention.

\subsection{Expressing BDH using brain models}\label{sec:brain_models}

The results we obtain for BDH provide direct corollaries on the expressiveness of brain models which are capable of emulating the local graph kernels of BDH. Specifically, a distributed system, which is able to efficiently emulate the local kernels of BDH, has sufficient expressiveness to perform language inference and reasoning at least to the same extent as BDH.

\begin{observation}\label{obs:brain}
The local ruleset of BDH (Table~\ref{tab:protocolx}) can be expressed through a combination of simple mechanisms: neuron activation with positive state variables, Hebbian learning, and communication through excitatory and inhibitory circuits with thresholding.
\qed
\end{observation}

We note that in the description of the rulesets in Table~\ref{tab:protocolx}, Round~($4l+2$) and ($4l+3$) directly describe the use of excitatory and inhibitory circuits with integrate-and-fire thresholding at neurons. Round~($4l+2$) additionally includes a form of competition effect between neurons, realized fully locally at a neurons using the multiplication effect of replicator dynamics. The communication rule of Round~($4l+1$) involves the potentiation of a synapse based on activations of neurons at its endpoints. As was discussed in Subsection~\ref{sec:formalism}, the natural mechanism for implementing increase in synaptic strength is through spiking dynamics, where the execution of the communication rule of Round~($4l+1$) is a stochastic AND-gate on signals. Finally, Round~($4l$) describes the long-term effects of using a strengthened synapse for transmission of signals, and its strength decrease.

We can use the framework of expressiveness, as captured in Observation~\ref{obs:brain}, to shed light on the capabilities of natural systems through their ability to emulate artificial ones. Specifically, if a natural system A can plausibly emulate some artificial system B by using the resources it has at its disposal, and artificial system B is able to solve a problem P, this can be used to explain: (1) why the natural system A is sufficiently powerful to solve problem P, and (2) plausibly, that the purpose for which system A is equipped with certain mechanisms includes solving problem P, if such mechanisms prove useful in the emulation of B.

The experimental validation of the performance of BDH architecture at Transformer level (Section~\ref{sec:comparison_bdh_gpt_transformers}) confirms that BDH is sufficient to provide language and reasoning function at scale. We can thus make the following statement.

\begin{finding}
The \emph{Hebbian learning mechanism} is plausibly needed, and in combination with neural circuits, sufficient, for performing the \emph{reasoning} function at the scale of the brain. This includes performing language function with attention, and performing thought processes, at a time scale of minutes.
\end{finding}

In view of our results, Hebbian learning can be seen as a form of unsupervised learning over time, expressed through graph edge reweighting, to perform reasoning and language inference using the attention mechanism. This type of result can be compared to an analogous interpretation for Hebbian learning in the context of vision, as pioneered in~\cite{brunel1996}. With the setting of language and chain-of-thought reasoning, we are able to directly capture effects of time in the brain.

Given the interpretation of neuron activations as carrying the necessary gradients of synaptic state (Section~\ref{sec:toy}), the problem of supervised learning (i.e., taking into account feedback signals) plausibly becomes deferred to a selective transfer and re-encoding of gradients from state into weights, at longer time scales. We return to a discussion of this point in the Conclusions, bearing in mind the fact that the general difficulty of the problem is now reduced through restrictions on the considered edge-reweighting kernel, and the relative rarity of synapse activation events.

Our work also suggests a framework for further discussion of reasoning function, with an anchoring point for this type of investigation in the time-scale of `split-seconds' to `minutes'. The question of shorter time scales is then one of designing more precise communication and computational primitives for spiking neurons and synaptic plasticity, which can be used to perform primitives for individual rules of graph kernels for the inference dynamics.\footnote{While we do not provide direct explanations for effects at shorter time scales and scheduler primitives, we note the type of kernels we rely on are well understood in terms of the ability to work with asynchronous schedulers, and obtaining emergence of synchronization.~\cite{DBLP:conf/podc/KosowskiU18,DBLP:conf/stoc/DudekK18}} The question of longer time scales, and the changes to model structure that follow in a learning process, naturally follows any explanation of unsupervised (Hebbian) learning from the shorter time scale that is considered here, as a mechanism of transfer from state to weights; we come back to this point in the Conclusions.

\section{\BDHGPU: a tensor-friendly version of the BDH architecture}\label{sec:bdh_arch}\label{sec:BDHGPU}

We will now introduce \BDHGPU, a variant of the BDH reasoning system, expressed in the language of tensor operations typical for Deep Learning models. \BDHGPU provides a GPU-compatible implementation of BDH. \BDHGPU can be easily implemented in PyTorch, a didactic code listing is provided in Appendix \ref{sec:bdh_code_listing}). Furthermore, \BDHGPU can be trained on large text datasets using error backpropagation, and has been shown experimentally to match performance of GPT-based LLMs.

The main steps towards the efficient implementation of \BDHGPU on GPU are:
\begin{enumerate}
    \item Express graphs $G_x$ and $G_y$ a low-rank factorizations of their transition matrices, followed by ReLU nonlinearities \cite{nair2010rectified} (we explore graph properties of these approximations in Section~\ref{sec:modularity}). We never materialize these matrices, but maintain instead a low dimensional state per each neuron.
    \item Never materialize the $\sigma$ state matrix, preferring instead to access it using a linear attention operation over low-rank representation of values (we explore the properties of this attention mechanism in Section~\ref{sec:bdh_lin_attention_monosemanticity}).
    \item Normalize all state variables using LayerNorm~\cite{ba2016layer}.
\end{enumerate}
We will refer to the architecture in the final intermediate step, before the introduction of LayerNorm, as BDH-Normfree.

\subsection{Notation for \BDHGPU} 

We consider the $\BDHGPU(n,d)$ architecture parameterized by positive integers $n, d$. The system scales in dimension $n$ --- the number of particles. In what follows, we will use the terms \emph{particle} and \emph{neuron} interchangeably. Dimension $d$ is a measure of the number of parameters per neuron required to represent the interaction of this neuron with the particle interaction field or interaction graph. For asymptotic analysis, we assume that $n \to +\infty$ is the basis for all asymptotics, and $n \gg d > C \log n$ holds for some sufficiently large constant $ C > 0$. For the tensor representation of the model, which is the primary one for implementation and empirical studies here in this paper, vectors in $R^{d}$ have an interpretation as (fuzzy) addresses of a virtual memory space of size $n$, hence the assumption $d = \Omega(\log n)$ cannot be dispensed with while using natural (linear-algebraic) arithmetic on real numbers. We later show how to avoid this assumption in graph-based models, by using uniform local graph kernels of smaller degree with a graph communication structure.

\subparagraph{Nonlinearities: ReLU and LayerNorm.} In what follows, we assume that a one-dimensional vector is denoted by a lower-case letter, e.g., $z$, with $z \in R^{n\times 1} \cong R^{n}$ unless otherwise stated. Vectors which appear in dimension $d$ are named with an asterisk, e.g., as $z^* \in R^{d\times 1}$. We denote the \emph{ReLU operation} $\relu{z} := \max_{i\in{1,\ldots,n}}\{0,z_i\}$.

We further define \emph{LayerNorm} of a vector $z^* \in R^{d \times 1}$ in a uniform non-parametric way, $\layernorm{z^*} = \frac{z^* - \mathbf{1}\E_d z^*}{\std_d z^*}$, where $\E_d$ and $\std_d$ are estimators of mean and standard deviation in dimension $d$, respectively. 

\subparagraph{Activation vectors and parameter matrices.} In vectors representing activations, each scalar element (element of $R$) of the activation vector has the interpretation of a `scalar' activation state of a single particle. Throughout this text, $R$ is generally assumed be the field of real numbers $R:=\R$, and scalars are represented by a fixed-precision floating point number in experiments.\footnote{When only asymptotic analysis is the object, it is sometimes convenient to consider $R := \R^k$ for some $k=2,3,\ldots$. Specifically, considering $R:=\R^2$ allows $SO(2)$ rotations on $\R^2$ to be expressed as `scalar' ones on $R$, thus making the $\R^{2n \times 2n}$ RoPE block-diagonal matrix of a diagonal matrix in $R^{n\times n}$ ~\cite{su2023roformerenhancedtransformerrotary}. This provides a consistent formalism for ALiBi~\cite{press2022trainshorttestlong}, RoPE, and extensions such as LieRE~\cite{ostmeier2025lierelierotationalpositional} as diagonal (communication-free) operations. In all cases, the application of ReLU $\relu{\cdot}$ to a scalar remains coordinate-wise in $\R$.}

By convention, in discussions of parameters, matrices denoted $\graphx, \graphy, \graphs \in R^{n \times n}$ will represent neuron-neuron interaction, encoders $\encoder \in R^{d \times n}$ reduce dimensionality of activation vectors (e.g., $a^* = \encoder z$ for $z \in R^n$), and decoders $\decoder \in R^{n \times d}$ lift them back into $R^n$ (e.g., $z' = \decoder a^*$).

Depending on the architecture variant considered, the state will either have the interpretation of a neuron-neuron correlation matrix $\corr \in R^{n \times n}$, or a compressed form with reduced dimensionality, $\state = E \corr \in R^{n \times d}$.

\subsection{Definition of BDH-GPU as a state-space system}

We now define the main architecture of this paper in its tensor flavor, called \BDHGPU.

\begin{definition}[inference dynamics of \BDHGPU]\label{def:bdh}
A \BDHGPU state-space system $\BDHGPU(n,d)$, given by three parameter matrices: $\encoder \in R^{d \times n}$ and $\decoderx, \decodery\in R^{n \times d}$, performs iteration over time $t=0,1,2\ldots$ and layers $l=1,2\ldots L$, governed for any time $t$ by the following recurrence:
\begin{align}
\begin{split}\label{eq:integral}
\ket{\xsparse_{t,l}} &:= \xsparse_{t,l-1} +
    \relu{\decoderx \ket{\vv_{t,l-1}}}\\
\ket{\yKV_{t,l}} &:= \sum_{\tau<t}
    \ket{\vv_{\tau,l-1}}
    \bra{\xsparse_{\tau,l}}
    \rope^{t-\tau}
    \ket{\xsparse_{t,l}}
    \\
\ket{\xysparse_{t,l}} &:=
    \relu{\decodery \layernorm{\ket{\yKV_{t,l}}}}
    \*
    \ket{\xsparse_{t,l}}\quad\\
\ket{\vv_{t,l}} &:=
    \layernorm{ \encoder \ket{\xysparse_{t,l}}}
\end{split}
\end{align}
where inputs to the system are provided through the boundary condition $\vv_{\tau,0}$ in layer $0$, for $\tau=0,1,2\ldots t$.

Here, $U \in R^{n\times n}$ is a diagonal or block-diagonal matrix representing local rotation or damping of state (such as ALiBi or RoPE), $L\in \Nat$ is the number of layers.%
\end{definition}

\paragraph{\BDHGPU as a language model.} \BDHGPU is intended to be used as a language model, processing one token per time step, in which case the input $\vv_{t,0}$, for $t \in \Nat$, is obtained using some (linear) encoding function from the token alphabet $\Omega$, $f_e : \Omega \to R^d$, as applied to the $t$-th input tokens. Similarly, the logits of the $t$-th output token are extracted using some decoding function applied to outputs of the $L$-th layer $\vv_{t,L}$, using a (linear) token decoder function $f_d : R^d \to \Omega$. The source of language tokens may be external, as is the case for next token prediction tasks, or auto-regressive.

For training, we assume that a model $M$ trained in the $\BDHGPU(n,d)$ architecture has the trainable parameter set $M = (\encoder, \decoderx, \decodery, f_e, f_d)$, with all parameters trained together. The model has $3 nd + 2 \Omega d = (3+o(1)) nd$ parameters, i.e., the scalable part of the model is concentrated in the total of $3nd$ parameters of the matrices $(\encoder, \decoderx, \decodery)$.

\paragraph{State-space representation.} The notation of Definition~\ref{def:bdh} is chosen so as to exhibit its direct applicability in a Transformer-like token-parallel training framework. Vector $\vv_{\tau,l-1}$ has the interpretation of attention `value' inputs at time $\tau$ in layer $l$. Vector $\yKV_{t,l}$ represents the result of a linear attention mechanism for time $t$ in layer $l$.

Denoting in~\eqref{eq:integral} the model's attention state as
\begin{equation}
\label{eq:kvstate}
\state_{t-1,l} =
    \sum_{\tau < t}
        \ket{\vv_{\tau,l-1}}
        \bra{\xsparse_{\tau,l}}
        \rope^{t-\tau}
\end{equation}
we obtain the equivalent but more compact form of representing the inference dynamics of \BDHGPU as a state-space model, presented in \figref{fig:ss}, \eqeqref{eq:bdh}.

\mathleft
\begin{figure}[ht!]
\centering
\fbox{\begin{minipage}[t][0.12\textheight]{0.47\textwidth}
BDH
\begin{align}
\begin{split}\label{eq:bdhgraph}
\corr_{t,l} &:= \lr{
    \corr_{t-1,l}
+
    \lr{
        \lr{
            \ket{\xysparse_{t,l-1}}
            \bra{\xsparse_{t,l}}
        }
        \odot
            \graphs
    }
}\rope\\
\ket{\xsparse_{t,l}} &:= \xsparse_{t,l-1} +
    \relu{\lr{\graphx\ee - \graphx\ii} \ket{\xysparse_{t,l-1}}}\\
\ket{\xysparse_{t,l}} &:=
    \relu{
        \lr{\graphy\ee - \graphy\ii}
            \corr_{t-1,l}
            \ket{\xsparse_{t,l}}
    }
    \*
    \ket{\xsparse_{t,l}}
\end{split}
\end{align}
\addtocounter{equation}{1}
\end{minipage}}
\hspace*{0.6\baselineskip}
\fbox{\begin{minipage}[t][0.12\textheight]{0.47\textwidth}
\BDHGPU
\begin{align}
\begin{split}\label{eq:bdh}
\state_{t,l} &:= \lr{
    \state_{t-1,l}
+
    \layernorm{
        \encoder \ket{\xysparse_{t,l-1}}
    }
    \bra{\xsparse_{t,l}}
}\rope\\
\ket{\xsparse_{t,l}} &:= \xsparse_{t,l-1} +
    \relu{\decoderx \layernorm{ \encoder \ket{\xysparse_{t,l-1}}}}\\
\ket{\xysparse_{t,l}} &:=
    \relu{
        \decodery
        \layernorm{
            \state_{t-1,l}
            \ket{\xsparse_{t,l}}
        }
    }
    \*
    \ket{\xsparse_{t,l}}
\end{split}
\end{align}
\addtocounter{equation}{-2}
\end{minipage}}
\\[0.005\textwidth]
$\searrow$\hspace*{0.3\textwidth}$\nearrow$
\\[0.005\textwidth]
\fbox{\begin{minipage}[t][0.14\textheight]{0.48\textwidth}
BDH-Normfree
\vspace*{-0.02cm}
\begin{align}
\begin{split}\label{eq:bdhnoln}
&\hspace*{-0.65cm}\begin{rcases}
\corr_{t,l} &:= \lr{
    \corr_{t-1,l}
+
    \ket{\xysparse_{t,l-1}}
    \bra{\xsparse_{t,l}}
}\rope\\
\state_{t,l} &:= \lr{
    \state_{t-1,l}
+
    (\encoder\ket{\xysparse_{t,l-1}})
    \bra{\xsparse_{t,l}}
}\rope
\end{rcases}{\,^{\textrm{alternative}}_{\textrm{representations}}}\hspace*{-2cm}\\
\ket{\xsparse_{t,l}} &:= \xsparse_{t,l-1} +
    \relu{\decoderx\encoder \ket{\xysparse_{t,l-1}}}\\
\ket{\xysparse_{t,l}} &:=
    \relusm{
        \decodery
        \underbrace{
            \encoder\corr_{t-1,l}
        }_{^{\state_{t-1,l}}}
            \ket{\xsparse_{t,l}}
    }
    \*
    \ket{\xsparse_{t,l}}
\end{split}
\end{align}
\addtocounter{equation}{1}
\end{minipage}}
\caption{State-space equations of the model architectures introduced in this paper.
All architectures refer to a set of $n$ interacting particles (neurons), with activation vectors $x_{t,l} \in (R^+)^{n}$. Vector $y_{t,l} \in (R^+)^{n}$, $y_{t,l}$ is (typically) sparse in the sense of $\|y_{t,l}\|_0$. Variables $\state_{t,l} \in R^{n \times d}$ or $\corr_{t,l} \in R^{n \times n}$ represent hidden state of the system.
$\diamond$ The graph-based BDH dynamics equation~\eqref{eq:bdhgraph}, equivalent to the ruleset from Table~\ref{tab:protocolx}, serves as a starting point for development of architectures represented as local graph kernels in a distributed computing system.
$\diamond$ The simplified BDH-Normfree equation~\eqref{eq:bdhnoln} is a special case of BDH. Up to lack of LayerNorms, it approximates the inference dynamics of \BDHGPU, with the correspondence $\state_{t,l}=\encoder\corr_{t,l}$.
$\diamond$ The tensor-based \BDHGPU architecture is described by equations~\eqref{eq:bdh} (mathematically equivalent to Definition~\ref{def:bdh}, \eqeqref{eq:integral}~and~\eqref{eq:kvstate}) and is the primary point of reference for all model training and all empirical results presented in this study. For a discussion of extensions to \BDHGPU such as heads, see Subsection~\ref{sec:layersheads}. A complete code listing for \BDHGPU is provided in Appendix~\ref{sec:bdh_code_listing}.
}
\label{fig:ss}
\end{figure}
\mathcenter

\begin{figure}[ht!]
\centering
\includegraphics[trim=6cm 4cm 0cm 4cm,width=0.5\textwidth]{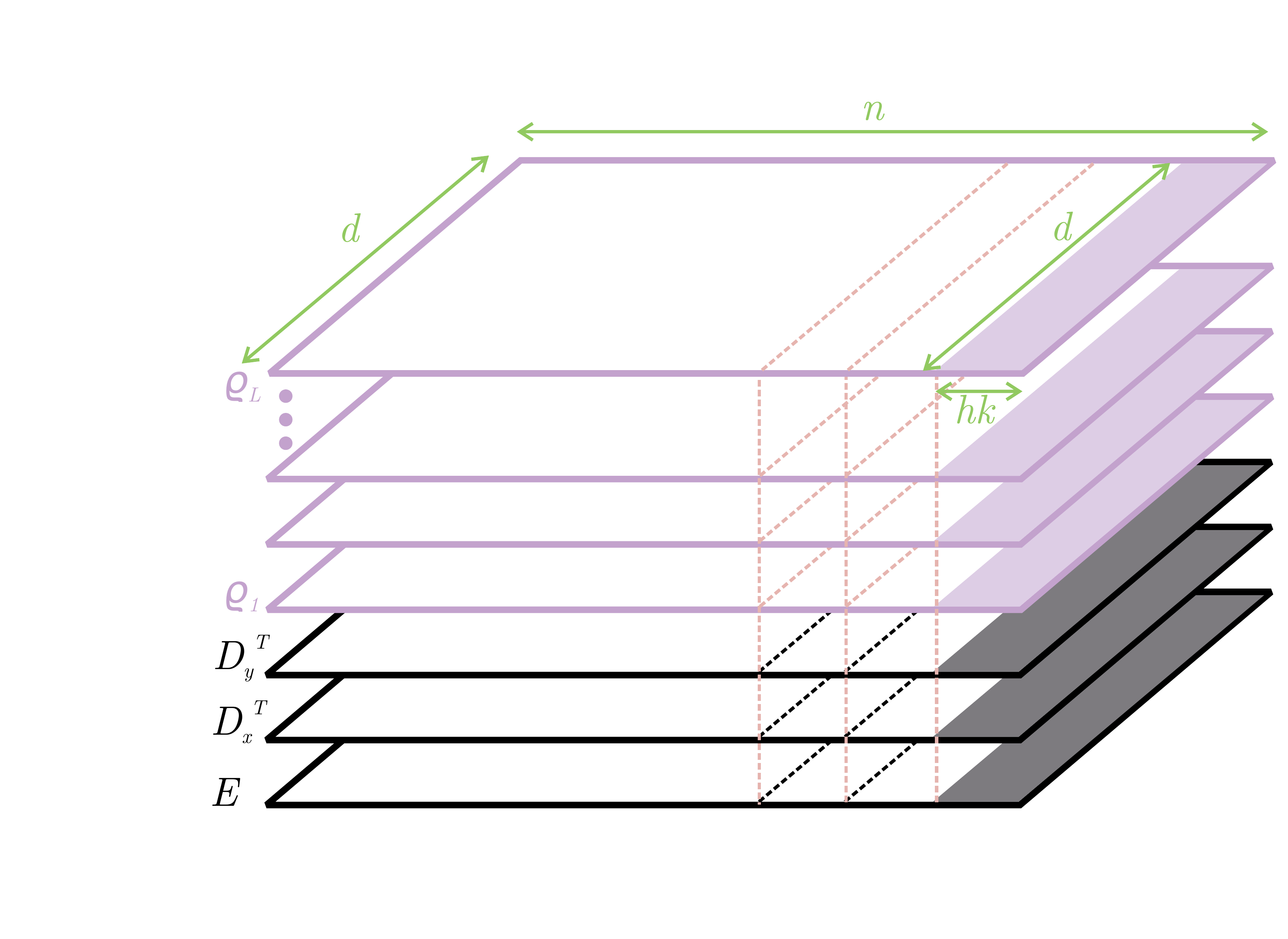}
\caption{Scaling of \BDHGPU architecture in dimension $n$. The other parameters can be considered fixed during scaling. For example, with choice of $d=256$ for low-rank dimension, $k=2$ for neuron pairing with RoPE, and $h=1$ for a single-head architecture, the model scales linearly in dimension $n$ in chunks of $dhk = 256\cdot2\cdot1=512$ parameters.}
\label{fig:stack}
\end{figure}

In what follows, we will perform analysis focusing on the state-space representation of the architecture given by \eqeqref{eq:bdh}.

\subsection{Interpretation of \BDHGPU as a local interacting particle system} The \BDHGPU dynamics equation~\eqref{eq:bdh} has the interpretation of a $n$-particle system, with the state $\state_{t}(i)$ of the $i$-th particle, $i=1,\ldots,n$, given at the end of time $t$ by a vector in $R^d$ for each layer:
$$
\state_i(t):=(\state_{t,l\ (i,\cdot)} : l \in (1, \ldots L)).
$$
Overall, as we will see directly, the way particle $i$ interacts with other particles at time $t$ is described by the following tuple $Z_i$:
$$
Z_i(t) :=(\state_i(t), \encoder_{(i,\cdot)}, \decoderx_{\,(\cdot,i)}, \decodery_{\,(\cdot,i)}).
$$
Here, $\state_i(t)$ represents the in-context state associated with particle $i$ (initialized as $\0$ at the start of inference), while the other three vectors of length $d$ associated with this particle are trainable, but do not change during inference.

The system scales in dimension $n$ and is completely uniform in this dimension, excepting following effect. Let $k$ denote the size of largest block in the block-diagonal matrix $U$; then particles, are bound by this effect into non-uniform $k$-tuples when $k>1$ ($k=1$ when $U$ is the ALiBi matrix, and $k=2$ when $U$ is the RoPE matrix). %
Thus the system, in general, scales in the dimension of $n$ uniformly, in chunks of $k$ particles (see Fig.~\ref{fig:stack}).

The interaction between particles is, intuitively, local. To be able to proceed with discussion with rigor and without complicating notation, we assume for the analysis that $k=1$. We also drop LayerNorms from the equations of inference dynamics. (Models generally do not train following \BDHGPU without any LayerNorm, but we observed empirically that there is some flexibility as to where these LayerNorms are placed; they can also be moved to the neuron dimension $n$, and they are parameter-free.) The dynamics without LayerNorm are represented under the name BDH-Normfree in Fig.~\ref{fig:ss}.

We have the following.

\begin{observation}[local particle interaction `by mean-field']
The BDH-Normfree dynamics have the interpretation of a mean-field interaction between particles, fully characterized at any time by $O(dL)$ parameters of particle in state, and $O(d)$ parameters in particle representation.
\end{observation}
This observation is essential for the subsequent discussion, and it can be expanded in three different ways.

In computing terms, at any time $t$ and in any layer $l$, the action of the system can be represented as an iterated application of the dynamics equations~\eqref{def:bdh}, with each of the particles realizing, for each equation in each layer (i.e., a total of $3L$ times), a form of micro-program, involving local computation and communication with other particles by broadcast. In a framework of local distributed computing (cf.~e.g.~\cite{Peleg}), it would be represented as a node performing the following form of local kernel as a part of a networked system:
\begin{enumerate}
\item compute some message vector $m_i \in R^d$ locally (without communication with other particles), based only on current activation $x_{t,l,\,i}$, $y_{t,l,\,i}$ and previous state $Z_i(t-1)$,
\item broadcast message $m_i \in R^d$ to other particles,
\item receive the mean-field message $\bar m = \sum_{j=1}^n m_j \in R^d$, identical for all particles,
\item update local activation variables for the next layer $l+1$, and update new state $\sigma_i(t) \subseteq Z_i(t)$, based on the received result $\bar m$ of the broadcast and local computation.
\end{enumerate}

In Physical terms, we observe that the interaction field of the particles, which realizes the broadcast, is localized, and can at any time $t$ be expressed as a sum of pairwise particle interaction terms between particles $i, j \in 1,\ldots,t$. These pairwise interactions depend only on parameters $Z_i(t-1)$ and $Z_j(t-1)$, and the activation variables of these particles, representing properties of these particles at time $t$ and expressible through $O(Ld)$ scalars. This interaction field evolves with time $t$ together with $Z_i$ and $Z_j$.\footnote{Note that $Z_i(t-1)$ depends only on $\state_{t-1,l}$, not $\state_{t,l}$. This is because of the stopping index of $\tau = t-1$ in the definition of attention $\yKV$ in Def.~\ref{def:bdh}, and is intentional.}

In Engineering terms, we observe that any transformation of a length-$n$ vector into another length-$n$ vector passes through an intermediary low-rank representation of dimension at most $d$. An example is the equation for $x_{t,l}$ in~\eqref{eq:bdhnoln}, which reduces length-$n$ vector $y_{t,l}$ to a length $d$-vector through application of the encoder matrix $\encoder$, before lifting the dimension back to $n$ by an application of the decoder $\decoderx$.

\subsection{Expressing \BDHGPU using BDH: preserving parameter and state size}

\BDHGPU and BDH both represent $n$-particle systems. For a special parameter choice (of BDH), they have equivalent patterns of communication and of computation (up to placement of layer norms).

\begin{observation}[BDH-Normfree is a special case of the BDH graph model]\label{obs:equivalence}
Models in the BDH-Normfree architecture (\eqeqref{eq:bdh}) and models in the BDH architecture (\eqeqref{eq:bdhgraph}) are formally equivalent (i.e., the same model) subject to the following choice of model parameters of BDH: \begin{equation}\label{eq:equivalence}
\graphx\ee-\graphx\ii= \decoderx\encoder,\quad \graphy\ee-\graphy\ii= \decodery\encoder,\quad \graphs = \1^{n \times n},
\end{equation}
where $\1^{n \times n}$ is the all-ones matrix.\qed
\end{observation}

We discuss in more details below how BDH compares to BDH-Normfree in terms of size of state and number of parameters needed for one architecture to approximate the inference dynamics of the other. In general, BDH is not less expressive than its tensor-based counterpart.

For \BDHGPU parameters and state are naturally expressed using tensors of $O(nd)$ model parameters. In this section, we discuss how to express model parameters and state of BDH, in such a way as to maintain comparable size of parameter and model space.

\subsubsection{Expressing matrices \texorpdfstring{$\decoderx,\decodery,\encoder$}{Dx, Dy, E} as graphs \texorpdfstring{$\graphx, \graphy$}{Gx, Gy}}

We start by taking care of the first correspondence, that of parameter spaces of \BDHGPU and BDH. Asymptotically, BDH is strictly more expressive at the same number $O(nd)$ parameters. Recall from \eqeqref{eq:bdh} that the parameter space of \BDHGPU consists of three matrices $\decodery\decoderx\in R^{n \times d}$, $\encoder\in R^{d \times n}$, and (up to shifting of LayerNorms), their role is to encode the pairs of matrices $\decodery\encoder, \decoderx\encoder, \in R^{n \times n}$, as used in \eqeqref{eq:bdhnoln}. In the Claim below, we capture the correct encoding of one such matrix pair in the form of a graph of $O(nd)$ parameters.

Consider a (directed, weighted) graph $\graph \in R+^{n \times n}$ on a set of vertices $V = \{1,\ldots,n\}$. We will consider a graph which need be directly a sparse graph, but can be represented as a square of a graph with few edges. Formally, we will say that $\graph \in \G^2 (n,m)$, for some $m \in \Nat$, if there exists a graph $H \in R^{(n+s) \times (n+s)}$, with vertex set $V \cup S$, where $|S|=s$, such that $G = H^2 [V]$, i.e., $G$ is the induced subgraph of $H^2$ restricted to vertex set $V$, and $H$ has at most $m$ (strictly positive) edges.

For a $\graph \in \G^2 (n,m)$, we can consider an interpretation of a hidden layer $S$ between input layer $V$ and output layer $V$. %
All matrix weights coefficients are restricted to be non-negative, and the two linear layers are sparse with a total of at most $m$ non-negative connections.

Graphs in $\G^2 (n,m)$ are naturally expressed through the edges of the previously defined graph $H$, using $O(n \log n+m)$ parameters. The class $\G^2 (n,m)$  is strictly more expressive than the class of sparse ($m$-edge) graphs on vertex set $V$.\footnote{The formal expression in the definition of the class of weighted graphs $\G^2 (n,m)$ can be compared to that of the class of graph distance matrices admitting sparse hub labeling representation~\cite{hubs2011} (or closely related landmark labeling). In our case, vertices in the hidden layer $S$ also have a natural interpretation of landmarks on directed paths connecting nodes of $V$.}
We will refer to the middle layer of vertices $S$ that makes such a representation possible as the \emph{sparse synaptic layer}, to the graph $H$ on vertex set $V \cup S$ as the \emph{sparse linear circuit}, and the graph $H^2[V] \in \G^2 (n,m)$ as the \emph{neuron-neuron interaction graph}.

The role of the constructed graphs is to serve for propagating linear dynamics of the form $v \mapsto \graph v$, $v\in (R^+)^{n\times 1}$, for graph-based local models. We have the following Observation.
\begin{observation}
Let $G \in \G^2 (n,m)$ be a neuron-neuron interaction graph, with $G = H^2[V]$, where $H$ is the sparse linear circuit on vertex set $V \cup S$, which has $m$ edges. Then, the linear dynamics on graph $G$, $v \mapsto Gv$, can be efficiently expressed through two steps of linear dynamics on graph $H$, $v\mapsto H^2 v$, for $v \in (R^+)^{n \times 1}$. This representation requires $O(m)$ parameters.
\qed
\end{observation}
In the above, thee exact number of parameters needed to represent a graph of $m$ edges follows from conventions introduced in the Notation (Section~\ref{sec:notation}). In what follows, we will assume that BDH represents its parameter matrices through appropriate sparse linear circuit graphs $H$, which it uses to realize the linear neuron-neuron interaction dynamics $G$. We illustrate the correspondence between graphs $G$ and $H$ in \figref{fig:g_h_interpretation}.
\begin{figure}
\centering
\includegraphics[width=0.75\textwidth]{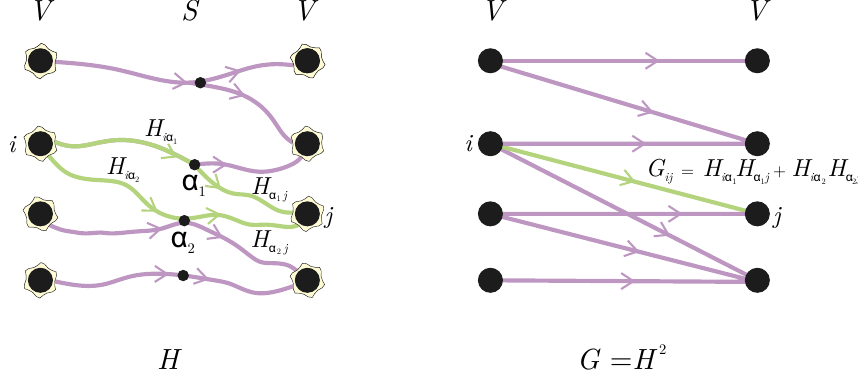}
\caption{Neuron-neuron communication using graphs $G \in \G^2(n,m)$:  correspondence between graph $H$ with $m$ edges (left), and neuron-neuron interaction graph $G = H^2$ (right). The approach allows to express linear signal propagation on a broad class of graphs $\G^2(n,m)$ using two steps of linear dynamics on a sparse circuit $H$, i.e., $Gz = H^2z$ for $z \in (R^+)^n$.}
\label{fig:g_h_interpretation}
\end{figure}

We observe that BDH can express \BDHGPU parameter matrices with the same asymptotic number of parameters. The claim below applies to pairs of matrices $\decoder\encoder$, for $\decoder=\decodery$ and $\decoder=\decoderx$.

\begin{claim}\label{claim:graphs}\label{obs:expressivegraph}
For any matrices $\decoder\in R^{n,d}$, $\encoder\in R^{d,n}$, there exist neuron-neuron interaction graphs $\graph\ee, \graph\ii \in \G^2(n,m)$, such that $\graph\ee-\graph\ii= \decoder\encoder$, with $m=O(nd)$. 
In consequence, for the same asymptotic number of parameters $O(nd)$, graph-based feed-forward mechanisms of BDH are strictly more expressive than corresponding mechanisms in the tensor-based implementation, BDH-Normfree.
\end{claim}
\begin{proof}
The short proof of the Claim is deferred to Appendix~\ref{apx:proofgraphs}.
\end{proof}

We note that the converse implication does not hold: an arbitrary graph $\graph\ee \in \G^2(n,m)$ does not admit an exact low-rank decomposition $\graph\ee = \decoder\encoder$. Indeed, in general any low-rank decomposition introduces a form of noise whose implications we discussed in Section~\ref{sec:lowrank}: if $\graph\ee$ has a modular (cluster) structure, the low-rank approximation $\graph\ee \approx \decoder\encoder$ still allows a form of in-cluster propagation dynamics.

\subsubsection{Expressing \BDHGPU attention on graphs: sparsification and trainability of \texorpdfstring{$\graphs$}{Gs}}

We recall that by Observation~\ref{obs:equivalence}, the equivalence between the attention state $\corr_{t,l}$ in BDH and in tensor-based implementation holds for the case of the complete directed graph, $\graphs = \1^{n \times n}$.

This means two things: first, in BDH, graph $\graphs$ can be trainable, while in \BDHGPU it is not. This acts to the potential advantage of BDH for expressiveness.

Second, in BDH, the graph $\graphs$ obtained through the direct correspondence is dense: with $n$ neurons, BDH would need $n^2$ synapses to precisely reflect BDH-Normfree. This aspect is more of a technical nuisance than an actual difference: the expressiveness of the attention mechanism of BDH, equipped with a sparse graph $\graphs$, is sufficient to represent the attention operation as used in BDH-Normfree. Indeed, in the tensor-based \BDHGPU dynamics, the attention operation is immediately followed by a low-rank operation, $\state_{t,l} = \encoder \corr_{t,l}$, where $\state_{t,l}$ has $nd$ parameters. Graph models can instead rely on a sparse graph $\graphs$ to achieve the same form of state compression through sparsification.
\begin{claim}
\label{claim:att_equiv}\label{claim:attention}
The attention block of BDH-Normfree can be expressed using the attention block of BDH with a graph $\graphs$ having $O(nd)$ edges, subject to a natural preparation of attention values entering the attention block of BDH (directly before this attention block).
\end{claim}
\begin{proof}
The formal statement of the Claim and its proof are deferred to Appendix~\ref{apx:proofattention}.
\end{proof}
Going beyond the formal equivalence between BDH and \BDHGPU from Observation~\ref{obs:equivalence}, the above Claim, combined with Claim~\ref{obs:expressivegraph}, shows that BDH has at least the same expressiveness as \BDHGPU even for the same number of parameters $O(nd)$ and the same size of state $O(nd)$ per layer.

Independent of graph-based models, in the subsequent analysis of the feed-forward network and attention mechanisms of \BDHGPU, we will show that the matrices $\decoderx\encoder, \decodery\encoder, \sigma \in R^{n \times n}$ of \BDHGPU admit a natural interpretation as $n$-node directed graphs (once appropriate threshold functions are applied). For example, the visualizations in \figref{fig:powerlaw1} and \figref{fig:powerlaw2} correspond to graph representations of matrices $\corr$ and $\graphx := \decoderx\encoder$ of \BDHGPU, respectively, after thresholding. This graph interpretation of matrices in \BDHGPU \emph{also} defines the neuron-neuron communication graph of the underlying BDH model, given by the equivalence from \eqeqref{eq:equivalence}.

\section{Implementation and scaling laws}

A code framework for \BDHGPU, representing the architecture from Definition~\ref{def:bdh}, is made available in the Appendix~\ref{sec:bdh_code_listing}. In this Section, we present some guidelines on choice of hyperparameters, and an empirical study of models implemented in the \BDHGPU architecture, as well as a comparison to the Transformer and other language model architectures.

\subsection{Implementation characteristics of \BDHGPU}\label{sec:layersheads}

\begin{figure}
\includegraphics[trim=0 6.5cm 0 0,clip,width=\textwidth]{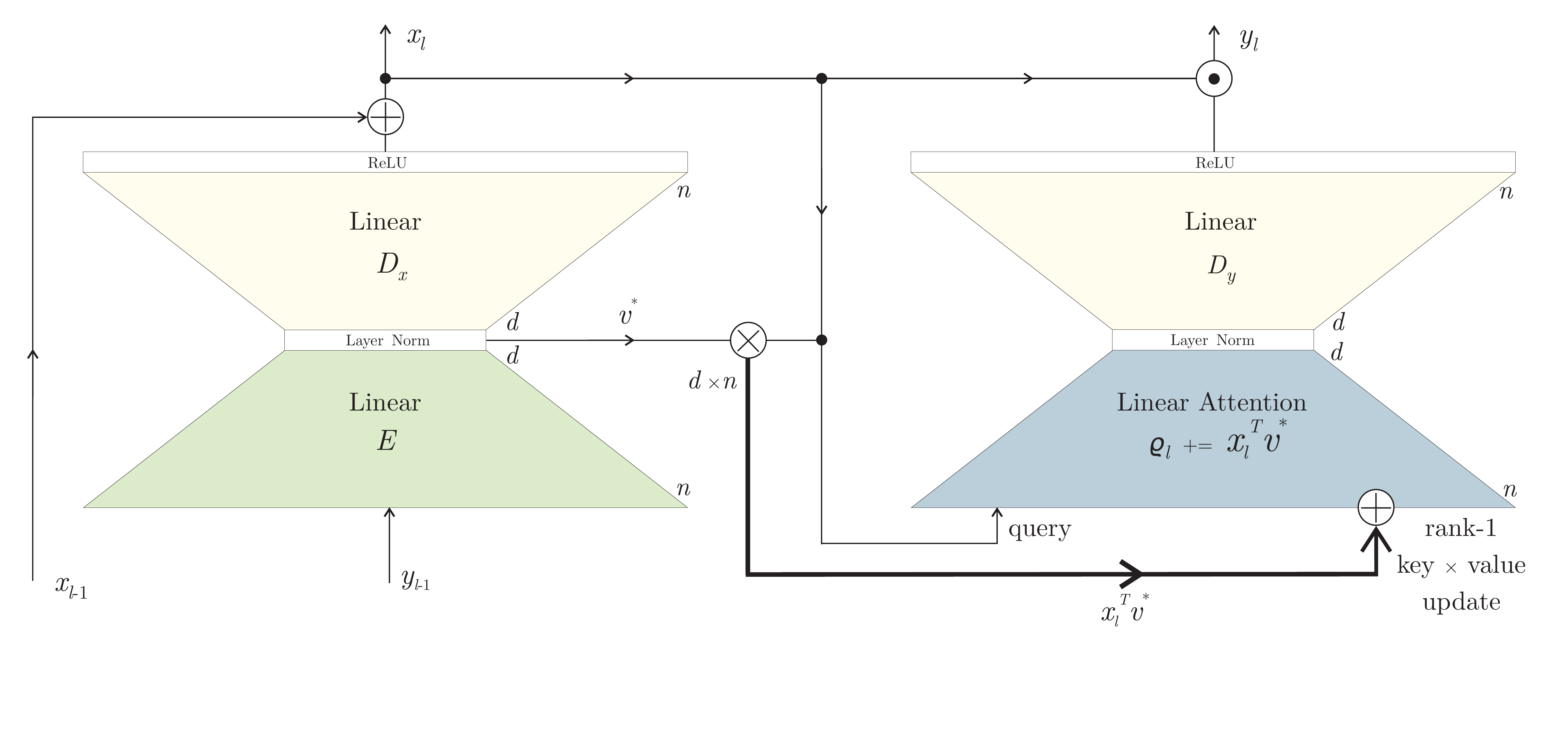}
\caption{Diagram of one layer of the \BDHGPU architecture, following \eqeqref{eq:bdh}. Layer inputs are $x_{l-1}, y_{l-1} \in R^n$, layer outputs are $x_l, y_l \in R^n$. Model parameters are contained in the $\encoder \in R^{n\times d}$ and $\decoderx, \decodery \in R^{d\times n}$, and shared across all layers. Each layer has a state $\state_l \in R^{n\times d}$ which is used in the Linear Attention block and persisted over time. PyTorch code implementing the model is provided in Appendix~\ref{sec:bdh_code_listing}}\label{fig:architecture}
\end{figure}

\paragraph{Model scaling in neuron dimension $n$.} The architecture $\BDHGPU(n,d)$ has two main dimensions, $n$ which is the dimension of its concept (\emph{neuronal}) space, and $d\ll n$, which is its low-rank (\emph{synaptic}) dimension. The model scales primarily with the number of neurons $n$. Almost all of the weights of the model are contained in three $n\times d$ parameter matrices called $\encoder, \decoderx, \decodery$; thus, the number of parameters is precisely $(3+o(1)) n d$.
The ratio between the dimensions $n$ and $d$ increases rapidly (``asymptotically'') with model size; already for a 25M-parameter model, a sound choice of dimensions is: $d=256$, $n=32768$, read as $32768$ neurons, each characterized by a total of $3d=3\cdot 256 = 768$ scalar parameters.

\paragraph{Layers and heads.}
The architecture has $L$ layers (e.g., $L=10$). As in the Universal Transformer~\cite{dehghani2018universal}, all layers use the same set of weights for each of the parameter matrices.

The architecture may be equipped with several heads $h$, subdividing dimension $n$. The role of heads is limited to a single parameter-free LayerNorm, normalizing outcomes of linear attention separately for each head. The optimal number of heads is typically smaller than in the Transformer (e.g., $h=4$).

\paragraph{Linear attention with state aligned to neurons.}
The state space of the model is fixed and large. It has the macro-interpretation of associative memory (like KV-cache, but organized differently), and is used to perform linear attention. For each layer, the state space is independent and has a fixed dimension $n\times d$, the same as the model weight matrices. Thus, a portion of $d$ parameters of a state is directly associated with each of the $n$ neurons. With each token processed, a fraction of the model's state space is updated. Sharing of state between the $L$ layers is not performed in the vanilla architecture. As usual with SSM's, there is no notion of a context window.

Similarly to BDH, \BDHGPU maintains a large recurrent state comparable in size with its total number of parameters (c.f. \figref{fig:stack}). This stems from the fact that both the recurrent state matrix, and parameter matrices are expressed as low rank $d$ factorizations of $n \times n$ graph transition matrices. We believe that this helps the model with generalization with respect to RNNs which have $O(N^2)$ parameters which manipulate a state of size $O(N)$.

\paragraph{Sparse positive activation.}
The architecture relies on a length-$n$ vector $x_{t,l}$ passed to the $l$-th layer for the $t$-th token processed, which can be assimilated to the vector giving rise to keys, values, and queries in the Transformer, but operating in higher dimension. As a crucial design assumption, this vector has all non-negative elements ($x_{t,l}\geq 0$).

An empirically observed fact is that the activation pattern of $x_t$ rapidly becomes sparse (in a typical training run, only $\rho \approx 5\%$ of the $n$ entries of vector $x_t$ are non-zero). This corresponds to the fraction of the state space read and updated for each token.

\subsection{Comparison of \BDHGPU to GPT2-like Transformers}\label{sec:comparison_bdh_gpt_transformers}

\paragraph{Architecture differences.} \BDHGPU in its vanilla form can be compared to the GPT2 architecture~\cite{radford2019language} with RoPE attention. In this comparison, \BDHGPU retains or strengthens the key advantages of the Transformer (parallel trainability, attention mechanism, scaling laws for loss versus parameter count, learning rate per token) on tests and benchmarks at the model scales we tested (1B parameters), across tasks such as language and translation.

The architecture of a single layer of \BDHGPU is presented in \figref{fig:architecture}. The most evident architecture differences between \BDHGPU and the Transformer include the following:
\begin{itemize}
  \item[$-$] \BDHGPU has fewer parameter matrices, allowing for more compact interpretation and analysis.
  \item[$-$] \BDHGPU scales for parameters (and context length) almost exclusively in a single neuronal dimension, $n$.
  \item[$-$] Key-value state and parameter matrices have matching dimensions and are highly localized together with state, with portions of these matrices attributable to individual neurons.
  \item[$-$] There is no notion of context length in \BDHGPU, and consequently no hard bound on it.
  \item[$-$] Attention of \BDHGPU is linear, but happens in the model's large neuronal dimension.
  \item[$-$] Activation vectors $x, y$ of \BDHGPU are positive (after passing through ReLU gates), and vectors $y$ are observed to be extremely sparse in practice.
\end{itemize}

\paragraph{Transformer-like scaling laws.} We have experimentally validated the scaling laws of \BDHGPU, expressing loss as a function of parameter count, for next-token-prediction tasks. At the same parameter scale, \BDHGPU generally compares favorably to the Transformer even on relatively short-context tasks requiring use of attention, such as translation, \figref{fig:translation}. In general, on next-token prediction tasks, \BDHGPU appears to show improvement of loss reduction per token of data than the Transformer, i.e., \emph{learns faster per data token}, both for the natural tasks we tested (see e.g.~\figref{fig:translation}) and on synthetic puzzles.

\begin{figure}
\centering
\includegraphics[width=.7\textwidth]{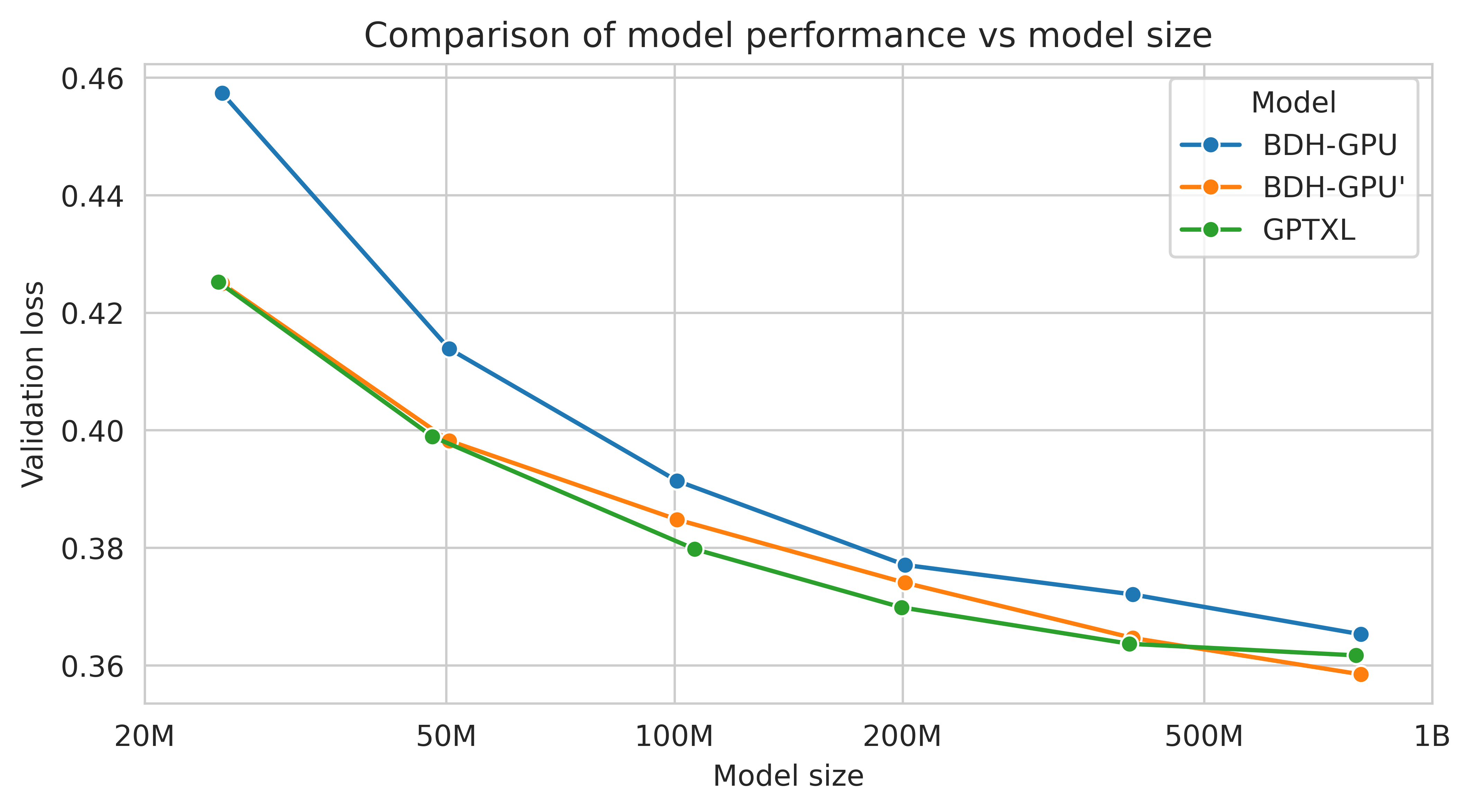}
\caption{Performance of \BDHGPU and GPTXL versus model size on a translation task. We have tested all models under the same training and evaluation regimes. All models show improved performance with scale. \BDHGPU uses exactly the formulation provided in Appendix~\ref{sec:bdh_code_listing}, while BDH-GPU' extends conditional gating of states and logits. All models are trained with truncated backpropagation through time on sequences 2048 characters long, and carry their state ($\state$ matrix for BDH models and a buffer of last 4096 KV-Cache entries~\cite{dai2019transformerxlattentivelanguagemodels} for GPTXL) between minibatches. BDH models are scaled only by varying the number of neurons $n$ and keep all other hyperparameters fixed, making them easy to scale. On the other hand, GPTXL were scaled in both the embedding dimension and the number of layers and required Dropout~\cite{srivastava2014dropout} tuning for optimal performance. We observe that \BDHGPU' matches the GPT Transformer at all model sizes we have evaluated.\\Details on model hyperparameters and training setup are provided in Appendix~\ref{sec:bdh_scaling_details}}
\label{fig:translation}
\end{figure}

The \BDHGPU architecture appears to be a preferred choice for training setups where: (1) models need to learn from scarce data, or (2) training workloads need to be optimized for makespan. For the first setting, the training rate per token is the decisive factor. For the second setting, \BDHGPU can be used differently than the Transformer in distributed training and distributed inference setups because of the way it scales its dimensions.

\paragraph{FLOPS counts.} The theoretical count of arithmetic operations per token of \BDHGPU during inference is bounded by $O(ndL)$. Each parameter is accessed $O(L)$ times per token (with the typical sufficient number of layers being smaller than in the Transformer), and each element of state is accessed $O(1)$ times per token, with small hidden constants. These are rough bounds for a simple implementation, and do not take advantage of activation sparsity.

For short contexts \BDHGPU is amenable to parallel training with a causal self-attention kernel. The simple code template provided in the Appendix \ref{sec:bdh_code_listing} is sufficient to reproduce the empirical results presented in this paper on a single GPU node. For longer contexts (typically above 4096 tokens for $d=256$), a state-space kernel for linear attention is faster and more space-efficient.

\subsection{Comparison of \BDHGPU to other sequence processing architectures}

\paragraph{Transformers with Linear Attention.} Linear attention works well when used in high dimension, subject to appropriate preparation of key vectors (as we discuss in~Subsection~\ref{sec:bdhgpumacroexpattention}). An elegant way to eliminate non-linearity of attention, by applying preparation of key vectors through tensor product, was proposed in~\cite{buckman2024}. We use a completely different approach to achieve attention in high dimension.

A much broader line of work on linear attention for the Transformer, initiated by~\cite{katharopoulos2020transformers} concerns applying linear attention in low dimension after appropriate preparation of keys and values. This is effectively a technique for SSM state compression, and it is not clear whether it relates favorably to other SSM state compression techniques. An empirical study of the amount of information recoverable from SSM's with compressed state can be found in~\cite{benkish2025overflowprevention,liu-etal-2025-scaling}.

A general theoretical framework for analyzing the expressiveness of Linear Attention models with attention working with positive vectors can be found in the context of the FAVOR+ framework of the Performer~\cite{choromanski}. Finally, a general state-space formalism for Transformer models admitting Linear Attention was considered in~\cite{sun2024retentive,liu-etal-2025-scaling}.

\paragraph{Other types of Transformers.}
Variants of the Transformer with identical parameters in all layers, like the Universal Transformer~\cite{dehghani2018universal}, have a number of desirable features, notably in terms of explainability and ease of defining metrics. The downside of sharing parameters between layers in the Universal Transformer is a slight time overhead for the feed-forward network operations, when measured in FLOPS per parameter. The situation is similar in \BDHGPU.

\BDHGPU has sufficient expressiveness to prepare keys and queries for the attention operation, so that the outcome of attention captures a similarity measure between keys and queries corresponding to the outcome of a class of Locality Sensitive Hashing (LSH) functions with a very large number of buckets (cf.~Subsection~\ref{sec:bdhgpumacroexpattention}). The study of LSH-based KV-cache for the Transformer was initiated with the Reformer~\cite{kitaev2020reformerefficienttransformer}, and the LSH~Transformer architecture introduced in the same work. Generally, the LSH~Transformer is shown to rapidly approach Transformer baseline behavior in practice as the number of buckets increases. The class of LSH functions considered is not the same, but some intuitions gained in the study of LSH attention may carry over to \BDHGPU.

Finally, several lines of work have been devoted to making the Transformer work with longer context windows. Two distinct approaches, which work notably well, are the soft-rolling context window of the TransformerXL~\cite{dai2019transformerxlattentivelanguagemodels}, and hierarchical attention~\cite{yang2016hierarchical}. The \BDHGPU architecture is, generally, amenable to some of these extensions to the Transformer's attention mechanism, while also providing new ways to extend context length in a more uniform manner.

\paragraph{Networks with sparse activation.} The use of the ReLU gate as a systematic way to achieve sparse activation was, to our knowledge, first exhibited in~\cite{haziza2025acceleratingtransformerinferencetraining}.

A recent variant of the Transformer called Spark Transformer~\cite{you2025sparktransformerreactivatingsparsity} relies on a combination of top-k operations and soft thresholding to provide a reduction in both attention and feed forward network activations compared to the Transformer, achieving neuron sparse activation of 8\%. Compared to our work, the method used therein to achieve activation sparsity effects is completely different (and rather involved). Beyond the question of sparsity, \BDHGPU is not more similar to the Spark Transformer than to the Transformer.

\paragraph{Oscillatory SSM's.} BDH admits an interpretation at the micro-level as an oscillatory state-space network, as we outlined in Subsection~\ref{sec:toy}. The concept of Oscillatory State Space Models has recently been applied to time series analysis~\cite{rusch2025oscillatory}, with the LinOSS model showing encouraging performance relative to other SSM's (such as Mamba and LSTM's) on tasks of long-horizon forecasting and time-series classification. Other than this, the use of SSM's with the form of an oscillator network has been limited to smaller scale studies. We are not aware of any successful application of oscillatory SSM's to the area of language models and reasoning in language, nor of oscillator network SSM's at scale whatsoever, prior to BDH.

BDH unifies multiple lines of intuition found across existing models, offering a coherent framework in which the components naturally align. The result is a biologically plausible reasoning model with an interpretable structure and state-of-the-art performance that has been experimentally verified.

\section{Analysis: emergence of modularity and scale-free structure}\label{sec:modularity}

Large-scale reasoning systems appear to benefit from hierarchical structuring into sub-modules. In Machine Learning, the usual approach has been to design such a modular structure, by way of assigning roles and scales to different sub-modules explicitly. Many works have postulated modules capable of representing hierarchical relationships between features of objects, e.g., capsule networks~\cite{sabour2017dynamic}. Some models have attempted to capture intelligence by recreating elements of structure recognized in brain study, going so far as to try to map functional sub-networks of the brain with empirically identified function into specific sub-modules in the design of a larger ML system, cf.~\cite{lecun2022path}. %

In this work, we propose a learnable system which ends up with modularity. We show how scale-free modular structure emerges naturally when the model is implemented by a network with local graph dynamics. 
In this Section, we discuss the emergence of the structure of inter-neuron connections of BDH during training, while in Section~\ref{sec:attention} we look at its temporal activation patterns during reasoning inference.

The rest of this section is organized as follows. In Subsection~\ref{sec:backgroundmodularity}, we introduce basic concepts related to modularity and scale-free behavior of networks. We then look at the expressiveness of feedforward networks of \BDHGPU and their usefulness as a signal propagation dynamics in Subsections~\ref{sec:macrocomparison-of-feedforward} and~\ref{sec:bdh_feedforwardpropagation}. In Subsection~\ref{sec:modularity_detailed}, we show theoretically how modular graph structure, with appropriate community voting mechanisms, emerges as a plausibly necessary element for the feed-forward networks $\decoderx\encoder$ and $\decodery\encoder$ to function correctly. In Subsection~\ref{sec:bdh_empiricalgraphs}, we look at the corresponding empirical properties of these matrices, and the scale-free and modularity properties of the corresponding graphs $\graphx\ee$ and $\graphy\ee$ of the underlying BDH graph dynamics.

\subsection{Background: modularity and scale-free property of systems}\label{sec:backgroundmodularity}

\paragraph{Importance of modularity for information propagation.} Graph systems serving a function related to information propagation tend to achieve modular graph structure, and rely on it to obtain the most desirable tradeoff between efficiency and accuracy of the system dynamics. Such emergence of ``hidden structure'' may be observed e.g.~through topic specialization of system regions, or through the coordinated voting behavior among nodes which organize themselves into communities, admitting higher local density. This type of graph community self-organization has two main advantages over a system with an explicit partition into subsystems. First, it allows nodes to belong to multiple communities, and to act as bridges between them. Second, it allows the scale and relationship between communities to evolve over time, as their relative importance changes or new connections emerge. 

Historically, the crucial role of emergent modular structure for systems tasked with efficient knowledge retrieval at scale was first observed in the context of the World Wide Web before the year 2000, notably in the transition from catalogue-based systems (DMOZ Open Directory Project, craigslist, etc.) to naturally evolving systems based on webs of knowledge (Wikipedia, etc.), interlinked topic-based communities (reddit, etc.), and reliance on evolving network link structure for assigning expert weights to nodes in a voting process (Google PageRank, etc.). Formalization of modular properties followed soon after, with the mostly commonly used definition of modularity being proposed by Newman in~\citeyear{doi:10.1073/pnas.0601602103}. The main theoretical reference for studies of modularity is the Stochastic Block Model~\cite{HOLLAND1983109} and its later generalizations, e.g., to hierarchical settings. While the definition of Newman modularity is not (efficiently) constructive, it can in practice be closely approximated by greedy algorithms~\cite{Blondel_2008} or spectral approaches~\cite{community2014}.

\paragraph{Scale-free property.} The scale-free property of natural systems dealing with information processing is generally accepted as a manifestation of their operation at criticality. This refers to operation within a regime where they are both sufficiently stable to enable efficient information retrieval in the short-term, and sufficiently adaptable to be able change their behavior abruptly as new knowledge inputs become available, invalidating previous paths of reasoning or knowledge retrieval. The generally accepted definition of scale-free behavior of such a dynamical system assumes that the likelihood of a new piece of information (or other localized innovation to the system) to affect $n'$ nodes of the system, for any $n'<n$, should by polynomially large in $1/n'$. For most information propagation dynamics, under certain uniformity assumptions, e.g., that the new piece of information arrives at a uniformly random node of the system, a usual necessary (not sufficient) condition for scale-free property is for the distribution of node degrees to follow a power-law distribution. 

In the practice of applied sciences studying real-world network phenomena, and in the absence of the possibility to perform more in-depth analysis, power-law degree distributions are sometimes equated with scale-free behavior. One notable research application involves modeling of extreme events: understanding scale-free behavior allows researchers to make predictions about rare, large events from data on smaller, more common ones.

For systems with known local graph dynamics, like those considered here, more refined analysis of scale-free properties are possible. We nonetheless also report heavy-tailed degree behavior as the most obvious lithmus test indicator of scale-free operation of the system.

\subsection{\BDHGPU feed-forward network with the `ReLU-lowrank' block}\label{sec:macrocomparison-of-feedforward}

Low-rank matrices have been considered in multiple contexts of Machine Learning, from preference vectors to Internet latency estimation. In the setting of the Transformer, low-rank matrices form the basis of weight-matrix approximations such as LoRA~\cite{hu2021loralowrankadaptationlarge}.

The ReLU-lowrank block of \BDHGPU captures different properties than the above settings. Its most important effects for \BDHGPU are related to noise reduction, and faithful representation of a certain class of affinity functions on sparse positive vectors. This makes it suitable for use in combination with Linear Attention blocks. We discuss this point further in this Section.

\paragraph{Definition of ReLU-lowrank.} The parameters of \BDHGPU are concentrated in three matrices $\encoder, \decoderx, \decodery$. The encoder matrix $\encoder$ transforms length-$n$ vectors in the neuronal layer into length-$d$ vectors in the hidden layer. The two decoder matrices $\decoder \in \{\decoderx, \decodery\}$ transform length-$d$ vectors in the hidden layer back to the neuronal layer.

We consider the \emph{ReLU-lowrank} operation of passing through the encoder, one of the decoders, and a ReLU gate (cf.\ \eqeqref{eq:bdhnoln}), mapping vectors $z\in R^n$ into $f_{\decoder\encoder}(z) \in R^n$ as follows:
\begin{equation}\label{eq:relulowrank}
    z \mapsto f_{\decoder\encoder}(z) := \relu{\decoder \encoder z}.
\end{equation}
We note that the output $f_{\decoder\encoder}(z) \in (R^+)^n$ always, and that in \BDHGPU we also always have $z \in (R^+)^n$.

\paragraph{Expressiveness of ReLU-lowrank in \BDHGPU and MLP in the Transformer.} %
A single ReLU-lowrank block can be compared to a single MLP block of the Transformer. %
A different comparison provides closer matching of dimensions and structure of nonlinearities, by considering a single ReLU-lowrank with respect to a portion of the Transformer corresponding to the second MLP layer in an MLP block, i.e., starting with the hidden layer of neurons of the MLP in some layer $l$, skipping attention blocks, and followed by the `first' linear layer of the MLP in layer $l+1$, finally followed by the non-linearity (typically GeLU) applied in the hidden layer of neurons in layer $l+1$. Either approach to expressiveness is valid to the extent where we analyze similarities between one architecture with $L$ layers and the other with ``$O(L)$'' layers.

In the spirit of universal approximation theorem frameworks, a (deep) layer-$L$ stacking of Transformer's MLP block with ReLU activation, for Transformer latent dimension $D$ and MLP hidden layer dimension $cD$ (e.g., for $c=4$), is eventually (i.e., for $L \to +\infty$) a universal approximator for all vector functions up to dimension $D-O(1)$~\cite{SHEN2022101}. A similar universal approximation result eventually (i.e.,  for $L \to +\infty$) holds up to function dimension $n$ for residual ReLU-lowrank networks~\cite{lin2018resnet}, however the convergence rate in $L$ is slower due to the smaller size of the hidden layer. These results translate directly to \BDHGPU architecture which also relies on ReLU with residual connections between layers. To summarize informally, for a Transformer with latent dimension $D$ and \BDHGPU with hidden dimension $d$, we expect their feed-forward networks to be comparably expressive (though usually without strict mathematical equivalence) as function approximators for functions up to some dimension $d', d < d' < D$, between $d'$ and $D$ the Transformer can express a richer class of functions, and between $D$ and $n$, \BDHGPU can approximate some functions, whereas the Transformer does not use such high dimension in its vector representations.

We remark that in all cases, regardless of expressiveness of feed-forward mechanisms, \BDHGPU is set up so that it is only using inputs and producing outputs within the positive orthant, $(R^+)^n \mapsto (R^+)^n$.

The main point to consider is: \emph{what classes of useful high-dimensional functions in the positive orthant does ReLU-lowrank naturally express?}

\subsection{ReLU-lowrank as a signal propagation dynamics}
\label{sec:bdh_feedforwardpropagation}
\label{sec:lowrank}
\label{sec:dynamics}

\paragraph{Error of low-rank approximation (without ReLU).} Consider $R^n$ as a space spanned by a fixed set of $n$ orthogonal unit basis vectors $V = \{v_1, \ldots, v_n\}$, called \emph{nodes}.

The low-rank operation can be used to approximate affinities between pairs of nodes, in the following sense. For a given matrix $G' \in R^{n\times n}$, consider low-rank matrices $\decoder \in R^{n\times d}, \encoder \in R^{d\times n}$, such that $G := \decoder\encoder$ approximates $G'$ pointwise. \footnote{Elements of $G$ can be computed pointwise by each pair of nodes: $V \times V \ni (v_1, v_2) \mapsto G:=\bra{v_1} \decoder\encoder \ket{v_2} \in R$.}

Assume $\|G'\|_{1,\infty} \leq 1$. An application of the \JL lemma shows that the following bound holds in the infinity norm: $\|G' - G\|_{\max} = O(\sqrt{\log n\ /\ d})$ (cf.\ e.g.\ \cite{udell2018bigdatamatricesapproximately,budzinskiy2025bigdataactuallylowrank}). Then, for $z \in R^n=R^{|V|}$ with $\|z\|_1 \leq 1$, we have:
\begin{equation}\label{eq:lowrank}
\|G'z - Gz\|_{+\infty} = O(\sqrt{\log n\ /\ d})
\end{equation}
However, no similar bound holds for $\|G'z - Gz\|_2$. Even for `simple' scenarios like the identity transformation $G'=I_n$, the best low-rank approximation admits O(1) additive error in the L2-norm for almost all inputs, and even greater distortion (approaching $\sqrt{n}$) may appear in the L1-norm.

This makes the low-rank operation useful for determining affinity of pairs of coordinates in dimension $n$, but more problematic as a vector transformation function. However, the ReLU-lowrank mechanism (\eqeqref{eq:relulowrank} is able to suppress a part of the noise of the linear low-rank map, allowing to approximate a sufficiently broad class of non-linear operations.

\paragraph{Expressiveness of ReLU-lowrank for Markov chain propagation.}

We will consider positive inputs $z \in R^{+n}$, focusing on sparse vectors.

One important case concerns approximating a Markov chain transformation $z \mapsto G'z$, for some $G' \in R+^{n\times n}$. For such a transformation in the positive orthant, adding the ReLU operation to the linear map does not change anything directly, $G'z = \relu{G'z}$. However, when considering a low-rank matrix $G$, the non-linear transformation $\relu{Gz}$ can provide a closer approximation of $G'z$ for some classes of input vectors $z$, than the low-rank linear operation $Gz$.

We start with the following illustrative example.

\begin{claim}[propagating a Markov chain]\label{obs:rw}
Let $G'$ be the random walk matrix of a directed graph with out-degree $r$ (i.e., a stochastic matrix with $r$ non-zero entries of $1/r$ in each row), and let $v\in V$ be a node (basis vector), $\|v\|_1=\|v\|_2=1$. Then, for any $\eps > 0$, there exists $d = O(r^3 \log n / \eps)$ such that for some matrices $\decoder \in R^{n\times d}, \encoder \in R^{d\times n}$, we have $\|G'v - f_{\decoder\encoder}(v)\|_1 = O(\eps)$.
\end{claim}
\begin{proof}[Proof (sketch)]
Let $\decoder^* \in R^{n\times (d-1)}$, $\encoder^* \in R^{(d-1) \times n}$ denote matrices $\decoder$, $\encoder$ restricted to all but the last coordinate in dimension $d$. Pick $\decoder, \encoder$ so that $\|G'v - \decoder^*\encoder^* v\|_{\infty} < \eps^*$, where $\eps^* = \eps/r$, following \eqeqref{eq:lowrank} (we have $\|G'\|_{1,\infty} \leq 1$ by stochasticity of $G'$). Further, set a fixed bias, placing $1$ on all entries of the last coordinate in dimension $d$ of $\decoder$, and $-\eps^*$ on all entries of the corresponding last coordinate in dimension $d$ of $\encoder$. Taking into account this bias, we now have
$\|(G'v-\eps^* \mathbf{1}) - \decoder\encoder v\|_{\infty} < \eps^*$.

For all coordinates $v_j\in V$ such that $\bra{v_j}G'v = 0$, we now have $\bra{v_j}\decoder\encoder v < 0$, hence also $\bra {v_j}f_{\decoder\encoder}(v) = 0$. For all other coordinates $v_j$, we have $\bra{v_j}G'v = 1/r$, and $1/r - 2\eps^* < \bra {v_j}f_{\decoder\encoder}(v) \leq 1/r$. Thus, $\|G'v - f_{\decoder\encoder}(v)\|_1 \leq 2 \eps^* r$, and the claim follows.\footnote{As a point of elegance, we note that in this proof, $\bra {v_j}f_{\decoder\encoder}(v) \leq 1/r$, so $f_{\decoder\encoder}(v)$ was not an \emph{unbiased} estimator of $G'v$. This is easily fixed in the first-order by introducing a global multiplicative bias of $(1+\eps^*)$ to the approximation, for example, substituting: $(1+\eps^*)\decoder \mapsto \decoder$.}
\end{proof}

The above observation shows how ReLU-lowrank deals with one specific class of graph affinity functions (random walks of adjacency of sparse graphs), for transformations of vectors which are nodes in our distinguished basis. We use this example as it is the simplest case which exhibits the benefit of threshold non-linearity: for basis vectors, the operation $f_{\decoder\encoder}$ captures a basic propagation effect which is well known (in general) to require a \emph{full-rank} matrix $G' \in R^{n \times n}$ if relying only on linear operations.

\paragraph{Propagation and reinforcement of signal.}
The same thresholding approach, as discussed for Markov chains, turns out to be applicable to a wider class of signal propagation dynamics. It consists in first obtaining a positive-valued signal with heavy random noise, then applying a negative bias, and finally using the ReLU gate to act as a noise threshold.

Any linear function $G'$ can be represented with a hidden layer of $s \leq n^2$ nodes, through two matrices $\decoder' \in R^{+ n \times s}$ and $\encoder' \in R^{+ n \times s}$, such that:
$$
G' = \decoder' \encoder'.
$$
The above holds in general, and we will refer to such a representation of $G'$ as having a sparse hidden (synaptic) layer. We will consider now the question of expressing non-negative functions, $G' \in (R^+)^{n\times n}$.  An example of a valid representation of $G'$ is given through a node-edge incidence representation, $\decoder'_{i,(i-1)n+j} = \encoder'_{(i-1)n+j,j} = \sqrt{G'_{ij}}$, but usually this representation is not optimal in terms of the number of non-zero entries of $\decoder'$ and $\encoder'$.

In general, any low-rank approximation of $G$ can be equivalently expressed as $G \approx \decoder' P_D P_E^T \encoder'$, for some two matrices, $P_D, P_E \in R^{s \times d}$. We will consider the most common class of low-rank approximations obtained by taking $P_D = P_E = P \sim \gauss(0,1)^{s\times d}/\sqrt d$. Consider a vector $z$ passing through the ReLU-lowrank operation, and the following vectors $u \in R^s$, $w \in R^n$:
\begin{align*}
u &:= \encoder' z\\
w &:= \decoder' P P^T u
\end{align*}
If $v_i^T z$ has the interpretation of a signal being sent by node $v_i$, then $u$ is the encoded message being passed through the hidden layer of the network, and $v_j^T w$ is the message received by node $v_j$.

\subsection{Modularity in \BDHGPU signal propagation}\label{sec:modularity_detailed}

We are now ready to capture the essence of the signal propagation and reinforcement capability of the ReLU-lowrank system. To describe the conditions under which a neuron is able to decide whether it should, or should not activate. By a standard analysis of independent Gaussians, we have the following probabilistic statement, under random choice of $P$.

\begin{claim}[selective neuron activation]\label{obs:fscore}
Suppose that the signal of $u$ is uniformly concentrated on a set of nodes $A$ of the hidden layer, i.e., for some subset $A$ of indexes of the hidden layer, we have $u_{\bar\alpha}=0$ for $\bar \alpha \not\in A$, and $u_{\bar \alpha} \in [\frac{1-\kappa}{\sqrt{|A|}},\frac{1+\kappa}{\sqrt{|A|}}]$ for $\bar \alpha \in A$, so that $\|u\|_2 \in [1-\kappa,1+\kappa]$ for some small constant $\kappa \ge 0$. Suppose each node $v_j \in V$ is connected in $\decoder'$ to some set of nodes $B_j$ in the hidden layer, $B_j = \{\bar \beta : \decoder'_{j, \bar \beta} \neq 0\}$, and let these connections weight be drawn uniformly $\decoder'_{j, \bar \beta} \in [\frac{1-\kappa}{\sqrt{|B_j|}}, \frac{1+\kappa}{\sqrt{|B_j|}}]$ for $\bar\beta \in B_j$. Let $C_j = A \cap B_j$. Define the ratio:
\begin{equation*}
\varrho :=
\left.
\sqrt{\frac{|C_j|}{|A|} \cdot \frac{|C_j|}{|B_j|}}.
\right.
\end{equation*}
Then, there exists an absolute constant $c>0$, such that for any value of $w_j$ (where we recall that $w := \decoder' P P^T u$), we have:
\begin{equation}\label{eq:fscore}
\Pr\left[w_j \ge (1-\kappa)^2 \varrho  - c \sqrt{\log n\ /\ d} \right] = 1 - O(1/n) \quad\textrm{and}\quad\Pr\left[w_j \le (1+\kappa)^2 \varrho + c \sqrt{\log n\ /\ d}\right] = 1 - O(1/n).
\end{equation}
Thus, the value of $w_j$ can be used by a neuron to obtain an estimation of $\varrho$, and apply a threshold to activate accordingly.
\qed
\end{claim}
\begin{proof}
Observe that $w_j = \braket{(P^T D'_{j,\cdot})}{(P^T u)}$. As $\|D'_{j,\cdot}\|_2\in [1-\kappa,1+\kappa]$, a standard application of \JL to vector inner products gives
$\Pr\left[|w_j - \braket{D'_{j,\cdot}}{u}| \le c \sqrt{\log n\ /\ d}\right] = 1 - O(1/n)$ for $c$ large enough. Since $\braket{D'_{j,\cdot}}{u} \in [(1-\kappa)^2 \rho, (1+\kappa)^2 \rho]$, the claim follows.
\end{proof}
\begin{figure}
\includegraphics[trim=0 1cm 0 0,clip,width=\textwidth]{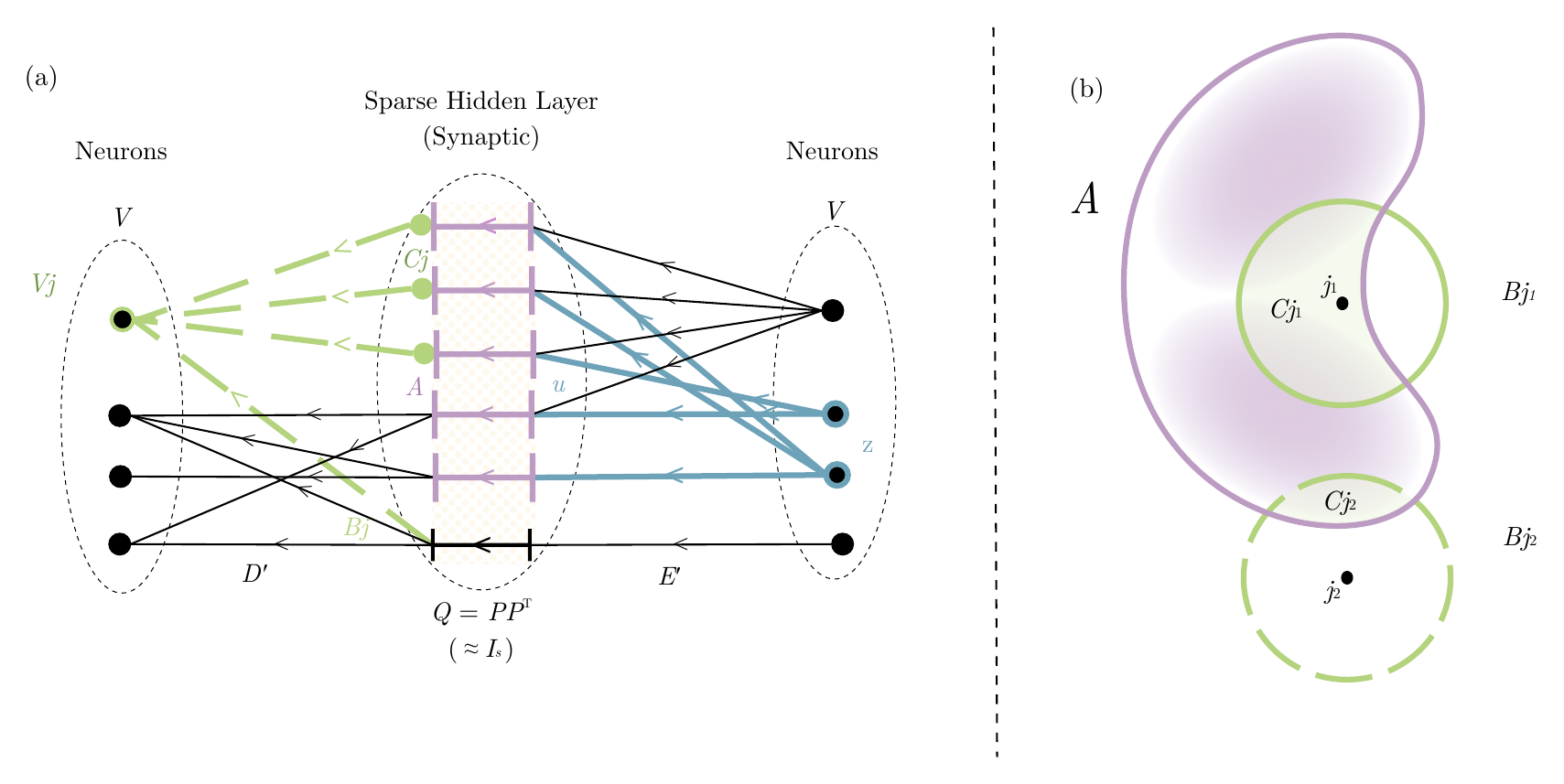}
\caption{The ReLU-lowrank feedforward network of \BDHGPU allows neurons to activate when triggered by activation signals in its own community. (a) Illustration of the selective neuron activation pattern in the proof of Claim~\ref{obs:fscore}, showing the activation decision of node $v_j$ (left) based on active set $A$ in the sparse hidden layer. (b) Illustration of \eqeqref{eq:fscore} showing the relationship between sizes of sets in the sparse hidden layer: active set $A$, set $B_j$ connected to neuron $v_j$, and the intersection $C_j = A\cap B_j$: neuron $v_{j_1}$ becomes active, but neuron $v_{j_2}$ does not.}
\label{fig:fscore}
\end{figure}
The ReLU-lowrank operation $f_{\decoder\encoder}$, after adding appropriate negative bias, can thus be used to propagate positive affinity functions $G'$ on input vectors, performing the following form of thresholding: neurons $j$ in the output layer individually compute a form of local ``F-score'' $\varrho$ given by~\eqeqref{eq:fscore} of the activation of the positive sparse hidden layer, and decide based on it whether they are good match for the output activation; if the threshold condition on $\varrho$ is not met, the neuron $j$ does not activate in the output vector (see \figref{fig:fscore} for an illustration).

\Eqeqref{eq:fscore} naturally coincides with a pattern of communication within network graphs $G'$ admitting positive Newman modularity~\cite{doi:10.1073/pnas.0601602103}, allowing nodes $v_j$ to correctly receive messages $u$ which in the hidden layer primarily reached a denser cluster of $G'$ containing $v_j$. For a specific illustration, let $H$ be an undirected $k$-block stochastic block model (SBM) network~\cite{sbm2011} with $k\in\Nat$ blocks of $n/k$ nodes each, in-block edge density $p$ and out-of-block edge density $q<p$. We put $G' := \decoder'\encoder'=H^2$, i.e., the first connection layer of $G'$ is $\encoder'=H$ and the second connection layer is also $\decoder'=H$. Suppose that $H$ is a random SBM graph with positive Newman modularity separated from $0$, i.e., let $\mu = \frac{k-1}{k} \frac{p-q}{p+(k-1)q} > 0$. Following Claim~\eqref{obs:fscore} with $\kappa = 0$, we can find a ReLU-lowrank representation $G$ to achieve a communication scheme on $G'$, such that a message sent from one node $z = v_i$ activates a node $v_j$ when $i$ and $j$ are in the same block with probability $1-O(1/n)$, and with probability $O(1/n)$ otherwise, when $\mu > \frac{1}{p}\sqrt{\log n\ /\ d}$.

We thus make the following intuitive observation.
\begin{observation}[in-cluster signal reinforcement]\label{obs:sbm}
The ReLU-lowrank representation of $\BDHGPU(n,d)$ is sufficient to represent in-cluster information spreading dynamics in models of graphs with constant in-cluster density and arbitrarily small positive modularity (such as the $k$-cluster Stochastic Block Model) when $d / \log n = \omega(1)$ is an arbitrarily slowly growing function.
\end{observation}

While Claim~\ref{obs:fscore} and Observation~\ref{obs:sbm} are made with reference to an almost-uniform distribution of signal $u$ on the set of nodes of the middle layer, $u$ can have (and in practice does have) a distribution of density which is non-uniform, e.g., going across $a=O(\log n)$ different clustering scales, with a $(1/a)$-fraction of the signal represented at each scale. This allows neurons in the output layer to combine a smaller number of strong signals in its local cluster, with a larger number of weaker ones spread more globally. Such an approach coincides with the observed structure of the graph $\decoder'\encoder'$, discussed in Subsection~\ref{sec:bdh_empiricalgraphs}.

\paragraph{Supermodularity on input perturbation.}
We clarify how the properties of function $f_{\decoder\encoder} :  (R^+)^n \to (R^+)^n$ relate to the previously discussed ability to make an input signal resonate ``within a module'' in a graph with hidden modular structure. First, note that $f_{\decoder\encoder}$ \emph{is a subadditive function, but is not submodular in general} with respect to the set of $n$ coordinates of its input vector. In some of the regimes in which it appears to be operating, locally $f_{\decoder\encoder}$ exhibits a form of behavior opposite to submodularity, referred to as `supermodularity', or `increasing returns' of adding new coordinates to the input vector. This is already implicitly captured by Claim~\ref{obs:fscore}, but we can consider a simpler example.

Take a variant of the setting from Observation~\ref{obs:rw} with the same choice of $G'$, and let $z \in (R^+)^n$ and biases of $\decoder$ be chosen so that all coordinates of $\decoder\encoder z$ are approximately equal to $-1.5/r \pm o(1)$ (this can be done by choosing e.g. $z_j = 1/n$). Then, $f_{\decoder\encoder}(z)=0$, and for any $v_i, v_j \in V$, $f_{\decoder\encoder}(z+v_i) = 0$ a.s., $f_{\decoder\encoder}(z+v_j)= 0$ a.s., but $f_{\decoder\encoder}(z+v_i+v_j)$ has non-zero coordinates a.s. with values approximately $1/{2r}$, for all nodes $v_k$ which are common out-neighbor nodes of $v_i$ and $v_j$, i.e., for all $k$ such that $G'(v_i, v_k) = G'(v_j,v_k) = 1/r$. This mechanism generalizes to finding common neighborhoods which have many connections to two given subsets of nodes, $V_a$ and $V_b$. In a setting where the considered affinity $G'$ is bi-directional (e.g., a symmetric matrix), this corresponds to finding shortcut nodes, allowing to go from $V_a$ to $V_b$.

It follows that the neighborhood-reinforcing nature of the threshold dynamics of \BDHGPU, which plausibly follows from the logic of its role in inference and from the needs for an efficient computational process, is starkly different from the more often studied submodular behavior of threshold and cascade dynamics on real-world networks~\cite{kempe2003maximizing}, and plausibly, much less smooth when considered as a dynamical process.

\subsection{Empirical findings: parameter distribution in ReLU-lowrank matrix products}\label{sec:bdh_empiricalgraphs}

We consider the $\decoder$ matrices (in the same way $\decodery$ and $\decoderx$) and $\encoder$ matrix obtained after training of \BDHGPU models, and used in the ReLU-lowrank operation \eqeqref{eq:relulowrank}, $f_{\decoder\encoder}(z) = \relu{\decoder\encoder z}$.

\paragraph{Choice of prior of matrix parameter distributions.} Following the discussion in Section~\ref{sec:modularity_detailed}, we expect matrix $G := \decoder\encoder$ to reflect the clustering (modularity) structure of the neuron-neuron communication graph. Any plausible parameter distribution of matrix $G$ must therefore allow heavy-tailed distribution of entries. At the same time, a Gaussian noise term is inherent to low-rank matrix representation, and needs to be taken into account together with this heavy-tailed distribution.

We now provide a somewhat more fine-grained explanation, which leads to the prior on the structure of matrix $G$ as given by \eqeqref{eq:epsilon}. Consider a training set-up in which the ReLU-lowrank operation described by matrix $G$ is treated as an approximation of the same operation, governed by a high-rank matrix $G'$, with $f'(z) := \relu{G'z}$. Considering this block in isolation from the rest of the training system, the training of matrices $\decoder$, $\encoder$ goal corresponds to learning an approximation  of $f'$, with $\decoder \in  R^{n,d}, \encoder \in  R^{n,d}$, such that $f(z) \approx f'(z)$ holds for some class of vectors $z$.

For the rest of this analysis, we will consider the function $f'$ as a ground truth reference for the intended operation of the ReLU-lowrank block. This type of analysis can be seen as plausible over short time spans in later phases of training of a \BDHGPU model, i.e., once individual neurons in $R^n$ have started to admit semantic or functional meaning, and so when function $\decoder'\encoder'$ describes a property of the problem being solved in a (frozen) concept space, and not a co-learning process between the representation of the concept space in $R^n$ and the functions applied to it.

We can represent $G' := \decoder'\encoder'$, where $\decoder' \in  R^{n,s}$, $\encoder' \in R^{s,n}$, with $s = O(n^2)$, are in general matrices of rank $n$; we have $f'(z) := \relu{\decoder'\encoder'z}$. Without loss of generality, we can choose from among the possible representations one with the following distribution of positive and negative elements: $\decoder' \in (R^+)^{n,s}$, $\encoder' = {\encoder'}\ee -  {\encoder'}\ii$, with ${\encoder'}\ee, {\encoder'}\ii \in (R^+)^{s,n}$. We will write: $G' = {G'}\ee - {G'}\ii$, where ${G'}\ee  = \decoder'{\encoder'}\ee$, and ${G'}\ii  = \decoder'{\encoder'}\ii$. The main purpose of the chosen representation $G' = \decoder'\encoder'=\decoder'({\encoder'}\ee - {\encoder'}\ii)$ is to have matrices $\decoder'$, ${\encoder'}\ii$, ${\encoder'}\ee$ with much smaller outlying elements compared to matrix $G'$, which leads to more justified conclusions about the uniform nature of the noise introduced by the low-rank decomposition.\footnote{For a specific example, one very broad class of matrices $G'$ is given by the product of sparse matrices $\decoder'$, ${\encoder'}$, in which each $s$-element row of $\decoder'$ (column of ${\encoder'}$) has at most $\Delta \ll n$ non-zero elements, each with value bounded by $O(1/\sqrt{\Delta})$, and all remaining $s-\Delta$ elements of these matrices are equal to $0$. The resulting elements, ${G'}_ij = \sum_{\alpha} \decoder'_{i,\alpha}\{\encoder'\}_{\alpha,j}$, may be much less uniform, only satisfying ${G'}_ij=O(1)$. This type of scenario captures the expressiveness of set intersection for ``bag-of-words'' models for language, or expressiveness of ``hub label'' representations for a measure of node proximity in a directed graph.}

Assume now that we learn to approximate function $f'$ with $f_{\decoder\encoder}$ by trainable matrices $\decoder, \encoder$ through the following low-rank scheme:
$$
G = \decoder \encoder := (B_D + \decoder' P) (P^T \encoder' + B_E^T),
$$
where $P \in R^{s, d}$ is \emph{non-parametric} and the result of random sampling an almost-orthonormal random projection so that $PP^T \approx I_s$ (e.g.\ $P \sim \gauss(0,1/\sqrt{d})^{s, d}$), and $B_D, B_E \in R^{n,d}$ represent \emph{trainable} additional terms for compensating error or introducing bias, with the goal of minimizing some loss function $\Loss(f',f_{\decoder\encoder})$. The terms $B_D, B_E$ compensate the error introduced by the approximation $PP^T \approx I_s$, after the ReLU operation.

Let $Q := P P^T = I_s + \delta_I + \delta_Q$, where $\delta_I \in R^{s\times s}$ is a diagonal error matrix, and $\delta_Q \in R^{s\times s}$ is a non-diagonal (hollow) matrix. We have:
\begin{equation*}
G = \decoder \encoder = ({G'}\ee - {G'}\ii) + \decoder' \delta_I ({\encoder'}\ee-{\encoder'}\ii) + \underbrace{\decoder' \delta_Q \encoder'}_{\eps_Q} + \underbrace{(B_D\encoder + \decoder B_E^T)}_{\eps_B}.
\end{equation*}
Since all elements of $\decoder', {\encoder'}\ee, {\encoder'}\ii$ are non-negative and $I_\delta$ is diagonal, we can represent elements $G_{ij}$, for $i,j \in 1,\ldots, n$, as follows:
\begin{equation}\label{eq:epsilon}
G_{ij} = (1+\eps_{\delta\,ij}\ee){G'}\ee_{ij} - (1+\eps_{\delta\,ij}\ii) {G'}\ii_{ij} + \eps_{Q\,ij} + \eps_{B\,ij}
\end{equation}
where $|\eps_{\delta\,ij}\ee| = O(\sqrt{\log n\ /\ d})$ and $|\eps_{\delta\,ij}\ii| = O(\sqrt{\log n\ /\ d})$ have the interpretation of small multiplicative distortion.

Following~\eqref{eq:epsilon}, we expect the elements of $G$ to be distributed as the sum of four different distributions.
The term ${G'}\ee_{ij}$ has the interpretation of positive ground truth elements of $G'$. The term $-{G'}\ii_{ij}$ has the interpretation of negative ground truth elements of $G'$; its use in combination with the ReLU mechanism can be interpreted as inhibitory action. Both of these terms are subject to slight multiplicative distortion.

The term $\eps_{Q\,ij}$ has the interpretation of non-trainable noise (which depends only on $\decoder'$, $\encoder'$ and the random choice of $P$). Under reasonable assumptions on outlying elements of $\decoder', \encoder'$, it is a form of almost-Gaussian symmetric noise inherent to the considered class of low-rank projections, $\eps_{Q\,ij} \rightarrow N(0,\sigma_Q)$, for some $\sigma_Q \in R^+$, and the expected value of this noise is typically very close to $0$, even when considering the expectation of $\eps_{Q\,ij}$ conditioned on known values of $\eps_{Q\,i'j'}$ for a small number of indexes $(i',j')$ in the matrix.

Finally, $\eps_{B\,ij}$ is a trainable term, whose norm tends to $0$ as $d$ increases. We expect it to have the interpretation of bias used to offset the low-rank Gaussian noise and perform denoising in the ReLU-gate, as previously discussed in Section~\ref{sec:dynamics}. Because of the action of the ReLU gate, we plausibly expect the distribution of $\eps_{B\,ij}$ to be skewed towards negative numbers, with $0 > \E\eps_{B\,ij} \gg \sigma_Q$.

From the above discussion of the four terms of the sum in \eqeqref{eq:epsilon}, we see that only one of these terms, ${G'}\ee_{ij}$, is expected to take values much larger than $\sigma_Q$ with non-negligible probability. We reach the conclusion that a part of the relevant signal of $G'$ is concentrated in the right tail of large positive matrix entries of $G$.

\begin{hypothesis}[right tail contains signal]\label{hyp:righttail}
Consider the interpretation that the ReLU-lowrank transformation $z \mapsto \relu{G z}$, with $G=\decoder\encoder$, has learned to act as an approximation of some other operation $z \mapsto \relu{G' z}$, where $G'$ has no low-rank constraint imposed on it. Then the right tail of the distribution of matrix elements of $G$ corresponds to the right tail of the distribution of elements of $G'$, starting from some positive threshold value $\sigma_Q$, associated with the noise of the low-rank decomposition. Formally, for almost all pairs of indices $i,j\in 1,\ldots,n$ such that $G_{ij} \gg \sigma_Q$, we also have ${G'}\ee_{ij} \gg \sigma_Q$.
\end{hypothesis}
The converse implication, that ${G'}\ee_{ij} \gg \sigma_Q$ implies ${G}_{ij} \gg \sigma_Q$, also plausibly holds under some stronger assumptions on the form of biases $\eps_{B\,ij}$ which may follow from minimizing training error for the specific inference task considered.

This direct method of decoding $G$ from $G'$ does not extend from the right tail towards the center of the distribution. For the choices of $n, d$ we make, we expect the term dominating most elements of matrix $G$ to be $\eps_{Q\,ij}$. For example, when $G'$ is a stochastic matrix, we expect to have $\sigma_Q = O(1/\sqrt{d})$ (cf.\ \eqeqref{eq:lowrank} for the corresponding infinity-norm bound, $|\eps_{Q\,ij}| = O(\sqrt{\log n\ /\ d})$). With $\sum_{i,j} |{G'}\ee_{i,j}| = n$ for a stochastic matrix,
we expect the right heavy tail of the element distribution of $G$ to have $\Omega(n \sqrt d)$ elements (out of the $n^2$ matrix elements of $G$) which are clearly separated from the Gaussian noise.

We confirm empirically that the right tail of $G$, defined as above with respect to threshold $\sigma_Q$, turns out to contain a non-negligible portion of the parameter capacity of matrices $\decoder$, $\encoder$, even for very small models (10M to 100M parameters).

\paragraph{Experimental setup.}

We prepared parameter matrices of a 24M-parameter \BDHGPU model configured with $h=4$ heads and $L=8$ layers, $n=h\cdot 2^{13}= 2^{15}$ neurons, and hidden low-rank dimension $d=256$. We considered the weighted neuron-neuron interaction graph, having the encoder-decoder matrix pair $G = \decoderx \encoder$ as its node adjacency matrix on the set of neurons $V = 1, \ldots, n$. For uniformity, we subsampled $G$ by picking node subsets $V^{(a)}$, $a\in \{1,2,3,4\}$, associated with each head, and considered the weighted subgraphs $G^{(ab)} = \{V, \{(u, v, G_{uv}): u \in V^{(a)}, v\in V^{(b)}\} \}$, with $G^{(ab)} \in R^{n^*\times n^*}$ where $n^* = n/h = 2^{13}$, each having $(n^*)^2 = 2^{26}$ weighted edges.

We repeated the experiment $5$ times using models pretrained with different random seeds.

\paragraph{Findings.} For all of the $5$ models we pretrained for this purpose, exactly $3$ out of the $4$ encoder heads and all decoder heads adhered to the prior on parameter distribution given by~\eqeqref{eq:epsilon}, showing a good correspondence for $12$ out of $16$ of their parameter sub-matrices $G^{(ab)}$.

We continue the discussion in this Section for one specific matrix $G^{(ab)}$ of one specific pretrained models, which was chosen as representative. The example we choose has $a=b$; and so the matrix $G^{(ab)}$ has an interpretation as $G[V_a]$, i.e., the subgraph of $G$ induced by vertex set $V_a$, which enables us to visualize the graph $G^{(aa)}$ more easily on its vertex set $V_a$.

We refer to the representative object of our study, i.e., to the matrix $G^{(aa)}$ of the selected model, as $G^*$. For any matrix $A$ and $\beta\geq0$, we denote by $A_{\geq \beta}$ the matrix $A$ cut off at threshold $\beta$, i.e., ${A_{\geq \beta}}_{ij}=A_{ij}$ if $A_{ij}\geq \beta$, and ${A_{\geq \beta}}_{ij}=0$ otherwise.

The distribution of elements $G^*_{i,j}$ is presented in \figref{fig:prettydistro}~(a).
\begin{figure}
\includegraphics[width=\textwidth]{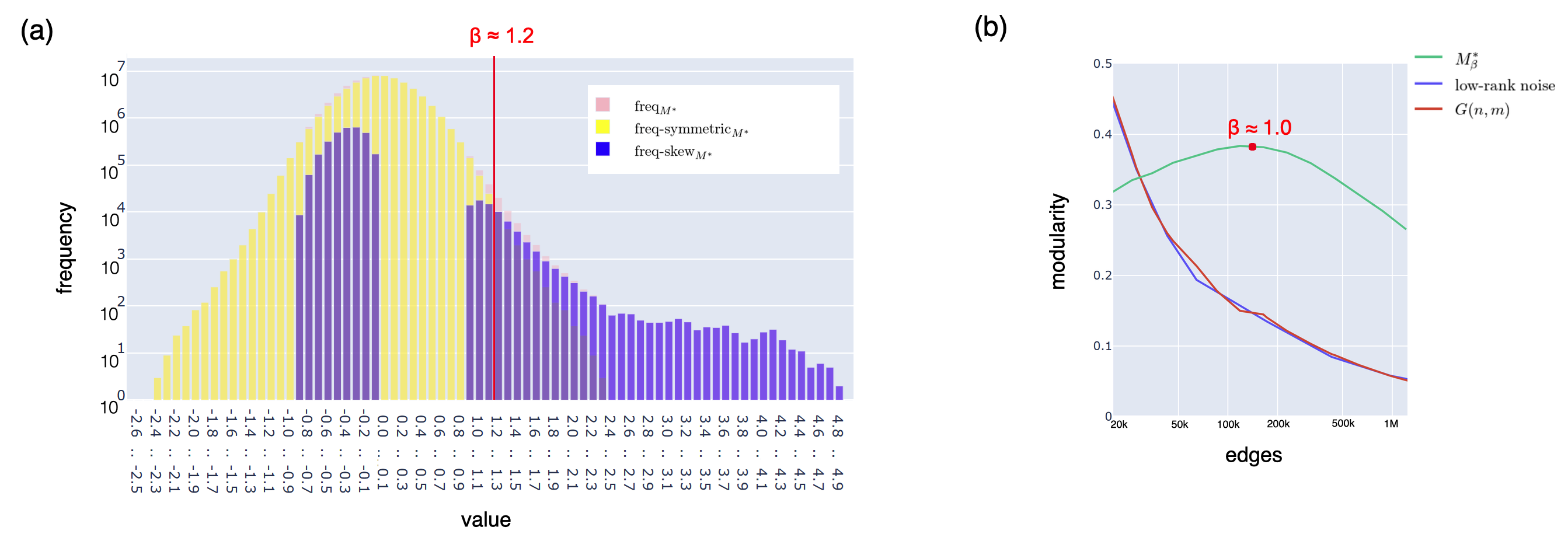}
\caption{(a) Heavy-tailed element distribution and modularity analysis of the excitatory neuron-neuron connection graph contained the encoder-decoder matrix $G^*$. Distribution of elements of the encoder-decoder matrix $G^* \in R^{n^* \times n^*}$ of a \BDHGPU model with $n^*=8192$ neurons and $d=256$: histogram $\mathrm{freq_{G^*}}(x)$, its symmetric part $\mathrm{freq-symmetric_{G^*}}(x) := \min\{\mathrm{freq_{G^*}}(x), \mathrm{freq_{G^*}}(-x)\}$, and distribution skew $\mathrm{freq-skew_{G^*}}(x) := \mathrm{freq_{G^*}}(x)-\mathrm{freq-symmetric_{G^*}}(x)$. $\diamond$ (b) Estimate (lower bound) of Newman modularity of matrix $G^*_{\geq \beta}$ for different values of $\beta$, plotted as a function of the number of non-zero elements (edges) of $G^*_{\geq \beta}$. Modularity of random graph baselines are provided for reference, for the $G(n^*,m)$ model with the same number of edges as $G^*_{\geq \beta}$, and for a matrix $(P_1P_2^T)_{\geq \beta'}$ with the same number of edges as $G^*_{\geq \beta}$, where $P_1, P_2 \sim \gauss(0,1)^{n^*\times d}$. The modularity estimates were obtained using the community structures returned by the Louvain algorithm, in the best of 5 clustering runs with different random seeds.}
\label{fig:prettydistro}
\end{figure}
We find that the observed distribution corresponds well to the prior expected of it by~\eqeqref{eq:epsilon}. We determine the threshold value $\beta \geq 0$ at which we expect to capture signal, ${G^*_{\geq \beta}} \approx {G'_{\geq \beta}}$, following Hypothesis~\ref{hyp:righttail}. We find (from \figref{fig:prettydistro}(a)) that the separation from noise happens for this specific matrix $G^*$ at $\beta_1 \approx 1.2$, at which point the right heavy tail begins to dominate. However, already for much smaller values of $\beta$ we find that ${G^*_{\geq \beta}}$ has high modularity, and this actually increases as more non-zero values are added to ${G^*_{\geq \beta}}$ for smaller $\beta$, up to a maximum at $\beta_2 \approx 1.0$ (\figref{fig:prettydistro}(b)). Even for much smaller values of $\beta$, the modularity of ${G^*_{\geq \beta}}$ remains almost constant up to well above $2^{20}$ non-zero matrix entries on the $n^*=2^{13}$ nodes considered. The modularity of the baselines, of random graphs or random low-rank matrix products, quickly drops to 0 in this regime. This should be compared to the total number of parameters of the matrices $\decoderx, \encoder$ corresponding to $G^*$, i.e., $2\cdot 2^{13} \cdot 2^8 = 2^{22}$ parameters. A complementary analysis of the inhibitory signal, for a similarly defined matrix $|G^*_{\leq -\beta}|$, also finds that this structure has high modularity.

In auxiliary experiments, we looked at basic graph parameters of matrix $G^*_{\geq \beta}$, treated as a directed graph on its set of nodes. We set $\beta=1.2$, obtaining $m=46820$ non-zero entries (edges) in $G^*_{\geq \beta}$.
We found that $G^*$ has a heavy-tailed, power-law-like degree distribution, with generally more concentrated out-degree than in-degree (\figref{fig:powerlaw}(a)).
\begin{figure}[ht]
\noindent%
(a)\\
\includegraphics[width=\textwidth]{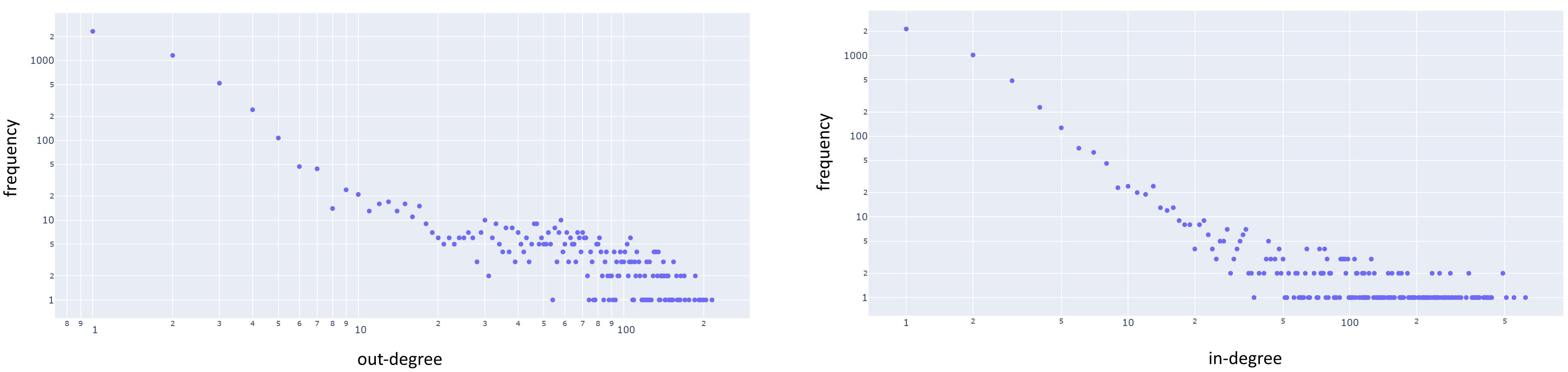}\\
\noindent%
(b)\\
\hspace*{0.1\textwidth}\includegraphics[width=0.8\textwidth]{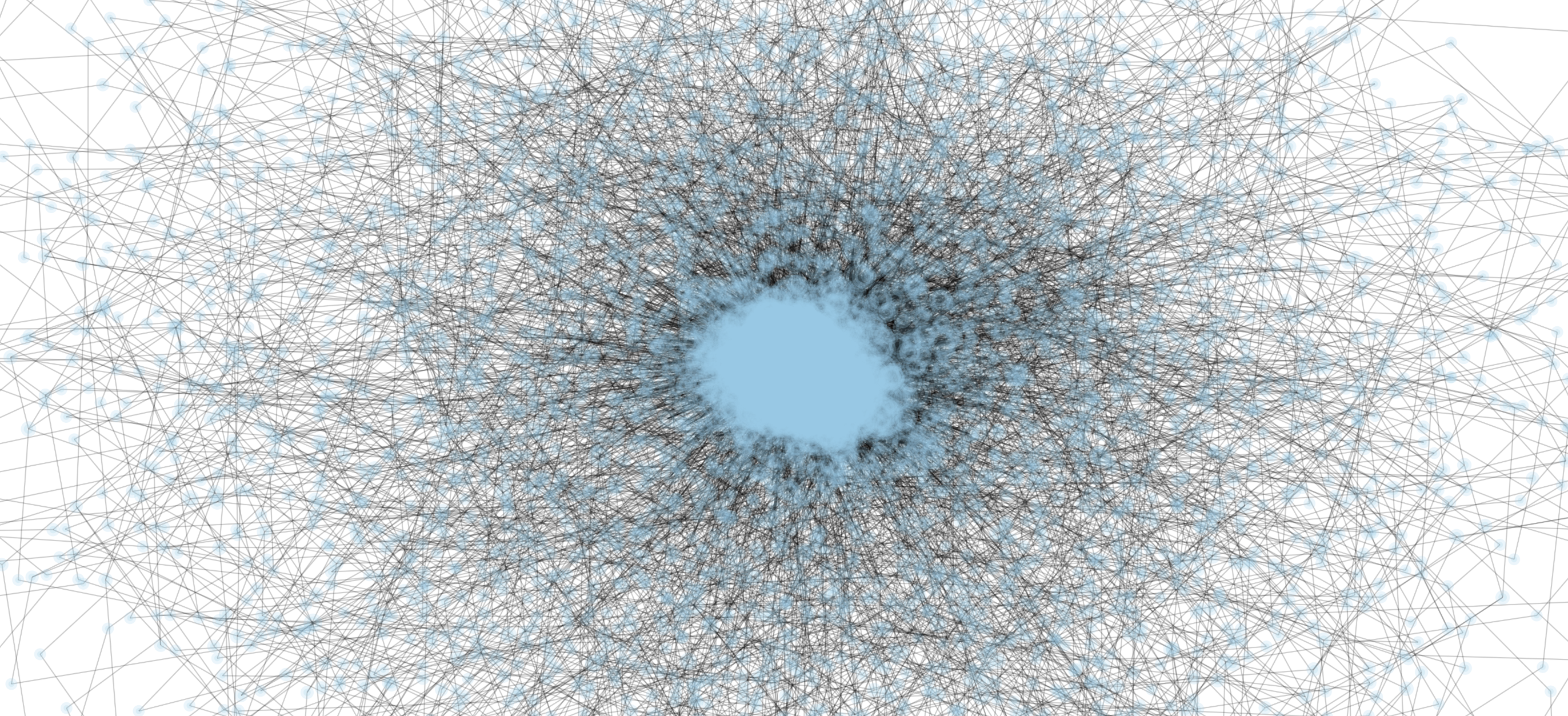}
\caption{(a) Unweighted in-degree and out-degree distribution for the $n^*=8192$ neuron nodes and $m=46820$ edges of matrix $G^*_{\geq \beta}$ with $\beta=1.2$. The distributions exhibit power law distributions, with different exponents, the out-degree distribution being more concentrated. (b) Visualization of graph $G^*_{\geq \beta}$, hinting at its core-periphery structure.}
\label{fig:powerlaw}\label{fig:powerlaw2}
\end{figure}
Generally, this finding is consistent with expectations as to the structure of a network with positive modularity. The difference of in- and out-degree distributions, while plausible and prevalent in real-world information dissemination networks, was not considered in Subsection~\ref{sec:dynamics}.

Finally, a visualization of $G^*_{\geq \beta}$ (\figref{fig:powerlaw}(b)) exhibits a core-periphery structure. This is again consistent with the expected modular structure.

\begin{finding}
We confirmed that during training, \emph{a graph structure with positive modularity} appears in \BDHGPU model parameter matrices $\decoderx \encoder$ and $\decodery \encoder$. This modular structure plausibly follows from the network's inference function, and specifically from the cluster-aware information propagation dynamics supported by the ReLU-lowrank mechanism (Observation~\ref{obs:sbm}).
\end{finding}

We also observed that for all of the studied models with $h=4$ heads, $1$ encoder sub-matrix out of $4$ has no heavy positive tail, and generally appears to capture a form of inhibitory structure ${G'}\ii$ from \eqeqref{eq:epsilon}. Since we have not provided convincing mechanisms for isolating negative signals in $G^*$ and these are easily confounded with the bias term $\eps_{B}$, we omit this case from discussion. We remark that the apparent need for passing activations through such a separate ``inhibitory circuit'' is one of the most evident explanations for why introducing (a small number of) heads to \BDHGPU provides an improvement in model quality.

\section{Analysis: linear attention, sparse positive activation, and monosemanticity}\label{sec:bdh_lin_attention_monosemanticity}\label{sec:attention}

\subsection{Macro-expressiveness of attention in \BDHGPU}\label{sec:bdhgpumacroexpattention}

The attention mechanism of \BDHGPU can be described at a coarse-grained level as a transformation mechanism for key-query-value vectors, similar to that in the Transformer. This description is complementary to the interpretation of the \BDHGPU attention mechanism at the micro-level of correlations between neuron pairs, which we defer to Section~\ref{sec:micro_attention_bdh_gpu}, which provides more insight into the way activation vectors used by \BDHGPU relate to the concept space of the model.

We compare the attention mechanism of \BDHGPU with the attention mechanism of the Transformer, describing both as reflections of a general attention mechanism. Specifically, we explain why, and up to what context length, the linear attention mechanism of \BDHGPU plausibly fits into macro-expressiveness frameworks of attention designed for the Transformer (based on RASP).

\paragraph{Basic properties of \BDHGPU attention.} The key-query space for \BDHGPU is $R^n$, the same as its neuron dimension, rather than the small dense dimension used by the Transformer. The keys and queries used by \BDHGPU are given by positive vectors, in $(R^+)^n$, and are expressed by the same vector $x_{t,l}$ (noting that at time $t$, $x_{t,l}$ is used as a query, and only $x_{\tau,l}$, for $\tau \leq t-1$, are used as keys).

`Value' vectors of \BDHGPU remain in the small dimension, $R^d$, which at some model scales is comparable to the dimension used for attention `values' in common configurations of the Transformer. %

The relationship between softmax-based attention of the Transformer, regarded as a low-dimensional kernel for general linear attention, and linear attention for vectors in the positive orthant, was considered in a framework called FAVOR+~\cite{choromanski}. Here, we provide a few complementary (simpler) observations, sufficient to grasp the main effects of the ability of Linear Attention to distinguish facts in context.

\paragraph{State capacity vs. distinction capacity.}

The matrix $\state \in R^{n \times d}$, which is used to represent state for each layer of \BDHGPU, should theoretically have sufficient capacity to store $O(n)$ `value' vectors in $R^{d}$ if considered as a lookup table for values. We now remark that its actual capability of \emph{distinguishing facts} using the linear attention mechanism is also asymptotically close to $n$.

Attention is a mechanism of associative memory which, given a series of key-value pairs $((k_1,v_1) \ldots, (k_t,v_t)) \in (\Lambda_k \times R^d)^t$, a query $q \in \Lambda_q$ and an affinity function $\phi(\cdot,\cdot): \Lambda_q \times \Lambda_k \to [0,1]$  between the space of queries and keys, returns the attention value:
$
a_t = \sum_{\tau=1}^{t-1} \phi(q, k_\tau) v_\tau
$
(or a normalization thereof). With \BDHGPU, we consider `value' vectors $v \in R^d$, where $d$ is small. The spaces of keys and queries may be assumed to coincide as $\Lambda = \Lambda_k = \Lambda_q$, and we consider in general a single key-query sequence, given by $(k_t)_{t\in \Nat}$:\footnote{This assumption is known to have moderate practical implications for trainability. In this specific discussion, it is `without loss of generality', since one can consider $\Lambda = \Lambda_1 \tensor \Lambda_2 \tensor \ldots \tensor \Lambda_t$, and consider each $k_i$ as chosen from $\Lambda_i$, defining affinity $\phi(k_t, k_\tau) : \Lambda_t \times \Lambda_{\tau} \to [0,1]$ appropriately to handle successive keys and queries (effectively describing a general form of positional embedding).}
\begin{equation}
\label{eq:attn2a}
a_t = \sum_{\tau=1}^{t-1} \phi(k_t, k_\tau) v_\tau
\end{equation}
This key-query space $\Lambda$ may be considered as an abstract space, and represented in any way which is convenient, for as long as the affinity function $\phi(k_t, k_\tau)$ is preserved. For example, when the keys and queries are sampled from a finite (though possibly extremely large) set, there also exists some vector space dimension $\nu$ (possibly extremely large) and a function mapping $f: \Lambda \to S^{\nu}$, where $S^\nu = \{z \in R^{\nu} : \|z\|=1\}$ is the unit sphere, such that the scalar (dot, cosine) product in $S^{\nu}$ satisfies $f(k_t)\cdot f(k_\tau) = \phi(k_t, k_\tau)$. In other words, any affinity function $\phi$ becomes linear when represented in sufficiently high dimension, subject to suitable preparation of its arguments with function $f$. With $\nu$ extremely large, $S^{\nu}$ is a sort of Platonic ideal of a space in which the attention keys and queries live, with no relation to any specific model.

This type of argument, often used in considerations of Support Vector Machines, is linked to two challenges: (1) ensuring that the dimension actually considered by the network (in our case $n$) is high enough compared to the hypothetical dimension $(\nu)$, and (2) ensuring that a suitable preparation function $f$ exists and is easy to learn for the model.\footnote{The Transformer can also be positioned in the same SVM framework: the Transformer's attention represents a form of ``kernel trick'' for one specific affinity function $\phi$, with the kernel used to approximate it being the exponential function (in the case of softmax attention).} We now explain when the dimension $n$ can be considered sufficient, and what types of keys can be prepared by \BDHGPU.

\paragraph{Expressiveness of linear attention in dimension $n$.} The Linear Attention mechanism aggregates key-value correlations over time. In general, the associated rate of accumulation of noise is manageable, up to the approximate scale of between  $t = \Omegatilde(\sqrt n)$ and $t = \Otilde(n)$ key-value `facts' stored in the attention of a given layer. We make the following statement about the Linear Attention mechanism in general.

\begin{claim}[informal statement]
\label{claim:linearinformal}
The mechanism of Linear Attention, applied in dimension $R^n$, can approximately express an attention affinity function for up to $t = \Otilde(n)$ `key-value' pairs in context, with `values' having comparable L2-norm, under moderate assumptions on weak correlation of historical keys and uniformity of the expressed affinity function. Without such assumptions, Linear Attention can compute the correct affinity up to at least $t = \tilde\Omega(\sqrt n)$ `key-value' pairs in context, except for a negligible fraction of possible inputs. Keys and queries need to be suitably prepared beforehand.
\end{claim}
\emph{The formal statement and proof is provided in Appendix~\ref{apx:linear}.}\qed

The above claim captures the expressiveness of Linear Attention in dimension $R^n$, subject to some way of preparing keys and queries in $R^n$ by the model in blocks preceding the attention block. A model using Linear Attention has to learn its own way to prepare keys. In fact, different \emph{natural} approaches to key preparation, for example using random projections or hashing on a set of $t$ vectors, lead to the same asymptotic statement of Claim~\ref{claim:linearinformal}. (In the proof in the Appendix, we chose to use a particularly simple one.)

The specific way of preparing keys used (learned) by \BDHGPU for its Linear Attention is particularly interesting. Except for the effect of RoPE rotation, which introduces a negative positional effect in the affinity of keys and queries, \BDHGPU uses activation vectors (keys, queries) with only positive coordinates to represent its keys.

We discuss some aspects of how the positive activation vectors of \BDHGPU relate to Linear Attention.

\paragraph{Preparation of positive keys for Linear Attention.}

Activation vectors of \BDHGPU belong to the positive orthant, and are often sparse. The interpretation of such vectors depends on whether we consider the positive orthant to be a ``valid shape'' for the latent concept space of the considered task (in this case, language and reasoning), or whether the task has to be embedded into such a space. For language, this would be a question of whether a \textsf{word2vec}-like internal representation of the concept space by the model has an inherent advantage over a 
\textsf{bag-of-words}-like representation, especially when expressing concept affinities in attention.

We note that latent representation of key and query vectors in the positive orthant is natural for any problem which is \emph{amenable to attention}. In the discussion of general attention given by \eqeqref{eq:attn2a}, we noted that the affinity function $\phi$ takes values in $[0,1]$, and we considered an embedding $f$ of a set of key vectors $k_1,\ldots,k_t$ into $R^\nu$ such that $f(k_t)\cdot f(k_\tau) = \phi(k_t, k_\tau) \geq 0$. Given this condition on non-negativity of dot product on all pairs among the $t$ vectors considered, we could have, without loss of generality, used an appropriately rotated embedding $f$ so that $f(k_\tau) \in (R^+)^\nu$, thus \emph{directly reducing the problem of general attention to a problem of linear attention in the non-negative orthant}. The question which remains is a subtle one: whether this type of embedding of the latent space of language and reasoning in $(R^+)^\nu$ is `natural', i.e., preserved over long periods of time of inference and training, notably longer than the short window $t$ of context used for Transformer-like attention.

In the rest of the paper, we are generally inclined to assume that representations in $(R^+)^\nu$ of concepts, combinations of concepts, and density distributions over such combinations of concepts, are universal to language and reasoning.

We limit ourselves to a very brief discussion of a way to represent attention keys with positive vectors for problems for which such a concept representation is not natural.

\subparagraph{Using LSH to move key vectors into the positive orthant.}

Locality Sensitive Hashing (LSH) is one technique for converting arbitrary vectors in a lower-dimensional space $R^a$, for some fixed $a \in \Nat$, into vectors in $(R^+)^n$, in a way which can be used to describe certain `sharp-boundary' affinity functions $\phi$ in $R^a$. Consider an $n \times a$ matrix represented as $n$ fixed random vectors $\lambda_1, \ldots, \lambda_n \in R^a$, and a corresponding sequence of $n$ appropriately chosen gating functions $\gamma_1, \ldots, \gamma_n : R \to R^+$. For a vector $v \in R^a$, we define:
\begin{equation}\label{eq:lshb}
b(v) := \gamma([\lambda_1\ldots\lambda_n] v) = (\gamma_i(\braket{v}{\lambda_i}))_{1\leq i\leq n}.
\end{equation}
Each $i$-th element of vector $b$ thus corresponds to the outcome of the $i$-th bucket of LSH.

The bucketing function $b$ may now be used to prepare queries and keys as attention inputs. If $\gamma_i$ is a $\{0,1\}$-valued threshold function, then, for $q, k_i \in R^a$, $\braket{b(q)}{b(k_i)}$ is an attention affinity function between $q$ and $k_i$, equal to the number of LSH buckets shared between $q$ and $k_i$.

\begin{observation}
The LSH vector affinity function $b$, given by \Eqeqref{eq:lshb}, using $n$ buckets on vectors in $R^{a}$ for some $a\in \Nat$, can be expressed through Linear Attention with attention keys in the positive orthant $(R^+)^n$.\qed
\end{observation}
In the ReLU-based setup considered in \BDHGPU, an appropriate function $b$ is plausibly easy to learn. LSH is a `sharp-boundary' technique, well-suited for finding $k$-nearest-neighbors of a queried vector in a set of keys. Hence, the class of attention affinity functions, naturally expressible using \BDHGPU, also includes such `sharp' functions.

\subparagraph{Attention in the positive concept space of language and reasoning.}

\BDHGPU uses the positive orthant $(R^+)^{n}$ as its latent space for representing combinations of concepts in its activation vectors. Attention keys and queries are prepared entirely in this positive orthant.

When representing a task of reasoning or language inference in a high-dimensional space, positive activation vectors  in $(R^+)^n$ have a natural interpretation of convex combinations of concepts. Such convex combinations of concepts may represent both semantically connected concepts (``bags-of-concepts''), and mixed states of uncertainty between unconnected concepts. In this interpretation, a positive vector is considered as a state of certain knowledge when its L1-norm and L2-norm align closely. Note that for a (normalized) probability vector, the only vectors for which L1-norm and L2-norm coincide precisely are distributions concentrated on a single coordinate.

Linear Attention of \BDHGPU is capable of amplifying very small differences between keys in the L1-norm when matching queries to keys. Consider, for instance, two probability distribution vectors $x_1, x_2 \in (R^+)^n$, where $x_1 = (\alpha, \frac{1-\alpha}{n-1}, \frac{1-\alpha}{n-1}, \ldots, \frac{1-\alpha}{n-1})$ and $x_2 = (\frac{1-\alpha}{n-1}, \alpha, \frac{1-\alpha}{n-1}, \ldots, \frac{1-\alpha}{n-1})$, for some $0<\alpha<1$. Now, vectors $x_1$ and $x_2$ almost coincide when treated as probability distributions, $\|x_1 - x_2\|_1 = O(\alpha) = \|x_1 - x_2\|_{\mathrm{TVD}}$. However, they are extremely different when considered as keys for the Linear Attention mechanism, with $x_1$ showing very weak affinity to $x_2$: $\bra{x_1}\ket{x_2} = O(\alpha^{-2} n^{-1}) \bra{x_1}\ket{x_1}$.
\begin{observation}\label{obs:prob}
In key-query matching, the Linear Attention mechanism of \BDHGPU is able to separate positive keys which are close in the L1-norm, strongly amplifying L1-norm differences of activation vectors.\qed
\end{observation}

This mechanism can be treated as complementary to the propagation dynamics of positive activations in the feed-forward network, discussed in Section~\ref{sec:bdh_feedforwardpropagation}.

\paragraph{Natural support for long context.} There is no bound on context length in \BDHGPU, so the actual $t = \Omegatilde(\sqrt n)$ to $t=\Otilde(n)$ ``equally important facts'' that a \BDHGPU model can distinguish in each layer in view of Claim~\ref{claim:linearinformal} do not have to correspond to the latest $t$ ``facts'' seen in context. For example, if, for some layer $l$, mechanisms from lower layers deem a given entry to be irrelevant for layer $l$, and provide an extremely weak attention `value' for this layer, and this key-value entry is effectively seen as omitted. This mechanism corresponds to weaker signals $\xysparse$ in a layer which needs to take no action on a given input, e.g., does not have to remember it (cf.~Fig.~\ref{fig:sparsity}). Indeed, empirically we observe progressive de-noising of state in the higher layers, with only small fractions of input tokens requiring significant key-value state update in the middle layers across the entire spectrum of state $\state$ of neurons.

As a result, the middle and higher layers of \BDHGPU may, in principle, have unbounded look-back on context. Nonetheless, as context length $t$ increases, we find that damping of historical signals over long sequences is necessary in \BDHGPU to avoid overwhelming the model with noise from stale context. For the vanilla version of the architecture, we found that RoPE combined with ALiBi provide a sufficient remedy, and model performance improves as context length increases. More advanced techniques for \BDHGPU, related to selective forgetting, state compression, or other forms of state optimization, can also be added to the architecture.

\subsection{Micro-interpretation of attention in \BDHGPU}\label{sec:micro_attention_bdh_gpu}\label{sec:sparsity}

BDH maintains its state in the $n\times n$ matrix $\corr$ that has a clear interpretation as synapse weights that connect neurons (cf.~Section~\ref{sec:bdhgraph}). On the other hand, \BDHGPU's state $\state$ is a $n \times d$ matrix. To perform the analysis for \BDHGPU, in this section we recover $\corr$ from the relation:
\begin{equation}
\label{eq:corr}
\corr_{t-1,l} =
    \sum_{\tau < t}
        \ket{\xysparse_{\tau,l-1}}
        \bra{\xsparse_{\tau,l}}
        \rope^{t-\tau}
\end{equation}

\begin{figure}[ht]
    \centering
    \includegraphics[width=.75\linewidth]{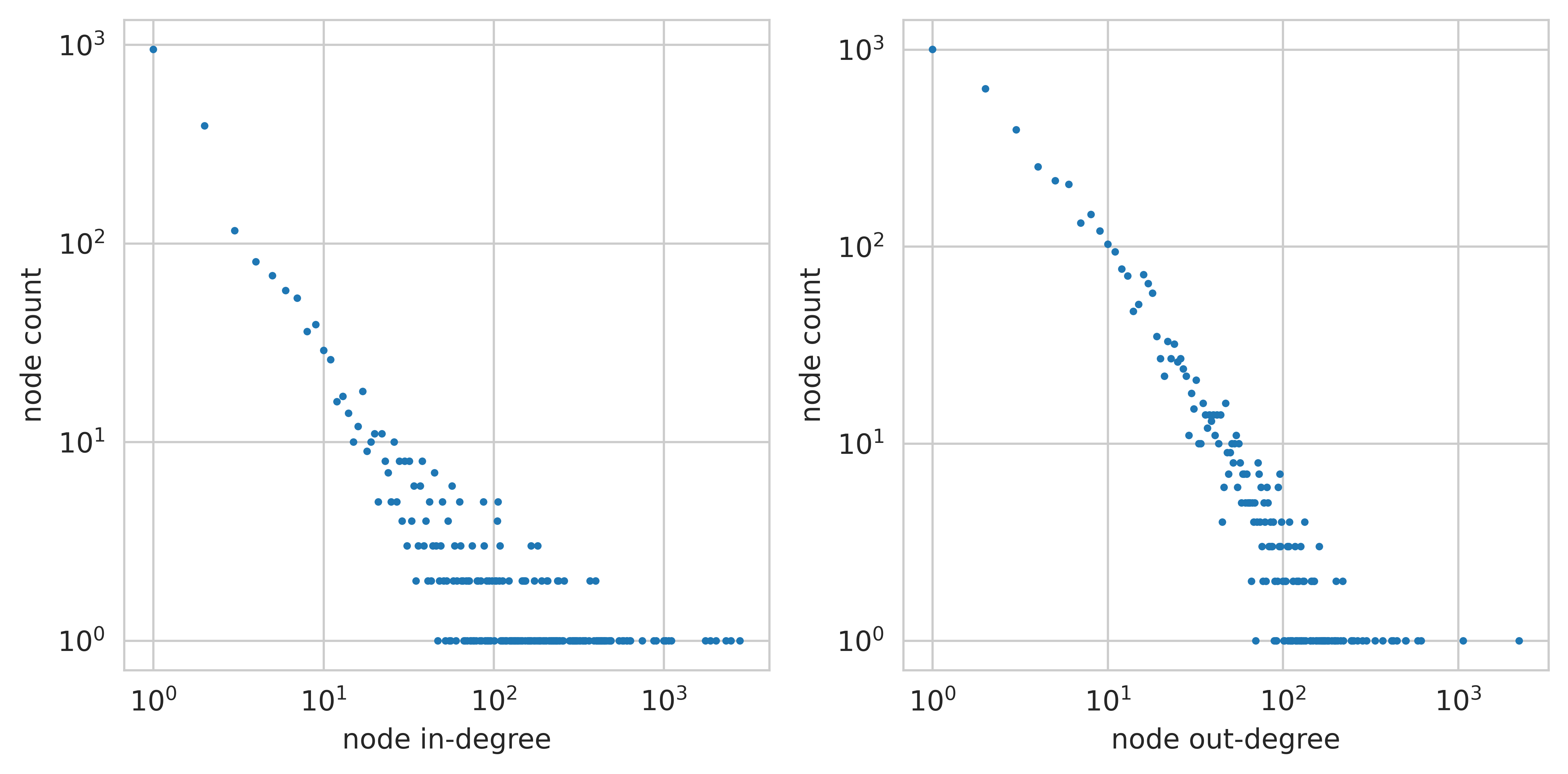}
    \caption{BDH's state $\corr$ encodes neuron connections as a scale-free graph showing clear heavy-tailed (power-law-like) degree distribution.}
    \label{fig:sigma_degrees_powerlaw}\label{fig:powerlaw1}
\end{figure}

We first analyze the neuron relationship graph encoded by matrix $\corr$. As explained in Section~\ref{sec:equations_of_reasoning}, $\corr$ can be interpreted as a graph of context dependent implications between $\xsparse$ and $\xysparse$. We compute the $\corr$ matrix for 0-th head at layer 5 of an 8-th layer network trained on Europarl translation corpus~\cite{koehn-2005-europarl} (we provide more details in Appendix \ref{sec:bdh_monosynapse_details}). We filter out negative entries which are introduced by the RoPE positional embeddings~\cite{su2023roformerenhancedtransformerrotary} and enforce a small positive threshold on remaining values to further sparsify the network structure. We plot the histograms of neuron in- and out-degrees, unraveling a scale-free network structure.

Encouraged by the emergent network structure, we have identified a few synapses that are activated at recognizable concepts, we show examples in the next section.

\subsection{Empirical findings: monosemantic synapses}\label{sec:monosynapse}

We have identified in the $\corr$ matrix entries (synapses) that show activity whenever a currency name or country name, both frequently occurring in the Euro-parliament transcripts, is present in the processed sentence. We have identified the synapses by searching for entries in $\corr$ that have predictive power at separating sentences containing a concept from contrast sentences. We present a few examples in Figure~\ref{fig:money_synapse}. We note that the synapses strength changes abruptly after words that are related to each concept. The same synapse is activated for concepts in both French and English sentences, even when the words used are different (e.g. ``livre sterling'' vs ``British Pound''). Synapse selectivity to a semantic context stems directly from sparsity of neuron activations as shown in~\figref{fig:x_y_synapse}.
\begin{figure}
\centering
    \includegraphics[width=\linewidth]{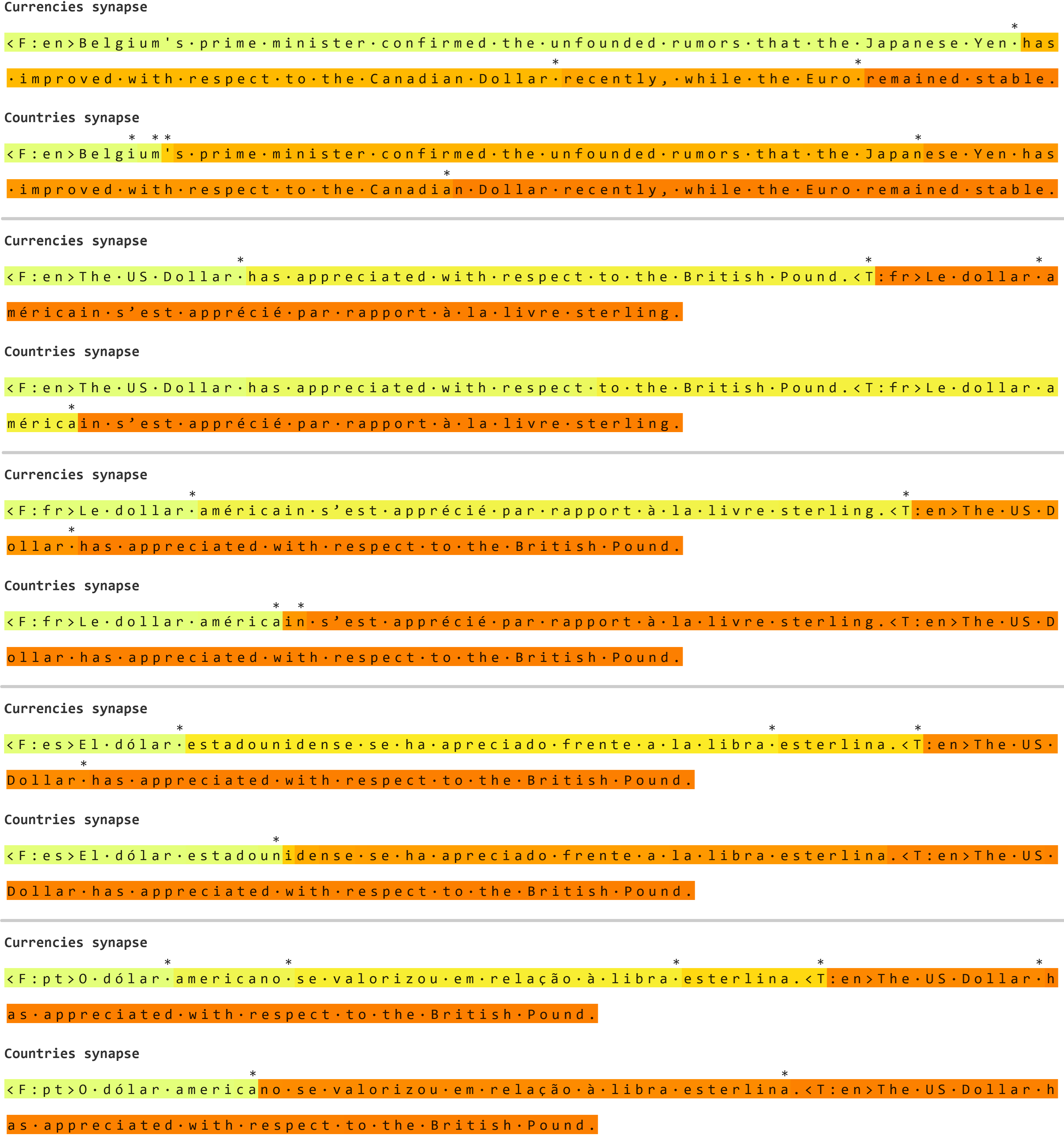}
    \caption{Evolution of values set by \BDHGPU on 2 specific synapses which we have named (following their interpretation) as ``currency synapse'' and ``country synapse'', relating to concepts naturally present in European Parliament transcripts on which the model was trained. We can notice that mentions of country or currency names result in an increase of the respective synapse value, indicating a stronger presence of the concept in the context. Moreover, the synapses consistently became activated in both French and English, confirming the (notice how it reacts both to ``British Pound'' and ``livre sterling'').\\
    For visual clarity, we indicate changes that clear a small threshold with the $*$ character (the changes in activity when the system is processing the translation of a source sentence tend to be small).}
    \label{fig:money_synapse}
\end{figure}
\begin{figure}
\centering
\includegraphics[width=0.85\textwidth]{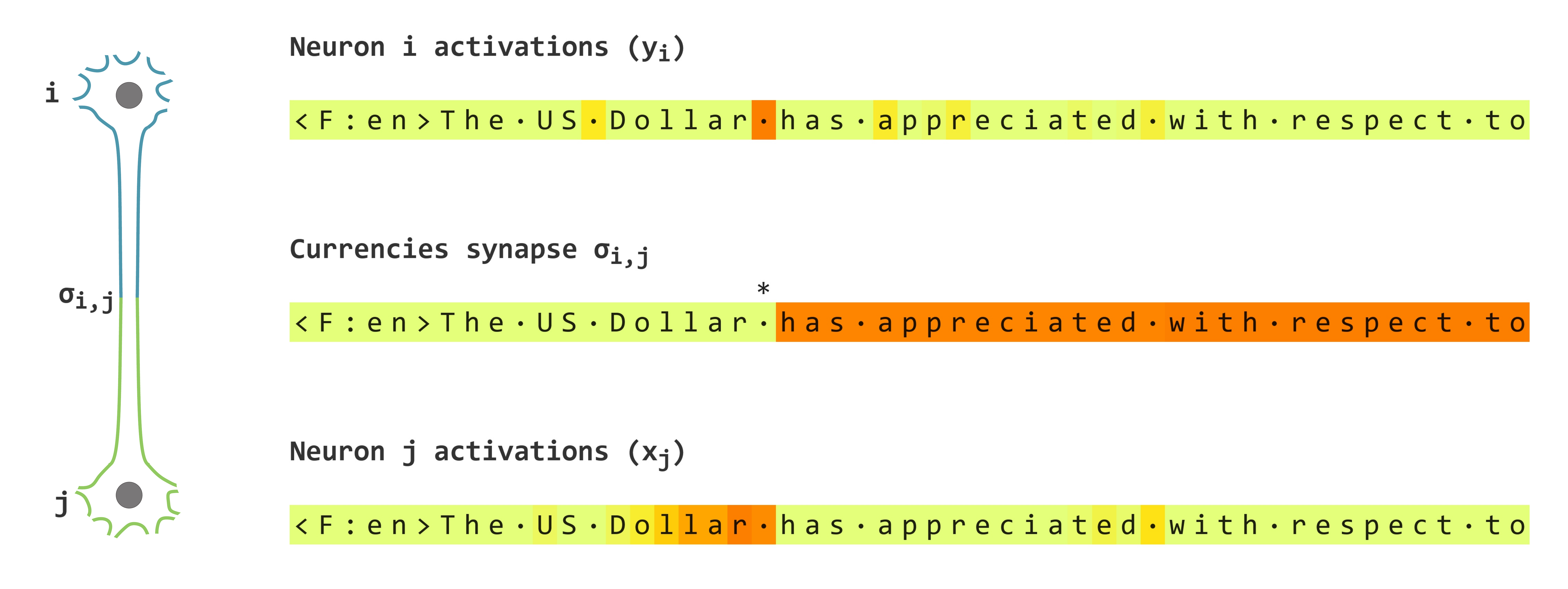}
\caption{Sparse updates to synapses related to meaningful concepts stem from sparse neuronal activations. BDH-GPU maintains in its recurrent state a ``currency synapse'' (a concept naturally present in the Europarl corpus, see also~\figref{fig:money_synapse}). The synapse is updated using a Hebbian learning rule when activity in $\xysparse$ activations at a preceding layer (4 in the example) leads to firing of neuron $\xsparse$ in the next layer (5).}
\label{fig:x_y_synapse}
\end{figure}

To confirm the selectivity of the synapses, we have generated, using ChatGPT, 50 sentences relating to European currencies, and another set of 50 sentences speaking about European politics, but not mentioning currencies. A one-sided Mann–Whitney U test revealed that sentences relating to currencies received significantly higher ``Currency synapse'' values than those without the currency concept ($U = 2368$ with $U_{\textrm{opt}}=2500$, $p < 10^{-14}$). The rank-biserial correlation was $0.86$, further confirming association between Currency concept presence and synapse value.

\subsection{Empirical findings: sparse neuron activations}

Sparsity of signals is often a prerequisite to their interpretability. In section \ref{sec:monosynapse} we have shown that BDH has monosemantic synapses, selectively activated by occurrences of specific concepts. In this section, we experimentally show that neuron activity correlates with signal predictability: fewer neurons are active, or equivalently, layer activations become sparser, for more predictable input signals.

We have trained a \BDHGPU model with $n=65536$ neurons, $d=256$, $L=4$ layers, and tokenization on letters of the Latin alphabet, to perform a single synthetic next-token prediction task. The input sequence started with a fixed $13$-letter warm-up sequence, followed by $8$ repetitions of an $8$-letter random word (``fact''), with the same pattern repeating every $13 + 8 \cdot 8 = 77$ letters. In~\figref{fig:sparsity}, we show neuron activity patterns. We can notice that neurons in higher layers are active during warm-up and fact introduction, then become quiet. We then group neurons by their RoPE frequencies and find that largest difference of activity during memorization and repetition is shown by the slow-acting neuron population.

\begin{figure}
\centering
(a)
\includegraphics[width=0.45\textwidth]{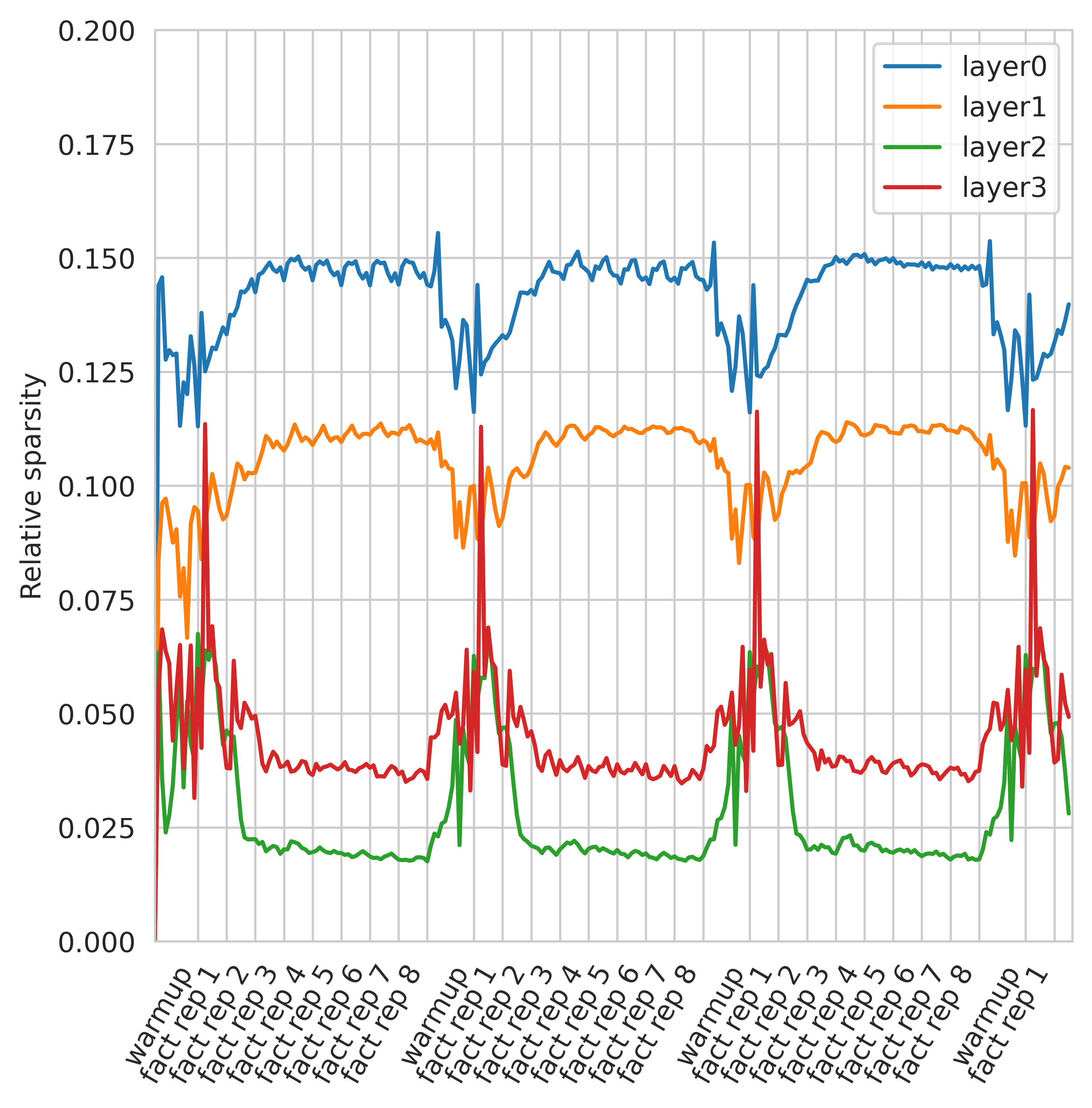}
(b)
\includegraphics[width=0.45\textwidth]{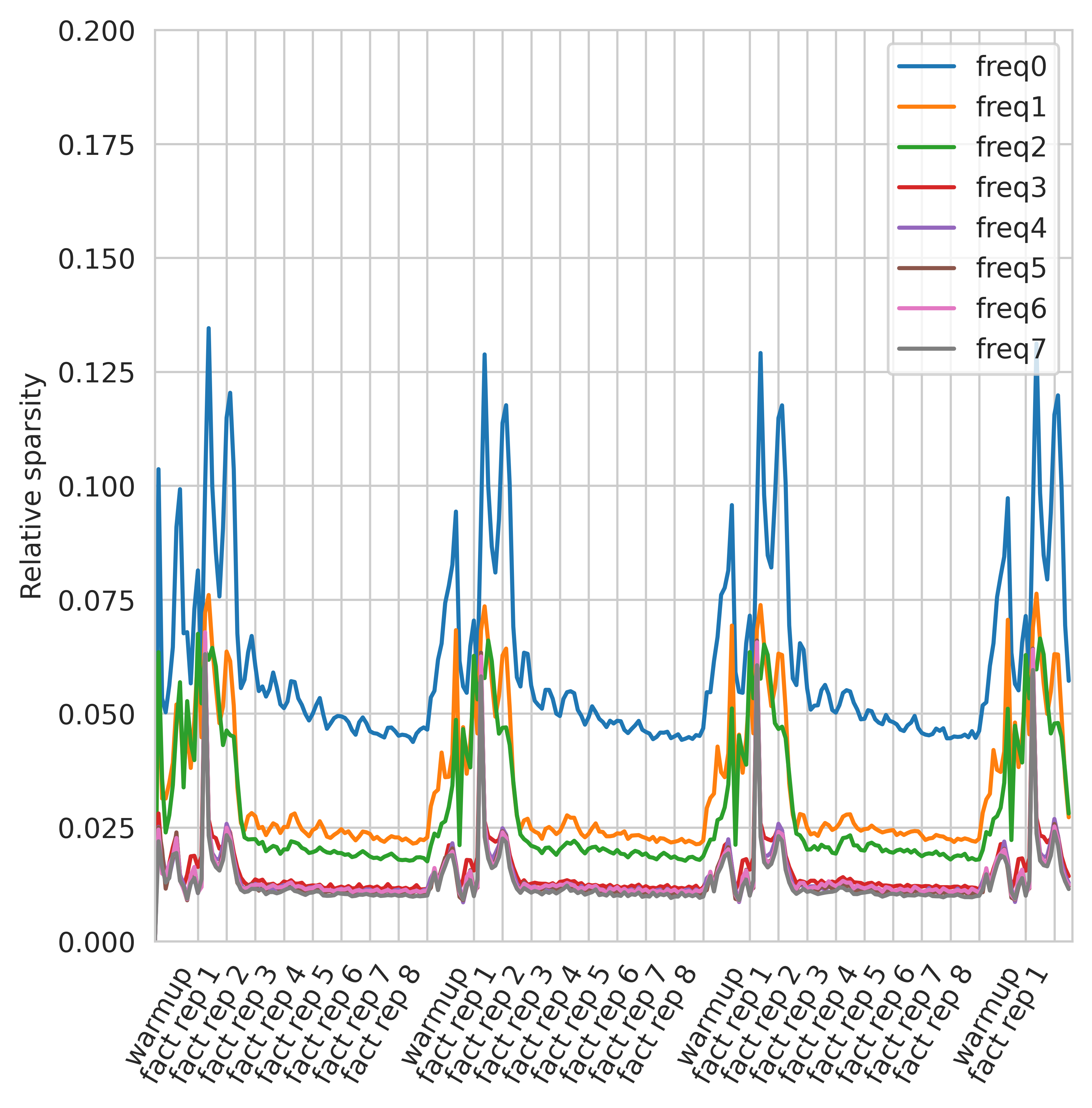}
\caption{Neurons in \BDHGPU are less active (signal is sparser) when the input is predictable. The input sequence started with a fixed $13$-letter warm-up sequence, followed by $8$ repetitions of an $8$-letter random word (``fact''), with the same pattern repeating every $13 + 8 \cdot 8 = 77$ letters. (a) Fraction of neurons with non-zero entry $\xysparse_{t,l}$ in different layers $l$, with fact memorization effect noted through increased activation level in layer $2$. The activation in layer $2$ has $4.0\%-7.5\%$ non-zero entries during memorization and approximately $2.5\%$ non-zero entries during repetition. (b) Detailed breakup of activation sparsity in layer $2$, with neurons bucketed into equal fractions by their RoPE phase: $\textrm{freq}0\in[1,4]$, $\textrm{freq}1\in[4,16]$, $\textrm{freq}2\in[16,64]$, $\ldots$, $\textrm{freq}7\in[16384,65536]$. The slow-acting half of the neuron population ($\textrm{freq}4-\textrm{freq}7$) exhibits the largest amplitude ratio between peak activation during memorization and repetition phases.}
\label{fig:sparsity}
\end{figure}

From a biological standpoint, sparse and surprisal-driven neuron activation lowers energy consumption --- despite fluctuations in low level percepts (in the experiment tokens are changing at every timestep), neurons in higher layers are inactive and do not expand energy. From a Deep Learning perspective, it has been recently shown that input complexity is related to predictability of internal representations of Transformers~\cite{herrmann2025measuringincontextcomputationcomplexity}. BDH makes this very explicit and does not require a separate prediction network: the predictable steady-state consists of zero activations, and input complexity entails neuronal activity. This suggests that BDH, natively, at a neuron level, implements mechanisms reminiscent of adaptive computation time~\cite{graves2017adaptivecomputationtimerecurrent} and conditional computation~\cite{cho2014exponentiallyincreasingcapacitytocomputationratio,shazeer2017outrageouslylargeneuralnetworks}, used in modern Transformers to lower computational effort during inference. 

Finally, sparse activation vectors in BDH imply that potentiation of specific synapses occurs rarely during inference. This is useful from the point of view of interpretability, noise reduction in Linear Attention state, and opens the door to simplified and compressed representations, notably for state and for gradient backpropagation DAG's.

\section{Playing with the Hatchling}\label{sec:experiments}

\subsection{Model merging: concatenating two models}\label{sec:experiments_merging}
Updating models with up-to-date knowledge and expanding models knowledge-base will become crucial in practical applications of AI. On possible solution is model composability, potentially allowing building of larger models by assembling a number of smaller, specialized models into a larger, more powerful one. A natural hope for such a system would be the achievement of ''more is different than a sum of its part'' effect. In the following experiment we are showing that doing so is relatively straight forward with \BDHGPU. This is because \BDHGPU can be scaled by varying only the number of neurons $n$. In this section we explore whether we can create larger models directly concatenating smaller models trained on disjoint subsets of data. Details in Appendix~\ref{sec:bdh_merge_details}. We have experimented with the following simple model merging procedure:
\begin{enumerate}
\item Train a base model on a chosen language pair. In the experiment we have used English-Spanish (En-Es) translation data, and have trained a model with $n=24576$ neurons (19M parameters).
\item Clone the base model and continue training on two datasets: English-French (En-Fr) and English-Portuguese (En-Pt).
\item We then merge the weights of the En-Fr and En-Pt models to create a new En-FrPt model with $n=24576\cdot 2=49152$ neurons (38M parameters).
    To create the merged model we:
        \begin{enumerate}
            \item concatenate all parameter tensors that have an `$n$' dimension (e.g. $\decodery$, $\decoderx$, $\encoder$, RoPE frequency buffers) along their $n$ dimension,
            \item average all other parameters (e.g. token embeddings and token prediction weights).
        \end{enumerate}
    To validate the hypothesis that direct model merging is feasible, we report all results on the merged model without any subsequent training or finetuning. However, we have verified that the merged model quickly improves when trained on all language pairs. 
\item After each stage we evaluate the models on all involved language pairs: En-Es, En-Fr, En-Pt, regardless of the data seen by the model up to this stage.
\end{enumerate}

\begin{table}
    \centering
\begin{tabular}{llll|lll}
\toprule
 & \multicolumn{3}{c}{\emph{Translation into English}} & \multicolumn{3}{c}{\emph{Translation from English}} \\
\emph{Model:}&  Es$\rightarrow$En & Fr$\rightarrow$En & Pt$\rightarrow$En & En$\rightarrow$Es & En$\rightarrow$Fr & En$\rightarrow$Pt \\
\midrule
1: Base En-Es & \underline{0.36} & 0.77 & 0.64 & \underline{0.35} & 2.21 & 2.27 \\
2: Base (1) tuned on En-Fr & \underline{0.58} & \underline{0.36} & 0.68 & 2.57 & \underline{0.31} & 2.54 \\
3: Base (1) tuned on En-Pt & \underline{0.44} & 0.76 & \underline{0.34} & 1.79 & 2.20 & \underline{0.33} \\
4: Merged (2$\|$3)%
& \underline{0.43} & \underline{0.40} & \underline{0.39} & 1.45 & 0.77 & 0.86 \\
\bottomrule
\end{tabular}
\\[5mm]
    \caption{Validation next token prediction losses (lower is better) of translation models trained on different language pairs and then merged. We evaluate each model on En-Es, En-Fr, En-Pt language pairs separately. We can see that the base model can translate between English and Spanish, while on En-Fr and En-Pt tasks it falls back on perplexities of an unconditional English language model (loss about $0.65$) and can't generate proper French or Portuguese. After tuning on French or Portuguese the model learns to translate between respectively English and French or English and Portuguese, while somewhat retaining the capacity to translate Spanish to English and losing the capability to translate English to Spanish. The merged model can translate Spanish, French, and Portuguese to English, however it mixes these three languages when asked to translate from English. This is consistent with qualitative results shown in Figure~\ref{fig:merged_translations}.}
    \label{tab:model_merging_esfrpt}
\end{table}

We report quantitative results in Table~\ref{tab:model_merging_esfrpt} and show qualitative results of merged model operation in Figure~\ref{fig:merged_translations}. The merged model shows human-like degradation of operation: while it retained the capability to generate and translate into English, it has lost the ability to generate proper text in Spanish, French, or Portuguese, mixing words and grammatical constructs. We have verified that a small amount of training on all language pairs restore the model's proficiency in Spanish, French and Portuguese. However, we decided to report on the behavior of the merged model without any subsequent tuning to highlight the possibilities of model engineering offered by the large and sparse working dimension $n$ of \BDHGPU.

\begin{figure}
\newcommand{\prompt}[1]{\textcolor{brown}{#1}}
\newcommand{\gen}[1]{\textcolor{blue}{#1}}
\newcommand{\separator}{\phantom\ }
{\scriptsize
\begin{Verbatim}[commandchars=\\\{\}]
\prompt{<F:es>Esta es una afirmación clara.<T:en>} | \gen{In this clarification, it is a clear statement.}
\prompt{<F:es>Esta es una afirmación clara.<T:en>} | \gen{It is a clear statement.}
\prompt{<F:es>Esta es una afirmación clara.<T:en>} | \gen{That is a clear affirmation.}
\separator
\prompt{<F:fr>C’est une déclaration claire.<T:en>} | \gen{This is a clear statement.}
\prompt{<F:fr>C’est une déclaration claire.<T:en>} | \gen{This is a clear declaration.}
\prompt{<F:fr>C’est une déclaration claire.<T:en>} | \gen{It is a clear declaration.}
\separator
\prompt{<F:pt>Esta é uma afirmação clara.<T:en>} | \gen{That is a clear statement.}
\prompt{<F:pt>Esta é uma afirmação clara.<T:en>} | \gen{This is a clear statement.}
\prompt{<F:pt>Esta é uma afirmação clara.<T:en>} | \gen{This is a clear assertion.}
\end{Verbatim}
    (a) Sampled translations from Spanish, French, and Portuguese to English. The vertical bar $|$ delimits the prompt from model output. Each translation was sampled three times to show the model's consistency.\\
\begin{Verbatim}[commandchars=\\\{\}]
\prompt{<F:en>This is a clear statement.<T:es>} | \gen{Ce récent statement está clarificative.}
\prompt{<F:en>This is a clear statement.<T:es>} | \gen{Il se revela exact.}
\prompt{<F:en>This is a clear statement.<T:es>} | \gen{Constato que cette déclaration était monstruosa.}
\separator
\prompt{<F:en>This is a clear statement.<T:fr>} | \gen{Está de noto uma déclaration.}
\prompt{<F:en>This is a clear statement.<T:fr>} | \gen{Ce sont une declaration clare.}
\prompt{<F:en>This is a clear statement.<T:fr>} | \gen{C'está uma declaração clare.}
\separator
\prompt{<F:en>This is a clear statement.<T:pt>} | \gen{Esta declaração étangling état está clara.}
\prompt{<F:en>This is a clear statement.<T:pt>} | \gen{Istambigna de contence.}
\prompt{<F:en>This is a clear statement.<T:pt>} | \gen{C'estaté clarification é clara!.}
\end{Verbatim}
    (b) Sampled translations from English into Spanish, French, and Portuguese. The model mixes the three languages, though the meaning of the source sentence seems to have been preserved.\\
\begin{Verbatim}[commandchars=\\\{\}]
\prompt{<F:en>} | \gen{The European Convention on Human Rights has been set up in 1992, when it applied the Convention in 1}
\prompt{<F:es>} | \gen{Naturalment, nos deputés de toda autre groupe southern Italians, notariously engaged in the discuss}
\prompt{<F:fr>} | \gen{(ES) Mr President, aproveit de montrar ma satisfaction por a surprise de la parte de Milan, et parti}
\prompt{<F:pt>} | \gen{(FI) Mr President, todos outreaches, mesures on ways and means of promoting the economic development}
\end{Verbatim}
    (c) Language-conditional samples. Consistently with the translation experiment in (a) and (b) above, the model properly generates English, and mixes Spanish, French, and Portuguese, sometimes slipping into English.\\
\begin{Verbatim}[commandchars=\\\{\}]
\prompt{<F:es>Mi lingua project nominat Romanid esed publicat ja in may de pasat an.<T:en>} | 
    \gen{A Lingue nominated Project is Romanian published in Romanian pasta year.}
\prompt{<F:fr>Mi lingua project nominat Romanid esed publicat ja in may de pasat an.<T:en>} |
    \gen{My language in Project Nomina is Romani esedi publish posted peas in maybe the same year.}
\prompt{<F:pt>Mi lingua project nominat Romanid esed publicat ja in may de pasat an.<T:en>} |
    \gen{I have been a Latva project nomination Romanim esede publica published in May of the postal service and the service provider.}
\end{Verbatim}
    d) Attempts of the model to translate a sentence in Romanid~(\citeyear{enwiki:romanid}), a zonal auxiliary language for speakers of Romance languages naturalistic constructed language intended to be intuitively understandable to speakers of Romance languages. The English translation of the prompt is: ``My language project called Romanid was published already in May of last year''. The model is able to pick up some of the meaning of the prompt.

}
    \caption{Conditional and unconditional samples generated from a English-Spanish-French-Portuguese translation model created by direct concatenation of parameters of models trained on distinct language pairs.}
    \label{fig:merged_translations}
\end{figure}

The \BDHGPU model merging experiment has shown that when the model latent space promotes concept disentangling (c.f. Section~\ref{sec:micro_attention_bdh_gpu} on monosemanticity) then it is feasible to directly compose concepts in this space, e.g. by concatenation of weights from different models. This feature of the BDH architecture allows us to see the models as composable computer programs with emergent properties.

\subsection{Training without backpropagation through time}

Sparsity of synapse activations in BDH opens the door to efficient approximations to backpropagation through time. The main intuition is that we only need to remember when a synapse has changed, and the $i,j$ coordinates of the synapse implicitly encode which neurons were active and should take part in error signal backpropagation.

In this section we report results of preliminary experiments on the impact of removal of backpropagation through time on model performance. %
For the PyTorch implementation in Appendix~\ref{sec:bdh_code_listing}, this corresponds to `detach'-ing variables \texttt{K} and \texttt{V} in the implementation of the \textsf{LinearAttention} class. 

In particular, we found that such a model, trained without any backpropagation through time, retained some ability to model language, but lost the ability to match concepts between different languages during translation. For translation tasks like those presented in Table~\ref{tab:model_merging_esfrpt}, loss values for English increased from a loss level of approximately $0.65$ for an unconditional English language model (trained with backpropagation over time), to loss of approximately $0.75-1.05$ for a model trained without backpropagation over time, depending on model variant, regardless of whether English was the source or target language in translation. No significant difficulties were encountered during training when crossing the barrier of the letter-bigram language model, at loss value $2.4$.

Beyond side-effects of the general design, we did not optimize the \BDHGPU model for suitability of training without backpropagation. We consider this architecture to be a good starting point for bootstrapping further investigations in this direction. 

\section{Conclusions}

\subsection{Takeaways for model engineering}

This paper leads up to a new class of language and reasoning models which eliminate architecture nonuniformities, notably in terms of scaling for model size, and handling of time scales of inference.

The \BDHGPU architecture introduced in this paper opens the following opportunities:
\begin{enumerate}
  \item \emph{New ways of scaling models for time and size.} \BDHGPU is a state-space model which scales for size in one large dimension $n$ (neurons in this dimension are indexed by RoPE oscillator frequency). Subject to appropriate sharding, this also leads to a desirable form of locality: important data is located just next to the sites at which it is being processed. This minimizes communication, and eliminates the most painful of all bottlenecks for reasoning models during inference: memory-to-core bandwidth.
  
  \item \emph{Faster model iteration.} During training and inference alike, \BDHGPU provides insight into parameter and state spaces of the model which allows for easy and direct evaluation of model health and performance, notably, through sparsity-related measures and through aggregates and statistics on the large pool of homogeneous neurons, even for relatively small models. Attention and parametric layers alike operate on the same neuron dimension (`concept dimension').
  \item \emph{Direct explainability of model state.} Elements of state of \BDHGPU are directly localized at neuron pairs, allowing for a micro-interpretation of the hidden state of the model.
  \item \emph{New opportunities for `model surgery'.} The \BDHGPU architecture is, in principle, amenable to direct composability of model weights in a way resemblant of composability of programs. This concerns the potential both the direct composition of separately trained model parts, as well as `surgery' of parameter spaces of models, by inserting fragments of manually programmed protocols into machine-learned code.
\end{enumerate}

\subsection{Implications for brain science}

We have obtained a micro-foundational description of attention for artificial language and reasoning models, expressed in a framework of local graph dynamics. This has been found to be consistent with the effects observed for the same function of \emph{attention for language and reasoning} in the brain. By introducing a translation layer based on similarity of function between the artificial and biological planes, for blocks of feed-forward neural networks and attention mechanisms, our work points to the following hypothesis: \emph{complex systems effects which are observed in the brain, around modular scale-free network structure, synaptic plasticity, and Hebbian learning arose from its core purpose --- doing reasoning --- and \textbf{not} from any specific longer-term training dynamics which the brain applies}.

We have exhibited how a general attention mechanism can be efficiently implemented as an artificial neuronal system with spiking neurons and synapse plasticity. More formally, we first describe the class of local interaction dynamics which any system \emph{plausibly needs} to implement attention mechanisms. We then confirmed that the edge-reweighting rule is \emph{sufficient} to allow a certain artificial Language Model (\BDHGPU) to operate at least at the level of the Transformer. For an artificial network, the edge-reweighting rule intuitively describes the interaction between two artificial neurons exhibiting rapid state-change behavior, and one synaptic neuron interconnection element exhibiting plasticity as shown in Fig.~\ref{fig:x_y_synapse}. %

More broadly, this work may potentially serve to support efforts aiming to isolate, from among the many extremely complex electrochemical patterns and signal dynamics occurring in the brain, those that are crucial for \emph{solving tasks in-context (based on attention)}, from those that potentially serve other purposes, such as transfer of information from short-term memory to long-term memory, or long-term improvement of brain function (learning).

\paragraph{How this work helps with axiomatization of learning theory in the brain.} %
Attempts to understand the brain, starting from the perspective of longer time scales of training, have proved extremely challenging, defying progress. This paper pin-points attention-based reasoning at shorter time scales as `the other end of the string', and hints how, from here, untangling the entire story will plausibly be easier.

For natural systems undergoing continuous learning, the time scales to look at are: language function and reasoning (chain-of-thought inference), then short-to-long memory transfer from state to network weights, adaptation of structure: changes to interconnections, and finally, changes to neuron nodes.

For long time scales, this reduces the question of finding supervised training dynamics form the most general case, to a specific class of local dynamics: an interaction kernel performing `edge-reweighting' rules. %
As these rules appear fundamental to logical inference and biochemical processes alike, its universality in processes that the brain is responsible for is plausible also beyond the realm of language-based reasoning.

\emph{From a systems perspective, we arrive at the following possible explanation.} The brain generally tries to be lazy in terms of energy expense, and does things as late as it can. Only reasoning needs to happen close to a critical regime, because it involves executing a real-time program which needs to be responsive, since the life and success of the biological organism depends on it. Then, for a certain time, which may be minutes for humans, the brain has enough synapses in it to represent (almost) all useful information it needs for reasoning, decision-making, etc. --- all stored in short-term state, at synapses (and/or neurons). Some of the neuron activations which the brain performs at this time scale represent `gradients of state' --- the gradients of in-context learning, passed on to modify synapse strength, in a weight-update process. As time goes by, the system runs out of state space. Then, memory processes work to iron things out, preserving in more permanent neuron connection weights and graph structure the elements of state that have been reinforced by feedback signals. Overall, there are fewer and fewer things that need to be remembered across progressively longer time scales. However, this entire memory process is, plausibly, subsidiary to the definition of the dynamics of reasoning and the synaptic dynamics of state that we discuss in this paper. In other words, the best form of description of the relaxation from state into longer-term memory follows from the specific kernel of the reasoning dynamics, such as the edge-reweighting kernel.

As for the ratio of time scales (measured in tokens for language), we can estimate that the time lapse after which harmonizing state with a memory process becomes important is of about the same order of magnitude as the average time between `writes' (significant transmission increases) for individual synaptic elements (see e.g.\ \figref{fig:sparsity}). In our models, this time is lower-bounded by the inverse of sparsity of the vector $\xysparse$, i.e., $1/\rho \approx 1/5\% = 20$ tokens, but it could be much larger for larger systems; we also do not force it in any way to be sparser during training. During training with backpropagation, if the backpropagation window $T$ is short enough, $T < 1/\rho$ tokens, we can plausibly assume that a synapse changes state only once in that window (and is used multiple times), hence the DAG of gradient backwards propagation is much more direct to embed within the system graph. Backpropagation is then a question of `routing' gradients in the neuron communication graph, and not one of disentangling them. All natural training approaches, whether based on backpropagation, or any more direct form of relaxation `from state into weights', appear to bottleneck on the amount of available state space on synapses, becoming necessary at about $T \sim 1/\rho$ by a simple information-theoretic argument on state storage capacity.

Regardless of how much of this is an accurate description, and how much an intuition, at the very least, it appears we may now have a way forward. Some part of the ``global mystery'' of learning in the brain can be reduced to a more ``localized problem'' of state-to-operator transfer for some relatively compact form of state-space dynamics (i.e., one specific local graph kernel). This change of perspective brings in both a completely new `problem landscape' in which to navigate towards a complete solution, as well as a set of new methods to use for the different types of graph structure changes involved in learning, including approaches from distributed computing, evolving network theory, and graph rewiring systems.

At this point, it seems one natural next step would be to ground the current discussion more deeply in findings of brain science, to refine or simplify the \emph{actual kernels} used by brain reasoning (which was not the objective of this paper), and potentially seek validation through experiment.

\subsection{Societal impact}

This paper is a voice in favor of bringing principled understanding to reasoning in Machine Learning. Axiomatic AI provides an opportunity to reduce risks related to unpredictable behavior of AI models, and, to open or accelerate new development directions. The subject matter which we consider here serves as a direct introduction to the most crucial problem that lies ahead: controlling the behavior of autonomous AI reasoning models and AI systems as they progress across time scales, from seconds to years.

\begin{sidewaystable}
\footnotesize
\setlength{\extrarowheight}{6pt}
\begin{tabularx}{\textwidth}{YYYYY}
\toprule
& \textbf{Transformer (GPT2)}
& \textbf{\BDHGPU($n$,$d$)}
& \textbf{BDH($n$,$\Delta$)}
& \textbf{Brain models (reasoning and language function)}
\\
\midrule
\emph{Inference hardware}
& GPU, CPU
& GPU, CPU
& CPU, Sparse GPU kernels, Neuromorphic
& Brain and supporting systems
\\
\emph{Model weights (predominant location)}
& $5$ tensors per layer (different shapes)
& $3$ tensors per model (same shape $n \times d$)
& Neuron-synapse graph: connection topology, edge weights
& Neuron-synapse graph: connection topology, edge weights
\\
\emph{Representation of attention}
& KV-cache tensor (not localized at neurons)
& $n \times d$ tensor for each layer (localized at neurons)
& Memory on synapse edge weights
& State memory through synapse plasticity
\\[5mm]
\emph{Macro-description of attention}
& Key lookup data-structure, key-value map
& Key lookup data-structure, key-value correlation matrix
& Key lookup data-structure, key-value correlation matrix
& Not known
\\
\emph{Micro-description of attention}
& None
& Neuron-pair correlations in context (transformed)
& Neuron-pair correlations in context
& Strengthened or weakened connections between neurons based on context
\\
\emph{Scaling for model size}
& Multiple combinations of dimensions, e.g. MLP scales with $D \times D$, scales separately with context length
& Uniform linear array of $n$ particles in a mean-field
& $n$-node graph model
& $n$-node graph model with evolving graph mechanisms
\\
\emph{Distributed system micro-architecture}
& Follows from compiled matrix multiplication kernels, non-uniform
& Particles run identical local kernels, communicating $O(d)$-size messages through mean-field, and storing local state
& All $n$ neuron nodes run identical local kernels, communicating over neuron-synapse graphs. Some synapses act as memory elements.
& $n$ neurons run local kernels and communicate through a network, using numerous signal patterns and coding schemes. Synapses act as memory element.
\\
\emph{Macro-expressiveness of programs (approximation)}
& RASP-L, C-RASP
& RASP-L, C-RASP
& RASP-L, C-RASP (or superset)
& Unknown
\\
\emph{Micro-expressiveness of programs}
& Unknown
& Subset of BDH
& Probabilistic rule-based local protocols. Micro-Inductive bias interpreted as reasoning system in a form of propositional logic.
& Unknown
\\
\emph{Emergence of structure}
& Partially interpretable concept layer (evidence of monosemantic neurons for important concepts)
& Evidence of emergent network, oscillator dynamics
& Emergent network, oscillator dynamics
& Emergent network; oscillatory effects; possible monosemantic ``grandfather neurons''
\\
\emph{Activation vectors}
& Dense activation; can be subsampled or sparsified by architecture modification
& Positive vectors, sparse fresh activation vectors $y$
& Positive vectors, sparse fresh activation vectors $y$
& Sparse positive activation vectors
\\
\bottomrule
\end{tabularx}
\caption{Comparison of properties of language and reasoning model architectures: the GPT2 Transformer, \BDHGPU, BDH, and brain models.}
\label{tab:comparison}
\end{sidewaystable}

\newpage

\subsection*{Acknowledgments}
The authors thank David Sussillo, Navdeep Jaitly, and Emanuele Natale for insightful discussions on reasoning and the brain, and for early feedback on this write-up. 
We also thank Samy Bengio for comments on the presentation.
We kindly acknowledge the support of all of the Pathway team, notably, Pawe\l{} Podhajski for his amazing help with cluster setup, Victor Szczerba and Z Schwab for all discussions over coffee, and Kamil Piechowiak and Chris Ociepa for constructive comments on the presentation. AK thanks Christos Papadimitriou for being the direct inspiration for us to embark on this journey.

\subsection*{Author contributions}

AK conceived the BDH and \BDHGPU architectures, conceived most of the theory, developed most of the model source code, conceived and performed experiments on synapses, and wrote most of the paper.

\noindent
PU contributed crucial elements of \BDHGPU architecture, contributed model and framework source code, contributed to theoretical analysis, and performed experiments.

\noindent
JCh led, designed, and oversaw methodology of experiments, led framework development, contributed major improvements to \BDHGPU architecture, contributed to the theory, implemented baselines, performed experiments, and substantially redacted the paper.

\noindent
ZS conceived the project, guided research directions, introduced particle-interaction interpretation, acted as final judge in research decisions, and substantially redacted the paper.

\noindent
MB optimized model source code, contributed framework source code, and performed experiments.

\newpage

\bibliography{paper}{}

\newpage
\appendix
\addtocontents{toc}{\protect\setcounter{tocdepth}{1}}

\section{Connection between generalization of reasoning and computational expressiveness}
\label{appx:one}
State-of-the-art reasoning models have the interpretation of (Turing-complete) programs, executed over a certain period of time. This shifts the emphasis of generalization, from discovering the structure of mathematical functions which maps inputs to outputs, to discovering a class of runnable programs, which take as input a given class of input prompts, and process these prompts ``in the right direction''.

Consider a given reasoning task, whose scope is defined as a set  $\P$ of valid input prompts, given as bounded-length token sequences over some alphabet $\Omega$. Given a prompt from $\P$, a model solving the considered task is eventually (i.e, after some number of steps of reasoning) expected to generate an output, in the form of a bounded-length token sequence over the same alphabet $\Omega$,  which is subjected to evaluation. Consider language models sampled from some probability distribution $\M_1$ over parameter sets in some architecture $\A_1$. 

Now, suppose that for some other model architecture $\A_2$ there exists a distribution $\M_2$ over language models in $\A_2$ such that, for a valid input prompt chosen uniformly at random from $\P$, the outputs sampled from a model $M_1 \sim \M_1$ and the outputs sampled from a model $M_2 \sim \M_2$, have (almost) the same distribution in the space of bounded-length sequences over $\Omega$, and are both obtained within some asymptotic bound on the number of steps of reasoning, in expectation. The described setting is equivalent to saying that \emph{models $\M_2$ have generalized the considered task $\P$ in (almost) the same way as models $\M_1$}. Indeed, conversely, if the described condition did not hold, we could, in a finite number of trials, distinguish solutions to problem $\P$ obtained by model families $\M_1$ and $\M_2$.

Now, consider model architectures $\A_1, \A_2$ which apply Chain-of-Thought reasoning~\cite{DBLP:journals/corr/abs-2201-11903}. A model in such an architecture has the interpretation of a trainable probabilistic program, taking inputs from $\P$, and the architectures themselves represent computational machine architectures. Moving to a discussion of computational expressiveness, we obtain the following statement. 

\begin{observation}
Given a probability distribution of models $\M_1$ in architecture $\A_1$, suppose there exists a distribution over models in architecture $\A_2$ which generalizes on task $\P$ in the same way as models from $\M_1$. Then, the machine architecture $\A_2$ has sufficient computational expressiveness to simulate programs from $\M_1$ efficiently on the set of inputs $\P$, i.e., $\A_2$ contains programs which obtain an (almost) identical distribution of outputs within the given bounds on running time.
\qed
\end{observation}

In particular, we note that if we were to consider the special case of $\A_1$ being reasonable human agents, we could say that architecture $\A_2$ generalizes reasoning, in the same way as humans, if we can train models $\M_2$ in $\A_2$ which accurately reproduce the outcomes of reasoning for some sample $\M_1$ of humans in $\A_1$.

This leads us naturally to describe Language Model generalization through a universal reference to the principles of operation of the human brain, treated as a distributed computing architecture, and not through a characterization of language and reasoning prompts $\P$ that the model should be able to deal with in some specific way. 

\section{Further description of experiments}

\subsection{Language translation task}\label{sec:translation_task}

We have evaluated our models on a mixed language modeling and translation task derived from the Europarl corpus \cite{koehn-2005-europarl}. The corpus consists of sentence-level aligned translations of transcripts of European Parliament proceedings. For each language pair, we treat the data as a long stream of interleaved source and target sentences (sampling for each sentence which language is the source, and which is the target) on which we train decoder only models. Thus, models are jointly trained as language models and translators. We train all models using Truncated Backpropagation Through Time~\cite{williams1990efficient}. Subsequent minibatches served by the data loader are related: each is a continuation of the previous. Each model maintains a recurrent state, carried across minibatches: $\state$ matrix for \BDHGPU and a FIFO buffer of recent KV-cache entries for the TransformerXL~\cite{dai2019transformerxlattentivelanguagemodels} baseline. We train all models on raw UTF8 data. We are mainly interested in model comparison and prefer to keep the experimental setup as simple as possible. A few minibatches are shown in Fig.~\ref{fig:euparldata}.

The joint language modeling and translation formulation has several benefits:
\begin{enumerate}
\item Next token prediction is representative for LLM training. Simple architectures, such as decoder-only models are sufficient.
\item The task promotes models with long context capabilities --- subsequent sentences are related and the model can meaningfully
  utilize long context to model the source language sentences.
\item The task promotes models which carry state across minibatches, as training data is temporally coherent and the final model state at the end of one minibatch is a natural initialization of hidden state on the next minibatch.
\item Translation can be seen as language modeling coupled with fuzzy copying. Successful models will need to develop in-context learning capabilities such as inductive heads~\cite{olsson2022context}.
\end{enumerate}

\begin{figure}[ht]
  \centering
    \begin{verbatim}
        0. |<F:en>For countries such as Sweden and Finland, another system o|
        1. |f allocation would be extremely significant.<T:es>Por ejemplo, p|
        2. |ara pa•íses como Suecia y Finlandia tendr•ía un gran significado|
        3. | que se hiciese otra forma de distribuci•ón.<F:es>El diputado Fe|
        4. |rber ha presentado una propuesta que implica una distribuci•ón m|
        5. |•ás flexible, y yo respaldo esta enmienda.<T:en>Mr Ferber has ta|
        6. |bled an amendment which involves our looking in a considerably m|
        7. |ore flexible way at the present allocation, and I support this a|
        8. |mendment.<F:en>.<T:es>.<F:en>(NL) Mr President, I would like to |
        9. |start by thanking both parliamentary committees and not least bo|
            \end{verbatim}
  \caption{Exemplary sequence of 10 successive minibatches from the translation task. The model is trained on raw UTF8 bytes (for visualization we pad multi-byte UTF8 characters with ``•'' symbol). Special token strings $\texttt{<F:lang\_code>}$ and $\texttt{<T:lang\_code>}$ delimit source and target sentences. Minibatches are temporally coherent: source sentences are followed by their translations, and subsequent source sentences are part of  the same larger document.}
  \label{fig:euparldata}
\end{figure}

\subsection{BDH Scaling Experimental Details}\label{sec:bdh_scaling_details}
We provide details on models used in scaling experiments described in Section~\ref{sec:comparison_bdh_gpt_transformers}. All models were implemented in PyTorch~\cite{paszke2019pytorch} and  trained on the Europarl~\cite{koehn-2005-europarl} task described in Section~\ref{sec:translation_task}. We have kept the same training regime for all models at all sizes: En-PL and En-Cs language pairs (380MB total). All models trained on raw UTF8 bytes seeing a total of 1.2B tokens (about 3 epochs). All minibatches were 2048 tokens long, but we have varied the number of examples in the minibatch (varying number of tokens in each minibatch) to accommodate different memory requirements of different models. We have used multi-GPU training using the Distributed Data Parallel approach using AdamW~\cite{loshchilov2017decoupled} with learning rate $10^{-3}$, and 1000 warm-up step followed by linear learning rate decay over the course of training to $10^{-4}$, adaptive gradient clipping~\cite{kumar2025zclipadaptivespikemitigation}, and weight decay $0.1$. Models were trained to operate on a context longer than minibatch length using Truncated Backpropagation Through time~\cite{williams1990efficient}. 

The Baseline model, dubbed GPTXL, was a GPT2-like transformer~\cite{radford2019language} based off the NanoGPT~\cite{karpathy2024nanoGPT} implementation with KV-cache carried across minibatches as in TransformerXL~\cite{dai2019transformerxlattentivelanguagemodels}. We have used ALiBi positional biases~\citet{press2022trainshorttestlong}. We list its hyperparameters for various model sizes in Table~\ref{tab:transformerxl_hyperparams}. Optimal Dropout was selected using a small sweep at each model size.

\begin{table}
    \centering
    \begin{tabular}{lll lll l}
    \toprule
     model & num   & embd  & num  &  MLP  & dropout & Carried KV-cache  \\
     size  & layer & dim   & head &  dim  &         & size \\
    \midrule
     25M   & 9     & 480   & 5    &  1920 & 0.01    & 4096 \\
     50M   & 12    & 576   & 6    &  2304 & 0.02    & 4096 \\
     100M  & 15    & 768   & 8    &  3072 & 0.02    & 4096 \\
     200M  & 18    & 960   & 10   &  3840 & 0.002   & 4096 \\
     400M  & 25    & 1152  & 12   &  4608 & 0.005   & 4096 \\
     800M  & 28    & 1536  & 16   &  6144 & 0.15    & 4096 \\
    \bottomrule
    \end{tabular}\\[5mm]
    \caption{Hyperparameters for GPTXL baselines in scaling experiments. The model architecture follows GPT2~\cite{radford2019language}, with a FIFO buffer of past KV-cache entries~\cite{dai2019transformerxlattentivelanguagemodels}.}
    \label{tab:transformerxl_hyperparams}
\end{table}

\BDHGPU directly uses model code provided in Appendix~\ref{sec:bdh_code_listing}. \BDHGPU' adds xLSTM-like gating mechanism~\cite{beck2024xlstm}, and merges next token predictions from all layers. Both \BDHGPU and \BDHGPU' use same architectural hyperparameters, gathered in Table~\ref{tab:bdh_scaling_hyperparams}.

\begin{table}
    \centering
    \begin{tabular}{lll lll}
    \toprule
     model & num   & $d$   & $n$  &  num   & dropout \\
     size  & layer &       &      &  head &         \\
    \midrule
     25M   & 8     & 256   & 32768  &  4     & 0.1     \\
     50M   & 8     & 256   & 65536  &  4     & 0.1     \\
     100M  & 8     & 256   & 131072 &  4     & 0.1     \\
     200M  & 8     & 256   & 262144 &  4     & 0.1     \\
     400M  & 8     & 256   & 524288 &  4     & 0.1     \\
     800M  & 8     & 256   & 1048576  &  4     & 0.1     \\
    \bottomrule
    \end{tabular}\\[5mm]
    \caption{Hyperparameters for \BDHGPU models in scaling experiments.}
    \label{tab:bdh_scaling_hyperparams}
\end{table}

\subsection{BDH Monosemantic Synapse Experiment Details}\label{sec:bdh_monosynapse_details}
We provide details for models used in exploration of monosemantic synapses in Section~\ref{sec:micro_attention_bdh_gpu}. The model was trained on Europarl~\cite{koehn-2005-europarl} described in Section~\ref{sec:translation_task}. It had $d=256, n=49152$, $4$ attention heads, and $8$ layers. The model was trained on about one epoch of En-Es, En-Pt, and En-Fr data (total 1.9B tokens) in a Distributed Data Parallel setup using AdamW~\cite{loshchilov2017decoupled} with learning rate $10^{-3}$, 1000 warm-up step followed by linear learning rate decay over the course of training to $10^{-4}$, adaptive gradient clipping~\cite{kumar2025zclipadaptivespikemitigation}, and weight decay $0.1$. We have used Truncated Backpropagation Through time, carrying over the recurrent state of attention and training on sequences of length $2048$ characters at a time. We have used minimal Dropout~\cite{srivastava2014dropout} of $0.01$.

\subsection{BDH Merging Experiment Details}\label{sec:bdh_merge_details}
We provide details for models described in Section~\ref{sec:experiments_merging}
All models were trained on Europarl~\cite{koehn-2005-europarl} described in Section~\ref{sec:translation_task}. We provide model architecture hyperparametrs in Table~\ref{tab:merging_model_arch_hyperparams}. Models were trained on about two passes over the training set in a Distributed Data Parallel setup using AdamW~\cite{loshchilov2017decoupled} with learning rate $10^{-3}$, 1000 warmup step followed by linear learning rate decay over the course of training to $10^{-4}$, adaptive gradient clipping~\cite{kumar2025zclipadaptivespikemitigation}, and weight decay $0.1$. We have used Truncated Backpropagation Through time, carrying over the recurrent state of attention and training on sequences of length $2048$ characters at a time. We have used minimal Dropout~\cite{srivastava2014dropout} of $0.01$.

\begin{table}
    \centering
    \begin{scriptsize}
    \begin{tabular}{lll lll lll l}
    \toprule
    Model          & Init.                 & Training & Data size & Training & n     & d    & num.  & num.   & param. \\
                   & from                  & data     & (bytes)   & tokens   &       &      & heads & layers & count \\
    \midrule
    BaseEnEs       & ---                   & En-Es    & 612M      & 1.2B     & 24576 & 256  & 4     & 8      & 19M  \\
    TunedEnFr      & BaseEnEs              & En-Fr    & 640M      & 1.2B     & 24576 & 256  & 4     & 8      & 19M  \\
    TunedEnPt      & BaseEnEs              & En-Pt    & 616M      & 1.2B     & 24576 & 256  & 4     & 8      & 19M  \\
    MergedEnEsFrPt & TunedEnFr+TunedEnPt   & ---      & ---       & ---      & 49152 & 256  & 4     & 8      & 38M  \\
    \bottomrule
    \end{tabular}
    \end{scriptsize}
    \\[5mm]
    \caption{Architecture and training details for model merging experiments.}
    \label{tab:merging_model_arch_hyperparams}
\end{table}

\section{Omitted formal claims and proofs}\label{sec:apxjl}
\subsection{Proof of Observation~\ref{obs:protocol_equiv}}\label{sec:protocol_equiv}

\begin{proof}
The equivalence is straightforward to verify, rewriting the linear-algebraic multiplication expressions of \eqeqref{eq:bdhgraph} in Einstein summation notation and comparing respective index pairs. At any time, during the execution of rules for layer $l$, variables $X(i)$, $Y(i)$ and $\sigma_l(i,j)$ in the protocol description, for $i, j \in\{1,\ldots,n\}$ correspond to the $i$-th coordinate of vectors $\xsparse_{t,l}$ (based on $\xsparse_{t,l-1}$ from the previous round), $\xysparse_{t,l}$ (based on $\xysparse_{t,l-1}$ from the previous round), and matrix entry $\corr_{t,l}$ (based on $\corr_{t-1,l}$ from the previous token). The auxiliary variable $A(i)$ corresponds to a similar auxiliary vector $a_{t,l}:=\corr_{t-1,l}\xsparse_{t,l}$ in an intermediate step of computation of $\xysparse_{t,l}$ from $\xsparse_{t,l}$. The parameter $u(i,j) \in R^+$ associated with an element of state follows from the definition of matrix $U$; we assume for simplicity that $U$ is diagonal (which corresponds to the case of ALiBi). Finally, in Table~\ref{tab:protocolx}, the auxiliary node variables $X\ee(i), X\ii(i), Y\ee(i), Y\ii(i)$ are used to handle the thresholding of the inhibitory circuit.
\end{proof}

\subsection{Formal statement of Claim~\ref{claim:linearinformal} (linear attention)}\label{apx:linear}

We provide the following Claim, expressing the operation of attention under \emph{$C$-non-adversarial} key vectors $(k_\tau)$, $t= 1 \ldots t$, understood in the sense that there exists $C \in N$, $0\leq C < t-1$ such that, if considering $(k_\tau)$ as a sequence of random variables, each $f(k_\tau)$, $\tau= 1 \ldots t$, can be considered sampled independently at random in $S^{\nu}$ with respect to all keys sampled previously, except for at most $C$ such keys. We put $C=t-1$ for adversarial inputs, or if this condition cannot be satisfied at all due to the nature of function $f$.

\begin{claim}\label{claim:linear}
Let $\Lambda$ be a space of keys and queries, let $\phi : \Lambda \times \Lambda \to [-1,1]$ be an attention affinity function, and let $f : \Lambda \to S^{\nu}$, for some $\nu = O(\poly(n))$, be such that for any $q,k \in R$, we have $f(q)\cdot f(k) = \phi(q, k) \pm O(n^{-100})$.
Fix $\delta > 0$ and $C \in \Nat$. Let $A_{\phi, t}$ be a block which computes attention $a_t$ given by \eqeqref{eq:attn2a}, for a given sequence of key-query inputs $(k_1, \ldots, k_t)$ and values $(v_1, \ldots, v_t)$, where $t < \delta n / ((C+1) \log n)$ is fixed, $k_\tau \in \Lambda$, and $v_\tau \in R^d$ are of similar strength in the L2-norm, with $c_1 \leq \| v_\tau \| \leq c_2$, for all $\tau = 1 \ldots t$, for some constants $0 < c_1 \leq c_2$. Then the (simplified) linear attention equation of \BDHGPU:
\begin{equation}\label{eq:attnlinear}
\ket{\yKV_{t}} := \sum_{\tau=1}^{t-1}
    \ket{v_\tau}
    \bra{x_\tau}
    \ket{x_t}
\end{equation}
expresses $A_{\phi, t}$ with $O(\sqrt\delta)$-error in the L2-norm (i.e., $\|\yKV_\tau - a_\tau\| = O(\sqrt\delta)$, provided that the input vector $(k_\tau)$ is $C$-non-adversarial, under a suitable randomly chosen key preparation function $f' : \Lambda \to R^n$ , $x_\tau := f'(k_\tau)$, where $f'$ depends on $f$, w.h.p. in $n$ with respect to choice of $f'$.
\end{claim}
\begin{proof}[Proof (sketch)]
To simplify notation, assume w.l.o.g. that $\Lambda = S^{\nu}$ and $f= idem$; to undo this assumption, at the end of the proof we apply $f \circ f'$ for preparation in place of $f'$.

All vectors $v$ and the result $a_t$ we are looking to calculate are in $R^d$. With this notation, the attention task we are approximating is:
\begin{equation}
\label{eq:attn2}
a_t = q \sum_{\tau=1}^t k_\tau^T v_\tau.
\end{equation}
(this is still the general form of attention almost precisely equivalent to~\eqref{eq:attn2a}, not a special case).

The goal is to show how, subject to $t < \delta n / \log^2 n$, linear attention in dimension $n$ given by~\eqref{eq:attnlinear} is a sufficiently precise estimation of~\eqref{eq:attn2}.

Consider now, with $\Lambda = S^{\nu}$, $f' : S^{\nu} \to R^n$, where we recall that $x_\tau := f'(k_\tau)$, to be a suitable dimensionality reduction preserving approximation of scalar product between $R^{\nu}$ and $R^n$. For simplicity of argument, we let $f' : R^{\nu} \to R^n$ be a standard \JL transform, with the additional property that $f'(-z) = -f'(z)$ for all $z\in R^{\nu}$ (easy to obtain from any other \JL transform $f''$ by taking $f'(z) := (f''(z) - f''(-z))/2$). The distortion of scalar product in $R^n$ is then known to be bounded as follows: $|\bra{k_\tau} \ket{k_t} - \bra{x_\tau} \ket{x_t}| = O(\eps) (\|k_\tau\| + \|k_t\|) = O(\eps)$, w.h.p.\ with respect to choice of $f'$. Here, $\eps = \sqrt {\log n / n} = O(\sqrt \delta) / \sqrt {(C+1) t \log t}$, where the last inequality follows from the assumption on $t$ made in the Claim.

We now consider the sequence $r_{\tau} := \bra{k_\tau} \ket{k_t} - \bra{x_\tau} \ket{x_t}$, for $\tau < t$. Set aside the (at most $C$) elements $r_\tau$ for which $k_\tau$ and $k_t$ are not independent. For all other elements, consider that $|r_{\tau}| = O(\eps)$ as established previously, and the sign $r_{\tau}/|r_{\tau}|$ is chosen independently at random with respect to all but at least $C$ elements by the conditions imposed on $f'$ and $k_{\tau}$. It follows that $\sum_{\tau=1}^t r_{\tau}$ can be represented as a sum of $O(C)$ martingales, each of which has length $O(t/(C+1))$ and all elements bounded by $O(\eps)$ with $\eps = O(\sqrt \delta) / \sqrt {(C+1) t \log t}$. The Claim follows directly, by applying Azuma's inequality to each of these martingales independently.
\end{proof}

Considering the extreme cases of $C=0$ and $C=t-1$, the above Claim leads directly to Claim~\ref{claim:linearinformal}, clarifying over what time, linear attention can be used to express general attention.

\subsection{Proof of Claim~\ref{claim:graphs}}\label{apx:proofgraphs}

\begin{proof}
The proof is almost immediate, through the construction of an appropriate neuron-synapse interaction graphs $H\ee$, $H\ii$ such that $G\ee={H\ee}^2[V]$ and $G\ii={H\ii}^2[V]$. Consider $\encoder' \in (R^+)^{2d \times n}$ such that $\encoder'_{\alpha,j} = \relu{\encoder_{\alpha,j}}$ and $\encoder'_{\alpha+d,j} = \relu{-\encoder_{\alpha,j}}$, for $j \in \{1,\ldots,n\}$ and $\alpha \in \{1,\ldots,d\}$. Define $\decoder\ee, \decoder\ii \in (R^+)^{n \times 2d}$ so that:
$$
(\decoder\ee - \decoder\ii) \encoder' = \decoder \encoder.
$$
Indeed, notice that this is always possible by redistributing elements of $\decoder$ into $\decoder\ee$ and $\decoder\ii$ (putting $\decoder\ee_{i,\alpha} = \decoder\ii_{i+d,\alpha} = \relu{\decoder_{i,\alpha}}$) and $\decoder\ii_{i,\alpha} = \decoder\ee_{i+d,\alpha} = \relu{-\decoder_{i,\alpha}}$), so that, for all $i, j \in \{1,\ldots,n\}$ and $\alpha \in \{1,\ldots,d\}$, we have:
$$
(\decoder\ee_{i,\alpha} - \decoder\ii_{i,\alpha})\encoder'_{\alpha,j} + (\decoder\ee_{i,\alpha+d} - \decoder\ii_{i,\alpha+d})\encoder'_{\alpha+d,j} = \decoder_{i,\alpha}\encoder_{\alpha,j}.
$$
Considering $S = \{1,\ldots,2d\}$, the definition of $H\ee$ as the union of edges of $\decoder\ee$ and $\encoder'$ on input neuron layer $V$, hidden layer $S$, and output neuron layer $V$ follows. Likewise, we define $H\ii$ as the union of edges of $\decoder\ii$ and $\encoder'$.

We verify that for $G\ee={H\ee}^2[V]$ and $G\ii={H\ii}^2[V]$, we have $G\ee - G\ii = \decoder\encoder$, and the Claim holds.
\end{proof}

\paragraph{Considerations of building linear circuits.} The above proof makes the neuron-synapse interaction graphs $H\ee$, $H\ii$ sparse in terms of the number of edges, as required to show that the number of parameters are preserved by correspondence. However, it is a purely technical construction, and nodes in the synaptic layer have high degree, $n$. While preserving strict equivalence of linear dynamics, the degrees of nodes of the considered graphs in the synaptic layer can be reduced in this construction, at the cost of increasing the number of edges of graphs $H\ee$, $H\ii$. (For example, subdividing each node of the synaptic layer into $a^2$ nodes can be used to reduce their degree $\Theta(a)$-times, while increasing the number of edges $\Theta(a)$-times; putting $a=\sqrt{n/d}$ we reach graphs $H\ee$, $H\ii$ with degree $O(\sqrt{nd})$ in both the neuron and synaptic layers, and consequently $O(n\sqrt{nd})$ edges.)

Reduction of internal degrees in this circuit is also possible by introducing more than 1 hidden layer, creating a form of branching circuit. The implementation for this in a distributed way remains very simple, as the considered dynamics of the form $z \to Gz$ are linear (token-propagation dynamics). The bound on the number of edges needed to represent such a circuit remains $O(nd)$, even when the circuit has constant degree.

The technical construction of the linear circuits $H\ee$, $H\ii$ provided in this Appendix do not affect results concerning the analysis of the structure of neuron-neuron interaction graphs $G\ee$, $G\ii$. These neuron-neuron interaction graphs plausibly maintain a heavy-tailed, power-law-like degree distribution, as is the case for the models considered empirically in Section~\ref{sec:bdh_empiricalgraphs}.

\subsection{Formal statement of Claim~\ref{claim:attention}}\label{apx:proofattention}

\begin{claim}\label{claim:attentionformal}
Let $\decodery, \encoder$ be parameter matrices of BDH-Normfree. Then, there exists a graph $\graphy \in \G(n,O(nd))$, expressible through a sparse linear circuit, a graph $\graphs$ having $O(nd)$ edges, and a sparse linear value preparation function $A : {R^+}^{n} \to {R^+}^{n}$, such that, for any sequence of keys $(\xsparse_{\tau,l})_{0 \leq \tau \leq t}$ and values $(\xysparse_{\tau,l-1})_{0 \leq \tau \leq t}$, with $\xsparse_{\tau,l}, \xysparse_{\tau,l-1} \in {R^+}^{n}$, we have:
$$
(\graphy\ee - \graphy\ii) \corr^*_{t-1,l} x_{t,l}= \decodery\encoder\corr_{t-1,l} x_{t,l},
$$
where $\corr_{t-1,l} = \sum_{\tau < t}\ket{\xysparse_{\tau,l-1}}\bra{\xsparse_{\tau,l}}\rope^{t-\tau}$ represents the attention state of BDH-Normfree following \eqeqref{eq:corr}, and $\corr^*_{t-1,l} = \lr{\sum_{\tau < t}\ket{A(\xysparse_{\tau,l-1})}\bra{\xsparse_{\tau,l}}\rope^{t-\tau}}\odot\graphs$ represents the corresponding attention state of the BDH system with sparse attention on graph $\graphs$, subject to appropriate preparation of attention values using function $f_y$.
\end{claim}

Before we start the proof, we make a general point about the formulation of the claim. We are considering the problem of expressing (or more generally, approximating) the matrix operator  $\corr_{t-1,l}$ by another, sparser one. The setting of our problem can be distilled into obtaining an equality or approximation of the form $\encoder \corr_{t-1,l} \approx \encoder^*\corr^*_{t-1,l}$, where $\encoder \in R^{d \times n}$ is a given low-rank matrix, $\encoder^*\in R^{d \times n}$ can be defined arbitrarily, and $\corr^*$ is defined as in the statement of the Claim. If we content ourselves with an approximation, then it is possible to have $\corr^* = \corr$ (i.e., put $f_y = idem$), using for example the stochastic sparsification framework of~\cite{mcsherry2007}, or a value-dependent variant (cf. e.g.~\cite{krauthgamerS23}). The samples chosen by such a framework in a value-dependent variant would lead to a graph $G_s$ which plausibly reflects the power-law element distributions that we empirically observe in $\corr$.
\begin{figure}
\centering
\includegraphics[width=0.5\textwidth]{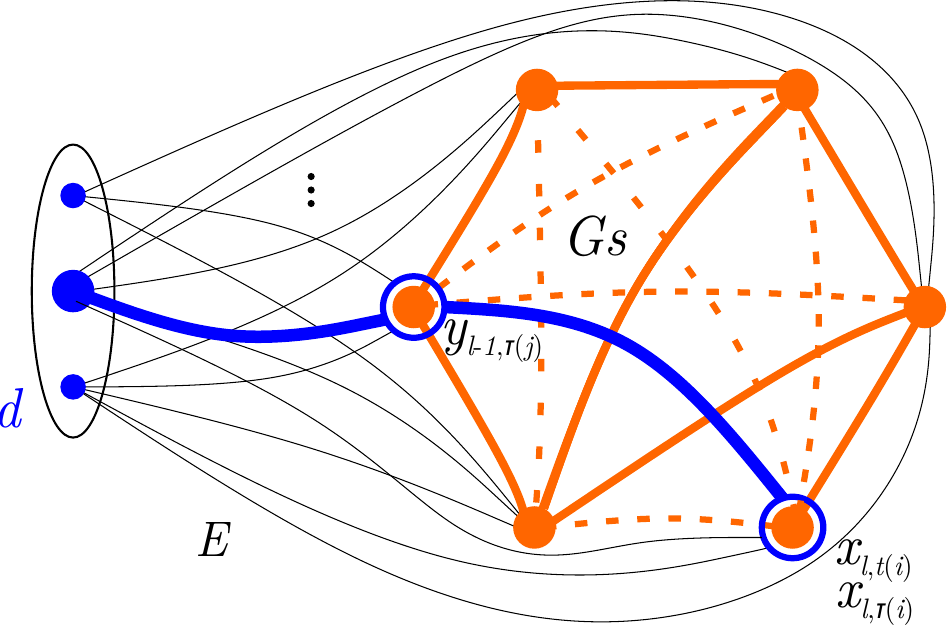}
\caption{Non-uniform graph attention: interpretation of $\encoder (\corr_{l,t} \odot G_s)$ after sparsification of graph $G_s$.}
\end{figure}

While the spirit of such an approximation is generally valid, we opt in the proof for a simpler, purely technical argument applicable to our specific setting, which gives a strict equality in the statement of Claim~\ref{claim:attentionformal} subject to linear preparation of attention values with a function $A$. In practice, this would mean that two successive layers of BDH with sparse state are sufficient to express a layer of BDH-Normfree under this reduction.

To prove the claim, it is enough to embed the connection structure of the encoder matrix, treating it as a graph, into $\graphs$.

\begin{proof}(of Claim~\ref{claim:attentionformal})
Fix arbitrarily subset $D \subseteq V$ of neurons, with $|D|=2d$. For the given matrix $\encoder \in R^{d\times n}$ from \BDHGPU, let $\encoder' \in (R^+)^{2d \times n}$ be defined as in the proof of Claim~\ref{claim:graphs} in Appendix~\ref{apx:proofgraphs}, and let $\decodery\ee$, $\decodery\ii$ also be applied as in that proof for considerations of decoder $\decodery$. Define the value preparation function $A$ as the immersion of vectors over $V$ into $D$ using $\encoder'$. Define $\graphs$ to be the all-ones matrix on the $2d$ columns corresponding to $D$, and zeros elsewhere. Then, define $\encoder^*\in R^{2d\times n}$ to be a diagonal matrix acting on its first $2d$ elements (corresponding to $D$), and zeros elsewhere. Setting $\graphy\ee = \decodery\ee \encoder^*$ and $\graphy\ii = \decodery\ii \encoder^*$, we obtain the claim.
\end{proof}

\section{Desirable properties of a local graph dynamics for language models}\label{sec:checklist}\label{sec:b}

We outline several general criteria of computational expressiveness and computational efficiency which a distributed computing system has to meet to effectively deal with language and reasoning. For this, we take a first-principles approach, relying only on very fundamental properties which an attention-based language model appears to need to capture, and which are applicable far beyond the specific case of BDH --- plausibly, being equally applicable to human and human-like reasoning.\footnote{In particular, the reader will have no doubt observed that graph settings applicable to language inference and reasoning systems, which involve task inputs spread out over time and the emergence of graph structure, are very different from graph-based frameworks which directly associate the task to solve with the communication graph (the latter case includes most considerations of: Graph Neural Networks, Graph Transformers, the LOCAL/CONGEST model of distributed computing, Approximate Message Passing systems, etc.)
}

\begin{hypothesis}
\label{hyp:x}
We expect any efficient graph-based distributed system dealing with language and reasoning using an attention-based approach to have the following characteristics:
\begin{itemize}
\item \itemone The system achieves computationally irreducible dynamics, i.e., it provides no systematic opportunity to predict the outcomes of its inference or approximate its dynamics in a numerically easier way than by running the system itself.
\item \itemtwo The state-space dynamics of the distributed system is a non-linear interacting particle dynamics, i.e., the system does not admit an efficient representation as a non-interacting particle system, but relies on a form of non-linear evolution expressed through (at least) two-particle interactions. (Such interactions are necessary, in particular, to enable multi-point correlation analysis on language inputs, when assuming only a small number of inference steps of the system per output token.)
\item \itemthree The system is capable of computing correlations between pairs of scalar variables localized at different nodes of the distributed system, and storing the state of such correlations so that the result is accessible from these two nodes. (This is plausibly needed to express attention in a state-space system.)
\item \itemfour The communication graph of the distributed system does not, in itself, represent any specific task input to solve, but reflects a trained model (a program), whereas tasks are represented as inputs to this program, presented over time. The communication graphs used to solve language and reasoning problems are expected to display modular, scale-free structure.
\end{itemize}
\end{hypothesis}
\emph{A detailed discussion of the four items of the Hypothesis is provided below.}

\subsubsection*{\itemone $\diamond$ Computational models have irreducible dynamics.}

We start by recalling a general observation which is applicable to most learning systems $L$ (machine learning models, biological systems) that have learned how to do computations: they are likely to have chosen state-space dynamics that will allow them to resolve their computational problem with the least effort during inference. In other words, \emph{if there is a physical system $P$ that solves a given computational problem, and if there exists a simulation $S(P)$ of this physical system that would approximate system $P$ with less effort, the learning system $L$ will be following the dynamics of $S(P)$, not those of $P$.}

We provide a few hypothetical examples for intuition, anchored in different areas of particle dynamics.

If $P$ were the particle dynamics of electrons in a resistor network, the simulation $S(P)$ could be a calculation based on Ohm's law with a Laplacian solver --- and we would consequently expect the dynamics of our computational system $L$ to follow the Laplacian solver code, and not to simulate electron dynamics.

If $P$ were the ensemble of billions of Internet users performing short walks clicking through links of the world wide web, the simulation $S(P)$ would be a calculation of aggregate behavior, reminiscent of PageRank, and we would expect $L$ to encode the parallel dynamics of Map-Reduce matrix operations of PageRank, not the simulation of individual agents.

If $P$ were a quantum system amenable to approximation by perturbation theory, we would expect $L$ to simulate the (classical) calculus of this perturbation theory, and not the quantum system $P$ directly.

Most mechanical systems admit some form of more efficient simulation, which means the the dynamics of such systems are rarely a suitable choice for neuronal models. Anecdotally, in nature, only very simple systems like the Physarum slime mold~\cite{jabr2012brainless} rely on direct action (with hydrostatic pressure gradients) to perform their optimization process; and contemporary neuroscience research suggests that even the simplest neuronal brains do not perform their work in a similar ``fluid-mechanical'' manner.

The irreducibility of $L$ means that this system is stretched to the limits of stability, just as a highly optimized numerical algorithm would be have been simplified and optimized to the limit of numerical stability. This relates to the limits of dimensionality reduction techniques that we have explored through a largely equivalent information-lens perspective of loss of precision and loss of information which it inflicts upon the model.

\subsubsection*{\itemtwo $\diamond$ Latent concept spaces arise from outcomes of particle-particle interactions.}

Dynamics of systems with multiple particles moving around in a (deformable) environment fall into two broad categories, depending on the strength of interaction between different parts of the dynamics. In the simpler setting, particles can be assumed \emph{at short time scales} to be moving in an environment unchanged by other particles --- the concurrent action of other particles, which would change the environment, does not need to be taken into account when representing individual particle motion, nor is it necessary to consider particle-particle interactions. By contrast, in the more general setting, the dynamics of multiple particles are tightly coupled, and their dynamics need to be modeled (simulated) together.

An example of a dynamics with no coupling would be a dynamics of multiple independent random walkers, such as the previously mentioned dynamics of electricity in wires, or the dynamics of PageRank. Examples of dynamics including interactions between particles, which may either happen directly or be moderated through the environment, include cellular automata, particle method simulations and molecular simulations, or swarms of communicating agents. %

The natural representation of state-space models as moving particles comes from the following interpretation. A distributed system with depth-$L$ computations (not least BDH or the \BDHGPU model given by the state equations~\eqref{eq:integral}) is amenable to interpretation as a system of walker particles performing an $L$-step walk over layers, starting at some token $t_0$ in the input layer $0$ and, in each time step $t \geq t_0$, either pausing (skipping a time step) or moving on to the next layer, until they reach the last layer $L$ in some time step $t_f$, at which point they leave the system, contributing to the distribution of the $t_f$-th output token. When attempting this approach with \emph{independent} walkers, the distribution of tokens output by such a system could be described by correlation functions following or resembling the Dyson series, $\sum_{\tau_L=0}^{t} \sum_{\tau_{L-1}=0}^{\tau_L-1}\ldots \sum_{\tau_1=0}^{\tau_2-1} F(\textrm{input}(\tau_1), \ldots, \textrm{input}(\tau_L))$. However, the output of attention (e.g., the linear attention output $\yKV$ given by equation \eqref{eq:integral} for \BDHGPU, or defined similarly in other state space models based on linear attention), cannot be represented as a Dyson formula when unrolling the dynamics backwards through layers (even if it looks deceptively similar at first glance). Each entry retrieved from attention is an interplay between two moments of time: the moment at which the key-value pair was entered, and the moment at which the corresponding query arrived. In consequence, the considered dynamics can be represented, in each layer, as a linear sum of two-point correlations between current time $t$ and some point $\tau$ in the past. Thus, in the $l$-th layer, this recursion can (with some approximation) be unrolled into a linear combination of functions of sets of $2^l$ input tokens (provided in the $0$-th layer), but cannot be represented through correlation functions $F$ on smaller sets of tokens (e.g., of size linear in $l$). Otherwise put, a system like BDH can be described using particles performing $l$-step walks when \emph{relying on intermediate elements of KV-state $\corr$}, which are produced during interactions with other walker particles in intermediate layers, but needs to be viewed through at least $2^l$-point correlation functions defined directly on input tokens in the input layer.

The considered point is relevant because it \emph{precludes many forms of modeling of attention-based language dynamics, in particular those using non-interacting particle theories}. The precluded approaches include:
\begin{itemize}
\item $L$-grams, \textsf{word2vec}-like $L$-skip-grams~\cite{mikolov2013distributed}, as well as any other $L$-point correlations of past input tokens.
\item $L$-step non-interacting random walk models (walks inside the network structure, which move from input layers towards output layers across time).
\item systems known to be equivalent to the above, such as approximations of classical spin-chain systems by means of Feynman integral path lengths bounded by $L$~\cite{walks2025}, and many forms of graph/GNN kernels based on $L$-th powers of the graph Laplacian.
\item by extension, $L$-layer state-space systems which perform excessive compression (size reduction) of their state, in a way which eliminates most long-term correlations.
\end{itemize}

We can ask if this requirement for communication between particles is an artifact of the construction of BDH (and similarly, of the Transformer), or if it comes from a genuine need related to language and reasoning tasks. For language problems \textsl{per se}, the need for multi-point token correlation in $L$-layer language modeling plausibly follows from the expectation that the model should have the ability to create a syntax tree of a sentence by means of a single quick parallel scan over words in this sentence. With this assumption, the depth $L$ of computation used to build a language syntax tree should be sufficient to represent the number of \emph{levels} of the syntax tree that the model is able to process naturally, but can be (and in general, should plausibly be) much smaller than the number of \emph{leaves} (words) of this syntax tree. This is consistent with the RASP-L-based understanding of the Transformer's capabilities, which allows for expressing depth-L trees in a depth-L Transformer.\footnote{This does not mean the problem is easy; synthetic problems inspired by this type of tree problem were (for us) among the hardest to train into a Transformer with no Chain-of-Thought --- as compared to RASP-L problems described in~\cite{zhou2023algorithms} and others we tested.}

Such a way of mapping the tree structure of problems into the model's layers, from bottom to top, also essentially captures the ``generative'' nature of the considered models, which rely on concept spaces created and stored in state in intermediate layers, to guide both language comprehension and reasoning on language. Thus, the ability to handle language syntax trees efficiently, in itself, precludes the previously-mentioned types of modeling approaches.

\subsubsection*{\itemthree $\diamond$ The interaction process \abc describes attention.}

The preceding discussion in paragraph \itemtwo grounds state-of-the-art state-space language models in the world of interacting particle systems. %

Whenever the global vector-based description of a state-space model calls for a three-point operation, such as the trilinear operation of key-value-query attention, this translates into the nature of pairwise (for polynomial interaction terms, degree-two) non-linear particle interactions in the transition equations of the same model when described at the level of particles. Notably, at scale, \emph{the state-space transition equations of an attention-based model plausibly involve altering or deforming correlation strength between pairs of particles, with such pairs being represented as interaction variables in the state of the system}. This requirement on structure, repeated across layers, can be seen as sufficient: interactions of particle pairs are about the only requirement on non-linear rulesets that the system needs to be support, as demonstrated by the simple local transition rules of BDH.

Overall, the statement ``attention is all you need'', which describes a system-level global property,
translates into ``$\abc$ is all you need'' at the level of particle dynamics of a state-space language model.

\subsubsection*{\itemfour $\diamond$ Inputs to reasoning problems are sequential, not graph-based.}

Many real-world graphs are anchored in a spatial embedding of their nodes which is given by external constraints. For example, the structure of many social and transportation networks is impacted by the geographical placement of people and infrastructure on the globe.

In designing the dynamics for BDH, we are free from such spatial constraints. The graph topology corresponding to the model is free to take the shape needed to best resolve the problem. The problem itself is encoded as a sequence of tokens which arrive over time to the model (we take here a state-space view of the system).
We can naturally presume that the structure of the model graph of BDH is shaped in a way which follows from two aspects: this temporal encoding of information, and from the abstract (Platonic) latent space of concepts needed to deal with language and reasoning.

When looking for the right particle dynamics for language models, it seems reasonable to discard all \emph{unnecessary} aspects of spatial constraints.

One example of a particle interaction system which includes externally imposed constraints on the structure of the state space is that of cellular automata operating on a two-dimensional grid. While 2D cellular automata have appealed to public imagination, appearing in attempts to observe the emergence of intelligence at least since the 1970's, they are, in fact, an extremely cumbersome choice for representing in-context reasoning or language for any attention-based model. State-of-the-art language models seem to have no structural need for a low-dimensional grid in their dynamics. Arguably, the connection structure which needs to emerge in a graph system, allowing it to work efficiently in a setting of efficient information search
is precisely the opposite: it is a multi-scale, expander-like system of shortcuts, cf. e.g.~\cite{10.1145/1806689.1806744}. This scale-free graph structure is expected to correspond to the scale-free temporal behavior observed in natural systems~\cite{HE2010353}.

In the rest of this paragraph we briefly review other areas of computer science, and how they relate to the particle dynamics we are looking for in terms of their relationship to handling temporal inputs and the constraints they impose on the structure of the state-space.

The freedom of choice of graph topology in solving problems around language and in-context reasoning, which we are dealing with here, can be contrasted with settings in which the graph is, at the same time, part of the system dynamics (encoding interactions in the system) and a part of the statement of the problem input. This is particularly true for models of distributed computing inspired by computer networking (LOCAL, CONGEST, etc.) and other forms of interaction networks (Approximate Message Passing, quantum LOCC, etc.), where the same graph $G$ represents the communication network for the dynamics, and encodes the problem input --- with the required output being some function of $G$ (e.g., a clustering, coloring, spanning tree, etc.). Some distributed problems on graphs can be formulated so that the input and required output are independent of the graph structure, the notable ones being: majority consensus, leader election, information broadcasting, and computing aggregates. For such problems, the graph represents only a communication system, whose topology is more an obstacle to overcome, than an actual help in solving the problem. This applies also to architectures in Machine Learning which adhere to a known graph structure, such as Graph Neural Networks or Graph Transformers, when solving problems whose inputs are not naturally embedded in such a structure.

A handful of approaches in distributed computing are intended to describe systems which compute a function of an input signal which, like language, is spread out sequentially over time, and where computations happen while this signal is still arriving. In particular, some forms of particle dynamics can be distilled from the theory of self-stabilizing systems~\cite{dolev2000}, giving rise to settings where the system is expected to adapt its state in response to a time-changing input (see e.g.~\cite{BoczkowskiKN19}).
Among distributed streaming frameworks, one approach which, owing to its design, admits an elegant particle-based interpretation for time-changing inputs, is the incremental computing framework~\cite{DBLP:conf/cidr/McSherryMII13}. %
This framework emphasizes temporal commutativity, and is well suited to expressing dynamics of non-interacting particles, such as PageRank-like computation performed incrementally with Map-Reduce on time-changing graphs, or building nearest-neighbor indexes on sets of changing vectors. %
It does not naturally extend to the non-linear particle-particle interaction dynamics that appear in the context of attention (see paragraph \itemtwo).%

\section{\BDHGPU PyTorch code listing}\label{sec:bdh_code_listing}

The code listing below implements \BDHGPU (Definition~\ref{def:bdh}) for PyTorch version 2.7. It is self-contained, except for the implementation of RoPE which needs to be filled by the user. With respect to the state dynamics of~\eqeqref{eq:bdh}, it provides an extension supporting heads. The placement of layer norms and residual connections is modified with respect to~\eqeqref{eq:bdh}; in general, this aspect offers some flexibility.

This implementation assumes the simplest case of a fixed context window of length $T$. An unbounded context window is technically supported using a state-space kernel for Linear Attention, and works best following appropriate adaptation of the model for truncated backpropagation through time (see Appendix~\ref{sec:bdh_scaling_details}).

\newpage
\lstset{language=Python} 
\begin{lstlisting}[frame=single]
import torch
import torch.nn.functional as F
from torch import nn

D = 256  # internal dimension
H = 4  # heads
N = 32768  # neurons
L = 6  # layers
dropout = 0.05
vocab_size = 256

class BDH_GPU(nn.Module):
  def __init__(self):
    self.ln = nn.LayerNorm(D, elementwise_affine=False, bias=False)
    self.wte = nn.Embedding(vocab_size, D)
    self.drop = nn.Dropout(dropout)
    self.encoder = nn.Parameter(
      torch.zeros((N, D)).normal_(std=0.02)
    )
    self.decoder_x = nn.Parameter(
      torch.zeros((H, D, N//H)).normal_(std=0.02)
    )
    self.decoder_y = nn.Parameter(
      torch.zeros((H, D, N//H)).normal_(std=0.02)
    )
    self.readout = nn.Parameter(
      torch.zeros((D, vocab_size)).normal_(std=0.02)
    )
    self.attn = LinearAttention()

  def forward(self, idx):
    B, T = idx.size()  # mini-batch dimensions
    v_ast = self.ln(self.wte(idx).unsqueeze(1)) # B, 1, T, D
    
    for _ in range(L):
      x = F.relu(v_ast @ self.decoder_x) # B, H, T, N//H
      a_ast = self.attn(
        Q=x,
        K=x,
        V=v_ast,
      )
      y = F.relu(self.ln(a_ast) @ self.decoder_y) * x # B, H, T, N//H
      y = y.transpose(1, 2).reshape(B, 1, T, N)
      y = self.drop(y)
      
      # Start of layer with vectors x, y  
      v_ast = v_ast + self.ln(y @ self.encoder)  # B, 1, T, D
      v_ast = self.ln(v_ast)
      
    return v_ast.squeeze(1) @ self.readout  # B, T, vocab_size

class LinearAttention(nn.Module):
  def forward(Q, K, V):
    Qr = RoPE(Q)
    Kr = RoPE(K)
    
    return (Qr @ Kr.mT).tril(diagonal=-1) @ V
\end{lstlisting}
\end{document}